%% file: main.tex
\documentclass[final]{colt2020} 

\usepackage{times}

\coltauthor{%
 \Name{Jeongyeol Kwon} \Email{kwonchungli@utexas.edu} \\
 \addr The University of Texas at Austin
 \AND
 \Name{Constantine Caramanis} \Email{constantine@utexas.edu} \\
 \addr The University of Texas at Austin
}

\usepackage[utf8]{inputenc} 
\usepackage[T1]{fontenc}    
\usepackage{url}            
\usepackage{booktabs}       
\usepackage{amsfonts}       
\usepackage{nicefrac}       
\usepackage{microtype}      

\usepackage{enumitem}
\usepackage[english]{babel}

\usepackage{amsmath}
\usepackage{amssymb}
\usepackage{graphicx}
\usepackage{bm}
\usepackage{algorithm,algpseudocode}

\usepackage[title]{appendix}
\usepackage{comment}
\usepackage{amsfonts}
\usepackage{dsfont}
\usepackage{hyperref}
\usepackage{lipsum}
\usepackage{dsfont,subcaption,float}

\allowdisplaybreaks

\usepackage{thm-restate}
\newtheorem{fact}[theorem]{Fact}

\input{my_macro}

\title[EM gives Sample-Optimality for Learning Mixtures of Well-Separated Gaussians]{The EM Algorithm gives Sample-Optimality for Learning Mixtures of Well-Separated Gaussians}

\begin{document}

\maketitle

\begin{abstract}
\input{Abstract}
\end{abstract}


\section{Introduction}
\input{Introduction.tex}

\section{Preliminaries and Notation}
\input{Problem_Setup.tex}

\section{Convergence Analysis of the EM Algorithm}
\input{Main_Results.tex}

\section{Sample Optimal Learning with the EM Algorithm}
\input{Sample_Optimal_Algorithm.tex}

\section{Conclusion}
\input{Conclusion}

\subsection*{Acknowledgement}
This work was partially funded by NSF grants 1704778, 1646522, 1934932 and the Army Futures Command.

\bibliography{main}

\newpage 

\appendix
\input{Appendix.tex}

\end{document}

%% file: my_macro.tex
\newcommand{\tmu}{\tilde{\mu}}
\newcommand{\tpi}{\tilde{\pi}}
\newcommand{\tsigma}{\tilde{\sigma}}

\newcommand{\R}[1]{{R_{#1}^*}}	
\newcommand{\sigt}[2]{(\sigma_{#1}^* \vee \sigma_{#2}^*)}	

\newcommand{\tO}{\tilde{O}}	
\newcommand{\F}{\mathcal{F}}	
\newcommand{\G}{\mathcal{G}}	
\newcommand{\E}{\mathbb{E}}	
\newcommand{\D}{\mathcal{D}}	
\newcommand{\indic}{\mathds{1}} 
\newcommand{\vdot}[2]{\langle #1, #2 \rangle}
\newcommand{\Eps}{\mathcal{E}}

\newif\ifdraft
\drafttrue  

\newcommand{\red}[1]{\textcolor{red}{#1}}

\newcommand{\jycomment}[1]{\ifdraft {\bf{{\red{{Jeongyeol --- #1}}}}} \else {#1}\fi}

%% file: Abstract.tex
    We consider the problem of spherical Gaussian Mixture models with $k \geq 3$ components when the components are well separated. A fundamental previous result established that separation of $\Omega(\sqrt{\log k})$ is necessary and sufficient for identifiability of the parameters with \textit{polynomial} sample complexity \cite{regev2017learning}. In the same context, we show that $\tilde{O} (kd/\epsilon^2)$ samples suffice for any $\epsilon \lesssim 1/k$, closing the gap from polynomial to linear, and thus giving the first optimal sample upper bound for the parameter estimation of well-separated Gaussian mixtures. We accomplish this by proving a new result for the Expectation-Maximization (EM) algorithm: we show that EM converges locally, under separation $\Omega(\sqrt{\log k})$. The previous best-known guarantee required $\Omega(\sqrt{k})$ separation \cite{yan2017convergence}. Unlike prior work, our results do not assume or use prior knowledge of the (potentially different) mixing weights or variances of the Gaussian components. Furthermore, our results show that the finite-sample error of EM does not depend on non-universal quantities such as pairwise distances between means of Gaussian components.
    

%% file: Introduction.tex
Learning parameters of a mixture of Gaussian is a fundamental problem in machine learning. In this model, we are given random samples from $k \ge 2$ Gaussian components without observing the label, {\it i.e.}, the indicator of which component each sample comes from. In this paper, we focus on an important special case of this model where the covariance of each Gaussian component is a multiple of the identity matrix. Formally, we consider a Gaussian mixture model $\G^*$ whose probability density function (p.d.f.) can be represented as $\sum_{j=1}^k \pi_j^* \mathcal{N} (\mu_j^*, {\sigma_j^*}^2 I_d)$, where $d$ is the dimension, $I_d$ is $d \times d$ identity matrix, and $\mathcal{N}(\mu, \Sigma)$ denotes the  p.d.f.\ of a single Gaussian distribution with mean $\mu \in \mathbb{R}^d$ and covariance $\Sigma \in \mathbb{R}^{d \times d}$. Here $\pi_j^*$ are mixing weights, $\mu_j^* \in \mathbb{R}^d$ are means, and ${\sigma_j^*}$ are scale factors for (identity) covariances of each Gaussian component. This special case is often called the {\it spherical} Gaussian mixture model. Our goal is to estimate all parameters $\{(\pi_j^*, \mu_j^*, \sigma_j^*), \forall j \in [k]\}$ up to accuracy $\epsilon$.

Learning a mixture of Gaussians has a very long and rich history (see Section \ref{subsection:prior_works} for an overview of previous works). A variety of algorithms have been proposed for parameter learning. These either require separation assumptions on the means, or structural assumptions on the mean placement, requiring control of the tensor singular values (essentially requiring affine independence). Most tensor-based analysis has therefore been done in a smoothed setting (e.g., \cite{bhaskara2014smoothed, ge2015learning}), however in the absence of such structural assumptions, mean separation is what controls the hardness of the Gaussian Mixtures problem. Without any separation assumptions, even in one dimension, worst case instances require $\Omega(e^k)$ samples  \cite{moitra2010settling, hardt2015tight}. At the other extreme, \cite{dasgupta1999learning} demonstrated that under $\Omega(\sqrt{d})$ separation, sample-efficient (polynomial number of samples) identifiability is possible, thus providing the first upper bound on required separation for efficient identifiability. An important recent work by \cite{regev2017learning} characterized the exact threshold for sample-efficient identifiability, establishing that with separation $\Omega(\sqrt{\log k})$ sample-efficient identifiability is possible, where as below that threshold, a super-polynomial number of samples are required.


Perhaps the most widely used algorithm for mixture models is the Expectation-Maximization (EM) algorithm \cite{wu1983convergence}. Recently, \cite{yan2017convergence, zhao2018statistical} established the local convergence of the EM algorithm ({\it i.e.,} EM converges if initialized from a neighborhood of the ground truth) for mixtures of $k$ spherical Gaussians. These results require $\tilde{\Omega}(\sqrt{k})$ separation between means, and assume all components have identity covariances ({\it i.e.,} $\sigma_j^* = 1, \forall j \in [k]$). Thus this leaves open the key question as to the (local) behaviour of EM with $\Omega(\sqrt{\log k})$ separation.


The only known (local) algorithm that is guaranteed to converge in the $\Omega(\sqrt{\log k})$ separation regime is the EM-like algorithm proposed in \cite{regev2017learning}. However, the sample complexity of their analysis also has a high dependence on $k$ and instance-specific parameters as we explain in more detail below. They also require the initialization to be $O(1/k^2)$-close to the true parameters. By obtaining guarantees that depend only on $O(1)$-close initialization, we are able to give an optimal upper bound on sample complexity for learning the parameters of a mixture of spherical Gaussians. 

{\bf Main Contributions}. In this work, we return to the classical EM algorithm in the same $\Omega(\sqrt{\log k})$ separation regime, which thanks to the lower bound of \cite{regev2017learning} we know is optimal. We obtain improved convergence guarantees in this regime, and thereby close some of the existing gaps in the literature. Specifically, our main contributions are as follows:
\begin{itemize}[leftmargin=*]
    \item We show that with separation $\Omega(\sqrt{\log k})$, population EM converges given a good initialization. This improves the previous results of \cite{yan2017convergence} that required $\Omega(\sqrt{k})$-separation. For the initialization, our result only requires $O(1)$-closeness to the ground truth, hence improving the $O(1/k^2)$ initialization requirement in \cite{regev2017learning}. Finally, unlike all prior work we are aware of, our result does not assume prior knowledge of mixing weights or variance parameters, and these need not be the same; we show we can adaptively estimate these quantities along with the means. These improvements enable our last result below on the optimal sample complexity of learning Gaussian mixture model parameters. 
    
    \item We show that (sample-splitting) finite-sample EM converges to the ground truth given an $O(1)$-close initialization. Our result has sample complexity $n = \tilde{O}(d\pi_{min}^{-1} /\epsilon^2)$ (where $\pi_{min} = \min_i \pi_i^*$) to recover all parameters up to $\epsilon$ accuracy:
    \begin{align*}
        \forall i \in [k], \ \|\mu_i - \mu_i^*\| \le \sigma_i^* \epsilon, |\pi_i - \pi_i^*| \le \pi_i^* \epsilon, |\sigma_i - \sigma_i^*| \le \sigma_i^* \epsilon / \sqrt{d}.
    \end{align*}
    Note that a natural sample complexity lower bound for the Gaussian mixture model is $\Omega(d \pi_{min}^{-1} / \epsilon^2)$, since we need to collect at least $\Omega(d/\epsilon^2)$ samples from each component. We give the best possible sample complexity in terms of all parameters $\pi_{min}, k, d, \epsilon$. This significantly improves over previous results \cite{yan2017convergence, regev2017learning} where even in the balanced setting ($\pi_{min} = O(1/k)$), the sample complexities are at least worse by $\Omega(k^4 \rho^6)$ with an instance-dependent parameter $\rho$ \footnote{$\rho$ depends on instance-specific parameters such as $\max_{i\neq j} \|\mu_i^* - \mu_j^*\|$ or $\max_{i\neq j} \sigma_i^* / \sigma_j^*$. Since our sample complexity result does not depend on $\rho$, we do not require boundedness on parameters.}. 
    
    \item We show the sample-complexity $\tO(\max\{kd/\epsilon^2, k^3d\})$ for learning the parameters of spherical Gaussian mixtures with $\Omega(\sqrt{\log k})$ separation. For $\epsilon \le 1/k$ this gives $\tO(kd/\epsilon^2)$ and hence is optimal. The sample complexity guarantee here does not require any initialization, or boundedness of the parameters. 
    The breakthrough in \cite{regev2017learning} was the first to establish polynomial sample complexity learning, but only gave sample complexity in the form $poly(k, d, \rho, 1/\epsilon)$, which is at least $\Omega(k^9 d \rho^4 \epsilon^{-2})$. Our result closes the gap and shows that the information-theoretically necessary sample complexity is also sufficient: as long as the separation is $\Omega(\sqrt{\log k})$, then $\tO(kd /\epsilon^2)$ samples are sufficient; this matches a lower bound $\Omega(kd /\epsilon^2)$ up to logarithmic factors. 
\end{itemize}

\subsection{Prior Art}
\label{subsection:prior_works}
\paragraph{Separation-Based Algorithms} Learning mixtures of Gaussians has a long and rich history since \cite{dasgupta1999learning} who gave a first polynomial-time algorithm under $\Omega(\sqrt{d})$ separation. This result has been followed by a number of works \cite{sanjeev2001learning, vempala2004spectral, dasgupta2007probabilistic, achlioptas2005spectral, kannan2005spectral} that improve the result in various separation regimes. Currently, the best algorithmic results are the recent works by \cite{diakonikolas2018list, hopkins2018mixture, kothari2018robust}, where they provide algorithms that estimate parameters assuming $\Omega(k^{1/\gamma})$ separation, using $\tilde{O}(poly(k^\gamma, d, 1/\epsilon))$ samples (and running time) for arbitrary $\gamma > 0$. In particular, their can recover parameters of spherical Gaussian mixtures under $\Omega(\sqrt{\log k})$ separation using $\Omega(poly(k^{\log k}, d, 1/\epsilon))$ samples. While all works mentioned here aim to get (nearly) polynomial-time algorithms, our work is more in line with \cite{regev2017learning}: we focus on upper bounding the sample complexity.

\paragraph{Iterative Algorithms} EM is one of the most popular algorithms for mixture problems. The first results on convergence were infinitesimally local, and asymptotic \cite{redner1984mixture, xu1996convergence, ma2000asymptotic}. 
Recently, the work in \cite{balakrishnan2017statistical} builds off the idea of coupling finite sample and population EM, characterizing the non-asymptotic convergence of EM within a fixed (non-infinitesimal) basin of attraction. A flurry of work has followed in recent years, making substantial progress in the theory of EM. For instance, \cite{xu2016global, daskalakis2017ten} established the global convergence of the EM algorithm from a random initialization for Gaussian mixture models with two components, and \cite{kwon2019global} established the same for two component mixed linear regression. For more than two components, we cannot hope for such a global convergence guarantee, as shown by \cite{jin2016local}. On the positive side, some recent works have established local convergence results, showing that EM converges from a well-initialized point, under some minimum separation conditions \cite{yan2017convergence, kwon2019converges}. However, the best known guarantees for the EM algorithm for a mixture of Gaussians with $k$ components require separation of order $\tilde{\Omega}(\sqrt{k})$, and, moreover, are restricted to the equal identity covariance setting across all components \cite{yan2017convergence}. Another popular iterative heuristic is the $k$-means algorithm (also called Lloyd's algorithm) \cite{kumar2010clustering, awasthi2012improved, lu2016statistical}. The best known convergence result for this algorithm requires at least $\Omega(\sqrt{k})$ separation. Thus the state of the art analyses of EM and Lloyd's both leave a significant gap to the limit of $\Omega(\sqrt{\log k})$ mean separation. A variant of EM proposed in \cite{regev2017learning} takes a big step forward: it is shown to converge locally for spherical Gaussian mixtures with $\Omega(\sqrt{\log k})$ separation. While their algorithm is sample-efficient, the bound on samples is a (large) polynomial in $k$ and instance-specific parameters, where as the information theoretic lower bound is $\Omega(kd/\epsilon^2).$

\paragraph{Moment-Based Methods} The method-of-moments is a powerful general-purpose technique for learning a family of parametric distributions.
However, even in one dimension, an information-theoretic argument shows that an exponential (in $k$) number of samples is required to recover the parameters of a Gaussian mixture model in the absence of a minimum separation condition \cite{moitra2010settling,hardt2015tight}.
To circumvent such information-theoretic bottleneck of parameter learning, a vast line of work explores tensor-decomposition methods in a smoothed setting  \cite{hsu2013learning, ge2015learning, anandkumar2014tensor, kalai2010efficiently, anderson2014more}. However, such an approach cannot work when the means of Gaussian components lie in a low dimensional subspace. Furthermore, statistical precision of moment-based methods has poor dependence on the geometric properties of an instance, such as singular values of a tensor matrix or the norm of means. Therefore, they are often used in conjunction with an iterative procedure such as the EM algorithm which produces much more accurate estimators as we show in this work.

\paragraph{Lower Bounds} Without any separation, $\Omega(e^k)$ samples are necessary \cite{moitra2010settling, hardt2015tight}. In fact, the separation condition has to be at least $\Omega(\sqrt{\log k})$ to get a fully polynomial sample complexity as shown in \cite{regev2017learning}. There are also computational lower bound results due to \cite{diakonikolas2017statistical, diakonikolas2018list} framed in a statistical query (SQ) model \cite{feldman2017statistical} for general (non-spherical) Gaussian mixtures. It would be interesting to understand the implications of their results in the setting of spherical Gaussian mixtures (though we do not explore this here).

\paragraph{Distribution Learning} Another branch of research in Gaussian mixture models is density estimation. \cite{feldman2006pac, chan2014efficient, suresh2014near, diakonikolas2019robust, li2017robust, ashtiani2018sample, ashtiani2018nearly}. In this problem, the goal is to learn a distribution $\G$ that minimizes the total variation distance to $\G^*$. When this hypothesis $\G$ is also in the class of mixtures of $k$ Gaussian components, it is called {\it proper} learning. Most known proper-learning algorithms are sample-efficient or even sample-optimal \cite{ashtiani2018nearly}. Our result connects the result in \cite{ashtiani2018nearly} to the parameter learning of a well-separated mixture of Gaussians.

\paragraph{Other Related Work} Finally, we mention that there are a few related branches of research to learning a mixture of Gaussians such as graph clustering and community detection \cite{massoulie2014community, mixon2016clustering, yin2019theoretical}. In such problems, an analogous notion of separation condition is considered for sample-efficient learning (see the recent survey in \cite{abbe2017community}). 

\subsection{A Roadmap and Proof Outline}
Our starting point is the by-now standard procedure where we analyze EM in the population setting first, and then connect the result to finite-sample EM through the concentration of measures. 

\paragraph{Analysis of the Population EM Algorithm.} The E-step in the EM algorithm constructs weights (soft-label) for all components to construct the expectations of the log-likelihood on the current estimators. If we have a good enough estimation of parameters, then we anticipate that most samples should have approximately the right weights according to which components they come from. Given the current estimate of the mean parameters, let $\Eps_{good}$ denote the set where the E-step assigns approximately the right label (weight). Then, the estimation error in means, for example, after one EM step can be decomposed as 
\begin{align*}
    \E_\D[wX] - \mu^* = \underbrace{(\E_\D[wX | \Eps_{good}] - \mu^*) P(\Eps_{good})}_{\text{errors from good samples}} + \underbrace{(\E_\D[wX |\Eps_{good}^c] - \mu^*) P(\Eps_{good}^c)}_{\text{errors from bad samples}},
\end{align*}
where $\E_\D[wX]$ is the M-step operator for means, and $w$ is a weight constructed for a sample $X$ in the E-step. In the well-separated regime, we can show that $\E_\D[wX | \Eps_{good}] \approx \mu^*$ for good samples, while $P(\Eps_{good}^c) \approx 0$ for bad samples. While previous local analysis has a similar flavor \cite{yan2017convergence, balakrishnan2017statistical}, it is significantly more challenging to construct these good events for non-symmetric and non-equal variance Gaussian mixtures, since the effective dimension of the EM operator does not shrink as in the symmetrized \cite{balakrishnan2017statistical} or equal-variance setting \cite{yan2017convergence}. We show that if 
all parameters are well-initailized (see details in Theorem \ref{theorem:gmm_pop_main}), 
then EM converges locally to the true parameters at a linear rate.

\paragraph{Analysis of the Finite-Sample EM Algorithm.} 
In order to get the right order of statistical error for the EM algorithm, our measure concentration analysis treats good and bad samples separately. We adopt the technique used in \cite{kwon2019converges}, and we split the concentration of measure argument into two parts: (i) concentration due to the sum of independent random variables and (ii) concentration of the empirical probability of the event (see Proposition \ref{lemma:indic_prob_decompose}). This categorization strategy is critical to get a finite-sample error that is instance-independent and does not depend on, for example, pairwise distances between means. This independence is in contrast with results of prior work, e.g., \cite{yan2017convergence, zhao2018statistical, balakrishnan2017statistical, regev2017learning}. While technical, this concentration argument is the key differentiator in our results. We use this same technique again later in the paper, to obtain the optimal statistical error for learning a mixture of well-separated Gaussians.

\paragraph{Minimal Initialization Requirements and Sample-Optimal Learning.} 
We next 
consider the problem of sample-optimal learning. To use our EM result, we need to show a sample-efficient strategy for obtaining a sufficiently good initialization for EM. In a recent work, \cite{ashtiani2018nearly} gives the optimal sample-complexity result for proper learning. When applied to our setting, we can learn a candidate mixture of spherical Gaussians $\G$ that satisfies $\| \G - \G^*\|_{TV} \le \epsilon$ using $\tO(kd/\epsilon^2)$ samples, where $\|\cdot\|_{TV}$ is the total variation (TV) distance. We use this result to obtain a guarantee on the mean parameters. Specifically, we show that if a candidate spherical Gaussian mixture $\G$ satisfies $\|\G - \G^*\|_{TV} \le \pi_{min}/4$, then the mean parameters in $\G$ satisfy 
\begin{align}
    \label{eq:intro_kmeans_cond}
    \|\mu_i - \mu_i^*\| \le \frac{1}{4} \min_{i \neq j} \|\mu_i^* - \mu_j^*\|,
\end{align}
where $\mu_i$ are mean parameters of $\G$ and $\mu_i^*$ are means of $\G^*$. From \eqref{eq:intro_kmeans_cond}, we get an initialization condition that is not yet good enough to plug in our EM result. We develop a bridge to obtaining a sufficiently good initialization of all parameters with (essentially) the $k$-means algorithm, and conclude that $\tO(\pi_{min}^{-1} d /\epsilon^2 \vee k\pi_{min}^{-2}d)$ samples are sufficient to recover all parameters up to arbitrary $\epsilon$ accuracy.
We note that this final result is possible because the initialization requirement we ultimately need is significantly weaker than what prior art requires.

%% file: Problem_Setup.tex
We consider the mixture of $k$ spherical Gaussian mixtures with parameters $\{(\pi_j^*, \mu_j^*, \sigma_j^*), \forall j \in [k]\}$. True parameters are denoted by $(\cdot)^*$, while estimators are denoted using the same letters without $*$. Each $j^{th}$ component has mean $\mu_j^*$ and covariance $\sigma_j^* I_d$ in $\mathbb{R}^d$. We use $\D$ to represent the distribution of the mixture of Gaussians $\G^*$, and $\D_j$ to represent the distribution of the $j^{th}$ Gaussian component. 

The EM algorithm is composed of two steps, the E-step that constructs the expectation of the log-likelihood on the current estimators, and the M-step that maximizes this expectation. For a mixture of spherical Gaussian distributions, each step of the EM algorithm is as follows: 
\begin{align*}
    \mbox{(E-step)}: && w_i(X) &= \frac{\pi_i \exp(-\|X - \mu_i\|^2/(2 \sigma_i^2) - d \log (\sigma_i^2) / 2)}{\sum_{j=1}^k \pi_j \exp(-\|X - \mu_j\|^2/(2 \sigma_j^2) - d \log (\sigma_j^2) / 2)}, \\
    \mbox{(M-step)}: && \pi_i^+ &= \E_\D[w_i], \quad \mu_i^+ = \E_\D[w_i X] / \E_\D[w_i], \\
    && {\sigma_i^+}^2 &= \E_\D[w_i \|X - \mu_i^+\|^2] / (d \E_\D[w_i]),
\end{align*}
for all $i \in [k]$. In the above notation, we use $\E_\D [\cdot]$ to denote the expectation over the entire mixture distribution. We use $\E_{\D_j} [\cdot]$ to represent the expectation over the single $j^{th}$ Gaussian component. In the E-step, $w_i := w_i(X)$ represents the probability of the sample $X$ being generated from the $i^{th}$ component as computed using the current estimates of parameters $\{(\pi_j, \mu_j, \sigma_j), \forall j \in [k] \}$. After the M-step, we use $(\cdot)^+$ to denote the corresponding updated estimators. When we consider the entire sequence of estimators, we use $(\cdot)^{(t)}$ to denote estimators in the $t^{th}$ step. Finite-sample EM replaces the expectation with the empirical mean constructed with $n$ i.i.d.\ samples. The corresponding estimators are denoted by $\{(\tpi_j, \tmu_j, \tsigma_j), \forall j \in [k]\}$. In the E-step in the finite-sample EM, we use $w_{1,i}$ to represent the weight for the $1^{st}$ component constructed with the $i^{th}$ sample $X_i$. 

Finally, we introduce some conventions we use in this paper. We define $\pi_{min} = \min_i \pi_i^*$, $\rho_\pi = \max_i \pi_i^* / \min_i \pi_i^*$, and $\rho_\sigma = \max_i \sigma_i^* / \min_i \sigma_i^*$. We use $R_{ij}^* = \|\mu_i^* - \mu_j^*\|_2$ to denote the pairwise distance between components. We often use $\|\cdot\|$ without subscript $2$ to denote the $l_2$ norm of a vector in $\mathbb{R}^d$.
The estimation error in the mean of the $i^{th}$ component is defined as $\Delta_{\mu_i} := \mu_i^* - \mu_i$. We use $\indic_{\Eps}$ for the indicator function for the event $\Eps$. We use standard complexity analysis notations $o(\cdot), O(\cdot), \tO(\cdot), \Omega(\cdot)$. Finally, we use ``with high probability'' in statements when the success probability of the algorithm is at least $1 - \delta$ where $\delta = n^{-\Omega(1)}$.

%% file: Main_Results.tex
In this section, we give local convergence guarantees for both population EM and finite-sample EM. We first study population EM, and connect the result to the finite-sample setting.

\subsection{Analysis of Population EM}
We first state our main result for population EM. Compared to previous works, we also consider the setting of unknown and unequal variances, and hence must estimate these along with the means and mixing weights. We focus on handling the means and mixing weights in the main text, as the analysis for variance estimators is significantly more involved and delicate (see Appendix \ref{appendix:population_em_variance} for the analysis of variance estimation).
\begin{theorem}
    \label{theorem:gmm_pop_main}
    There exists a universal constant $C \ge 64$ such that the following holds. Suppose a mixture of $k$ spherical Gaussians has parameters $\{(\pi_j^*, \mu_j^*, \sigma_j^*): j \in [k]\}$ such that
    \begin{equation}
        \label{eq:pop_separation_condition}
        \forall i \neq j \in [k], \|\mu_i^* - \mu_j^*\| \ge C (\sigma_i^* \vee \sigma_j^*) \cdot \left( \sqrt{ \log k + \log(\rho_\sigma \rho_\pi)} \right),
    \end{equation}
    and suppose the mean initialization $\mu_1^{(0)}, ..., \mu_k^{(0)}$ satisfies
    \begin{equation}
        \label{eq:pop_initialization_condition_means}
        \forall i \in [k], \|\mu_i^{(0)} - \mu_i^*\| \le \frac{\sigma_i^*}{16} \min_{i \neq j} \|\mu_i^* - \mu_j^*\| / (\sigma_i^* \vee \sigma_j^*).
    \end{equation}
    Also, suppose the mixing weights and variances are initialized such that
    \begin{align}
        \label{eq:pop_initialization_others}
        &\forall i \in [k], |\pi_i^{(0)} - \pi_i^*| \le \pi_i^* / 2, \  |(\sigma_i^{(0)})^2 - {\sigma_i^*}^2 | \le 0.5 {\sigma_i^*}^2 / \sqrt{d}.
    \end{align}
    Then, population EM converges in $T = O(\log(1/\epsilon))$ iterations to the true solution such that for all $i\in [k]$, we have $\|\mu_i^{(T)} - \mu_i^*\| \le \sigma_i^* \epsilon$, $|\pi_i^{(T)} - \pi_i^*| / \pi_i^* \le \epsilon$, and $|(\sigma_i^{(T)})^2 - {\sigma_i^*}^2| / {\sigma_i^*}^2 \le \epsilon / \sqrt{d}$.
\end{theorem}
Note that the convergence rate is linear and does not depend on instance-dependent quantities. Furthermore, for estimating variances, we eventually get $O(\epsilon/\sqrt{d})$ accurate estimates. Hence, the final output parameters are also $O(\epsilon)$-close to the true mixture distribution in total variation distance. 

The proof of Theorem \ref{theorem:gmm_pop_main} starts with the useful fact about the EM operator.
\begin{fact}
    \label{fact:fixed_point}
    True parameters $\{(\pi_i^*, \mu_i^*, \sigma_i^*): i \in [k]\}$ are the fixed point of the EM operator, {\it i.e.},
    \begin{align*}
        \pi_i^* = \E_\D[w_i^*], \ \mu_i^* = \E_\D[w_i^* X]/\E_\D[w_i^*], \ {\sigma_i^*}^2 = \E_\D[w_i^* \|X - \mu_i^*\|^2] / (d \E_\D[w_i^*]),
    \end{align*}
\end{fact}
where $w_i^*$ is the weight constructed with true parameters in the E-step. Then, we can represent the estimation error after one EM iteration with the following lemma.
\begin{lemma}
    \label{lemma:estimation_error}
    Define $\Delta_{w_i} = w_i - w_i^*$. The estimation errors after one EM iteration can be written as
    \begin{align}
        \label{eq:est_error_pi_mu}
        \pi_i^+ - \pi_i^* &= \E_\D[w_i] - \E_\D[w_i^*] = \E_\D[\Delta_{w_i}], \nonumber \\
        \mu_i^+ - \mu_i^* &= \E_\D [\Delta_{w_i} (X - \mu_i^*)] / \E_\D[w_i].
    \end{align}
\end{lemma}
From this lemma, we observe that errors are proportional to the difference in weights given by the E-step that are constructed with true parameters and current estimates respectively. We focus on the parameters for the first component $i = 1$ and omit the subscript $i$ when the context is clear. The next lemma is key; it defines what we call good samples.
\begin{lemma}
    \label{lemma:weight_good}
    Suppose $X$ comes from the $j^{th}$ component ($j\neq1$). Let $v = X - \mu_j^*$ (thus, $v \sim \mathcal{N}(0, {\sigma_j^*}^2 I_d)$) and $\beta = \R{j1}^2/(64(\sigma_1^* \vee \sigma_j^*)^2)$. Consider the following events:
    \begin{align}
        \label{event:good_j_neq_1}
        \Eps_{j,1} &:= \{-{R_{j1}^*}^2/5 \le \vdot{v}{\mu_j^* - \mu_1^*} \}, \nonumber \\
        \Eps_{j,2} &:= \{-{R_{j1}^*}^2 / 64 \le \vdot{v}{\Delta_{\mu_1}} \} \cap \{ \vdot{v}{\Delta_{\mu_j}} \le {(\sigma_j^* / \sigma_1^*)}^2 {R_{j1}^*}^2 / 64 \}, \nonumber \\
        \Eps_{j,3} &:= \left\{d \left(1 - 2\sqrt{\beta/d} \right) \le \|v\|^2/{\sigma_j^*}^2 \le d \left(1 + 2\sqrt{\beta/d} + 2 \beta/d \right) \right\}. 
    \end{align}
    If the above three events occur, then the E-step assigns exponentially small weight to the other components, {\it i.e.}, $w_1 \le (\pi_1/\pi_j) \exp (-\beta)$ and $w_1^* \le (\pi_1^* / \pi_j^*) \exp (-\beta)$. We call such $X$ a good sample. This good event happens with probability at least $1 - 5 \exp (-\beta)$.
\end{lemma}
Combined with the expression in \eqref{eq:est_error_pi_mu}, this lemma implies the following.
\begin{corollary}
    \label{corollary:error_bound_Dm_1over2}
    In the setting of Lemma \ref{lemma:weight_good}, for any $j \neq 1$ and $s \in \mathbb{S}^{d-1}$,
    \begin{align}
        \E_{\D_j} [\Delta_w] &\le \left( 3 (\pi_1^*/\pi_j^*) + 5 \right) \exp \left(-\beta \right), \nonumber \\
        |\E_{\D_j} [\Delta_w \vdot{v}{s}]| &\le \left(3 (\pi_1^* / \pi_j^*) \sigma_j^* + 5R_{j1}^* \right) \exp \left (-\beta \right).
    \end{align}
\end{corollary}
The corollary bounds the estimation errors coming from the $j^{th}$ component for some $j \neq 1$. We obtain this result by decomposing the error term as
$\E_{\D_j} [\Delta_w] = \E_{\D_j} [\Delta_w \indic_{\Eps_{good}}] + \E_{\D_j} [\Delta_w \indic_{\Eps_{good}^c}]$,
for estimation errors in mixing weights. Then we control the error from good samples by small $\Delta_w$, and control the error from bad samples with the small probability of event $P(\Eps_{good}^c)$. Finally we sum up all errors from all components, and bound this sum with the following lemma.
\begin{lemma}
    \label{lemma:error_bounds_Dm_ge_1over2}
    For well-separated mixtures of Gaussians that satisfy the separation condition \eqref{eq:pop_separation_condition},
    \begin{align}
        \label{ineq:sum_weight_error}
        \sum_{j\neq 1} (\pi_1^* + \pi_j^*) \R{j1}^q \exp \left(-\R{j1}^2 / 64\sigt{1}{j}^2 \right) &\le c_q {\sigma_1^*}^q \pi_1^*,
    \end{align}
     for $q \in \{0, 1, 2\}$ with sufficiently small absolute constants $c_q$. 
\end{lemma}

By combining these results, we can guarantee that after one single iteration, we get $\|\mu_i^+ - \mu_i^*\| \le 0.5 \sigma_i^*$, $|\pi_i^+ - \pi_i^*| \le 0.5 \pi_i^*$, and $|{\sigma_i^+}^2 - {\sigma_i^*}^2| \le 0.5 {\sigma_i^*}^2 / \sqrt{d}$ for all $i \in [k]$. Hence when $\|\mu_i - \mu_i^*\| \ge 0.5 \sigma_i^*$, the lemmas stated in the main text suffice to guarantee that we get improved estimators through the EM operation. The full proof in this case is given in Appendix \ref{appendix:population_em_Dm_ge_1over2}. 

When $\|\mu_i - \mu_i^*\| \le 0.5 \sigma_i^*$, we define 
\begin{equation}
    \label{eq:definte_Dm}
    D_m = \max_i \left[ \max \left(\|\mu_i - \mu_i^*\|/\sigma_i^*, |\pi_i - \pi_i^*| / \pi_i^*, \sqrt{d} |{\sigma_i^2} - {\sigma_i^*}^2| / {\sigma_i^*}^2 \right) \right], 
\end{equation}
and show that $D_m^+ \le \gamma D_m$ for some absolute constant $\gamma < 1$. We first get a more fine-grained expression for \eqref{eq:est_error_pi_mu} that is proportional to $D_m$ by differentiating $\Delta_w$ using the mean-value theorem. Then we use the (essentially) same lemmas to show that  $D_m^+ \le \gamma D_m$, which guarantees a linear convergence in $D_m$ to $\epsilon$ in $O(\log (1/\epsilon))$ steps. Since the proof involves significant algebraic manipulation, we defer the proof for this case to Appendix \ref{appendix:population_em_Dm_le_1over2}.

\subsection{Finite-Sample EM Analysis}
Now we move on to the finite-sample EM algorithm. For analysis purposes, we use the common variant of the iterative algorithm which is often referred to as the \emph{sample-splitting} scheme. This scheme divides the $n$ samples into $T$ batches of size $n/T$, and uses a new batch of samples in each iteration, which removes the probabilistic dependency between iterations. We note that the uniform concentration approach typically used to avoid sample-splitting often results in overly pessimistic statistical error. While the analysis with sample-splitting suffices for the optimal sample-complexity guarantee in this work, it will be interesting to remove this dependency in the future. We now state the main theorem for the finite-sample EM algorithm.
\begin{theorem}
\label{theorem:gmm_fin_main}
There exists a universal constant $C \ge 64$ such that the following holds. Suppose a mixture of $k$ spherical Gaussians has parameters $\{(\pi_j^*, \mu_j^*, \sigma_j^*): j \in [k]\}$ such that
\begin{equation}
    \label{eq:finite_separation_cond}
    \forall i \neq j \in [k], \|\mu_i^* - \mu_j^*\| \ge C(\sigma_i^* \vee \sigma_j^*) \cdot c \sqrt{\log (\rho_\sigma /\pi_{min})},
\end{equation}
with some given constant $c > 2$, and suppose the initializers $\{(\tpi_j^{(0)}, \tmu_j^{(0)}, \tsigma_j^{(0)}): j \in [k]\}$ satisfy
\begin{align}
    \label{eq:finite_initialization_mean}
    \forall i \in [k], \|\tmu_i^{(0)} - \mu_i^*\| \le \frac{\sigma_i^*}{16} \min_{i \neq j} \|\mu_i^* - \mu_j^*\| / (\sigma_i^* \vee \sigma_j^*), \\
    \label{eq:finite_initialization_others}
    \forall i \in [k], |\tpi_i^{(0)} - \pi_i^*| \le \pi_i^* / 2, \  |(\tsigma_i^{(0)})^2 - {\sigma_i^*}^2 | \le 0.5 {\sigma_i^*}^2 / \sqrt{d}.
\end{align}
Suppose we use $n \ge C' (d \pi_{min}^{-1} \log^2 (k^2T / \delta) / \epsilon^2)$ samples with sufficiently large universal constant $C'$ for every iteration. Then, sample-splitting finite-sample EM converges in $T = O(\log(1/\epsilon))$ iterations to the true solution such that for all $i\in [k]$, we have $\|\tmu_i^{(T)} - \mu_i^*\| \le \sigma_i^* \epsilon$, $|\tpi_i^{(T)} - \pi_i^*| \le \pi_i^* \epsilon$, and $|(\tsigma_i^{(T)})^2 - {\sigma_i^*}^2| \le {\sigma_i^*}^2 \epsilon / \sqrt{d}$, with probability at least $1 - \delta - T / n^{c-2} k^{30}$.
\end{theorem}
The separation condition \eqref{eq:finite_separation_cond} in the statement differs by the constant $c > 2$ from condition \eqref{eq:pop_separation_condition}, which we need to bound the failure probability by $1/n^{c}$. This inverse-polynomial failure probability arises from the concentration of the empirical probability of rare events. 

As mentioned earlier, our finite-sample analysis develops a concentration bound that handles good and bad samples separately. Note that without splitting the analysis for different events, our statistical error may be unnecessarily large. For instance, if we simply apply standard sub-Gaussian tail bounds to $\frac{1}{n} \sum_{i=1}^n w_{1,i} (X_i - \mu_1^*)$ for mean updates, we end up having dependency on the norm of $X_i - \mu_1^*$, which could be as large as $O(\max_{j\neq1} \|\mu_j^* - \mu_1^*\|)$. This dependence on the pairwise distances of the true means is present in much of the prior art. However, this approach overlooks the fact that the statistical properties of good samples and bad samples are very different. To overcome this issue, we adopt the main statistical tool introduced in \cite{kwon2019converges}.
\begin{proposition}[Proposition 5.3 in \cite{kwon2019converges}]
    \label{lemma:indic_prob_decompose}
    Let $X$ be a random $d$-dimensional vector, and $A$ be an event in the same probability space with $p = P(A) > 0$. Define random variable $Y = X|A$, {\em i.e.}, $X$ conditioned on event $A$, and $Z = \indic_{X \in A}$. Let $X_i, Y_i, Z_i$ be the i.i.d.\ samples from corresponding distributions. Then, the following holds, 
    \begin{align}
        \label{eq:conditional_decomposed_prob}
        P\Bigg( \Big\|\frac{1}{n} \sum_{i=1}^{n} X_i \indic_{X_i \in A} - &\E[X \indic_{X \in A}] \Big\| \ge t \Bigg) \le \max_{m \le n_e} P \left(\frac{1}{n} \left\|\sum_{i=1}^{m} (Y_i - \E[Y]) \right\| \ge t_1 \right) \nonumber \\
        & + P\left(\|\E[Y]\| \left|\frac{1}{n} \sum_{i=1}^n Z_i - p\right| \ge t_2 \right)
        + P\left(\left|\sum_{i=1}^n Z_i\right| \ge n_e+1\right),
    \end{align}
    for any $0 \le n_e \le n$ and $t_1 + t_2 = t$.
\end{proposition}
We give a short overview of how we use this proposition in the proof for mixing weights. The estimation error in mixing weight after one EM iteration is $\tpi_1^+ - \pi_1^*= (1/n \sum_{i=1}^n w_{1,i} - \E_\D [w_1]) + (\E_\D[w_1] - \E_\D[w_1^*]).$
The second term is $(\E_\D[w_1] - \E_\D[w_1^*]) = \pi_1^+ - \pi_1^* \le \gamma \pi_1^* D_m$, for $\gamma$ and $D_m$ as in \eqref{eq:definte_Dm}. This is the error after one population EM iteration. Hence, if we can show that the first term is less than $\epsilon \pi_1^*$, linear convergence is guaranteed with convergence rate $\gamma < 1$ until $D_m$ reaches the target statistical error $\epsilon$. 

Now we decompose a single random variable $w_{1,i}$ into several parts using disjoint indicator functions. Let us define $\Eps_{j,good} := \Eps_{j,1} \cap \Eps_{j,2} \cap \Eps_{j,3}$ where good events are as defined in \eqref{event:good_j_neq_1}, and let
\begin{align*}
    w_{1,i} = w_{1,i} \indic_{\Eps_{1}} + \sum_{j\neq 1} w_{1,i} \indic_{\Eps_j \cap \Eps_{j,good}} + w_{1,i} \indic_{\Eps_j \cap \Eps_{j,good}^c}.
\end{align*}
Then, we treat each decomposed $O(k)$ terms as distinct quantities, and find a statistical fluctuation of the sum of each term over samples using Proposition \ref{lemma:indic_prob_decompose}. 

For each term, the first key step is to find a sub-Gaussian or sub-exponential norm of random variables conditioned on each event. We use this to control the statistical fluctuation of the conditioned random variable on each event, which appears as the first term in \eqref{eq:conditional_decomposed_prob}. This mainly controls the concentration of good samples.
For bad events $\Eps_{j,good}^c$, it is either that there are no bad sample with probability $O(1/poly(n))$, or $O(1)$ bad samples with high probability if $1 / n^c < P(\Eps_{good}^c)$. This controls the concentration of bad samples. The full proof of Theorem \ref{theorem:gmm_fin_main} is given in Appendix \ref{appendix:finite_sample}.

%% file: Sample_Optimal_Algorithm.tex
\label{section:sample_optimality}
We now show we can learn well-separated Gaussian mixtures with nearly optimal sample complexity. 
\begin{theorem}
    \label{theorem:sample_optimal_learnability}
    Suppose a mixture of $k$ spherical Gaussians has parameters $\{(\pi_j^*, \mu_j^*, \sigma_j^*): j \in [k]\}$ such that
    \begin{equation}
        \label{eq:sample_opt_separation_cond}
        \forall i \neq j \in [k], \|\mu_i^* - \mu_j^*\| \ge C(\sigma_i^* \vee \sigma_j^*) \cdot c \sqrt{\log (\rho_\sigma /\pi_{min})},
    \end{equation}
    with a universal constant $C \ge 64$ and some given constant $c > 2$. Then there exists a (possibly inefficient) algorithm that for any $\epsilon > 0$, returns parameters $\{(\pi_j, \mu_j, \sigma_j): j \in [k] \}$ (up to permutation) 
    such that
    \begin{align*}
        | \pi_i^* - \pi_i | / \pi_i^* \le \epsilon, \| \mu_i^* - \mu_i \| / \sigma_i^* \le \epsilon, |\sigma_i^* - \sigma_i | / \sigma_i^* \le \epsilon / \sqrt{d}, \ \forall i \in [k],
    \end{align*}
    using $n = \tO(\pi_{min}^{-1}d/\epsilon^2 \vee k\pi_{min}^{-2}d)$ samples with high probability.
\end{theorem}
While a polynomial sample upper bound with $\Omega(\sqrt{\log k})$ separation previously has been established in \cite{regev2017learning}, our result guarantees the tightest sample complexity for $\epsilon = O(\pi_{min})$. In particular, our result implies that a trivial lower bound $\Omega (d\pi_{min}^{-1} /\epsilon^2)$ is indeed a tight upper bound in the well-separated regime. Furthermore, we do not impose any constraints on the norms of means, or require prior knowledge on mixing weights or variances. Hence Theorem \ref{theorem:sample_optimal_learnability} allows any possible realization of Gaussian mixture models that satisfies \eqref{eq:sample_opt_separation_cond}. 


We first show that we can relax the initialization conditions \eqref{eq:finite_initialization_mean}, \eqref{eq:finite_initialization_others} such that it is sufficient to have a good initialization {\it only} for mean parameters. It makes the connection to proper-learning significantly easier since then we do not need any requirement on $\G$ other than being close to $\G^*$ in TV distance. In contrast, \cite{regev2017learning} requires $\G$ to be close in TV distance as well as to have {\it all} parameters close to $\G^*$, which raises a technical challenge in connecting the proper-learning and parameter initialization. This initialization requirement results in a much higher sample complexity at least $\tO(\pi_{min}^{-2} k^3 d^3 \rho^{16})$ with an instance-specific parameter $\rho$.

\subsection{Better Initialization with the \texorpdfstring{$k$}{Lg}-Means Algorithm}
\begin{algorithm}[ht]
    \caption{One-Step $k$-means with Good Mean Initializers \label{alg:k-means}}
    {\bf Input:} $n$ i.i.d.\ samples from a mixture of well-separated Gaussians $\G^*$ with parameters $\{(\pi_j^*, \mu_j^*, \sigma_j^*), j \in [k] \}$, and initial estimate of means $\{\mu_i^{(0)}, \forall i \in [k]\}$, satisfying \eqref{eq:minimal_initialization}. 
    
    {\bf Output:} Good initializers for the EM algorithm satisfying \eqref{eq:finite_initialization_mean} and \eqref{eq:finite_initialization_others}. 
    \begin{itemize}[leftmargin=0.3in]
        \setlength\itemsep{0em}
        \item[1.] Using $n$ samples from mixtures of well-separated Gaussians, cluster points according to the rule: $C_i = \{ X: \|X - \mu_i\| \le \|X - \mu_{j}\|, \forall j \neq i\}$.
        \item[2.] For each cluster $C_i$, let $\pi_i = |C_i| / n$ and $\mu_i = mean(C_i)$, the average over all elements in $C_i$.
        \item[3.] For each cluster $C_i$, let samples in $C_i$ stand in any pre-defined order. Compute pairwise distances between all adjacent samples. Let $F(x)$ be a cumulative distribution function of chi-square with $d$ degrees of freedom. Collect all $|C_i| - 1$ computed values, and find $(\alpha_d |C_i|)^{th}$ quantity among them where $\alpha_d = F(d)$. Set ${\sigma_i}^2$ as the quantity divided by $2d$.
    \end{itemize}
\end{algorithm}
Algorithm \ref{alg:k-means} is essentially the $k$-means algorithm except for Step 3 which estimates the variances. The next lemma is critical: it says Algorithm \ref{alg:k-means} can help initialize EM.
\begin{lemma}
    \label{lemma:k_means_converge}
    Suppose we are given $\mu_1^{(0)}, ..., \mu_k^{(0)}$ such that 
    \begin{align}
        \label{eq:minimal_initialization}
        \|\mu_i^{(0)} - \mu_i^*\| \le \frac{1}{4} \min_{i \neq j} \|\mu_i^* - \mu_j^*\|,
    \end{align}
    where $\{(\pi_j^*, \mu_j^*, \sigma_j^*): j \in [k] \}$ are the parameters of well-separated mixtures of Gaussians. Then there exists a universal constant $C' > 0$ such that given with $n \ge C' (d\pi_{min}^{-1} \log^2 (k/\delta))$ samples, Algorithm \ref{alg:k-means} returns estimators satisfying $\|\mu_i - \mu_i^*\|/\sigma_i^* \le 4$, $|\pi_i - \pi_i^*|/\pi_i^* \le 0.5$, and $|{\sigma_i}^2 - {\sigma_i^*}^2| / {\sigma_i^*}^2 \le 0.5/\sqrt{d}$ with high probability.
\end{lemma}
Thus Algorithm \ref{alg:k-means} succeeds, as long as the initialization satisfies $1/4$-closeness to the true means, which significantly relaxes the condition required for the EM algorithm. 
The key elements of the proof of the lemma for mixing weights and means are reminiscent of the ideas we exploit in population EM, as the $k$-means algorithm can be viewed as a variant of EM with hard-label assignment in the E-step. 
Estimating variances is trickier since we need to get estimators as good as $O(1/\sqrt{d})$.
Controlling the noise to get a $O(1/\sqrt{d})$ estimate requires more than simply computing the average over each cluster (which is what we do in Step 2 to compute the means, and the mixing weights $\pi_i$). The key is contained in Step 3 of Algorithm \ref{alg:k-means}, which is essentially computing the median of pairwise distances. The full proof of Lemma \ref{lemma:k_means_converge} is given in Appendix \ref{Appendix:k_means_converge}.

\subsection{Sample-Optimal Algorithm for Theorem \ref{theorem:sample_optimal_learnability}}
In this section, we first show that if we find a proper hypothesis that is $\pi_{min}/4$-close in total variation distance, then all means in the hypothesis are $1/4$-close to true means of well-separated Gaussian mixtures. The is the content of the following lemma.
\begin{lemma}
    \label{lemma:tv_implies_param}
    Suppose a mixture of well-separated Gaussians $\G^*$ with parameters $\{ (\pi_i^*, \mu_i^*, \sigma_i^*), \forall i \in k \}$, and a candidate mixture of (any spherical) Gaussians $\G$ with parameters $\{ (\pi_i, \mu_i, \sigma_i), \forall i \in k \}$, satisfy $\| \G - \G^* \|_{TV} \le \pi_{min} / 4$. Then
    \begin{align}
        \label{eq:tv_to_initializetion}
        \max_i (\min_j \|\mu_i^* - \mu_j\| / \sigma_i^*) \le 16 \sqrt{\log(1/\pi_{min})}.
    \end{align}
\end{lemma}
The proof of Lemma \ref{lemma:tv_implies_param} is given in Appendix \ref{Appendix:tv_implies_param}. While similar connection between TV distance and parameter distance are stated in \cite{diakonikolas2017statistical, regev2017learning}, they consider distances smaller than $o(1/k^2)$ either in parameter space or total variation.  Matching the sample complexity lower bounds requires a simpler connection to $O(1/k)$-TV distance. As mentioned earlier, $\G$ may completely ignore some components if $\|\G - \G^*\|_{TV} > \pi_{min}$, hence $\pi_{min}/4$ (which is $O(1/k)$) is order-wise the minimum possible requirement for this approach. 

Now we recall that \cite{ashtiani2018nearly} provide a sample-optimal guarantee $\tO(kd/\epsilon^2)$ for the proper-learning of a mixture of {\it axis-aligned} Gaussians, {\it i.e.}, Gaussians with diagonal covariance matrices. From their result, it is straightforward to get the same sample-optimal guarantee for spherical Gaussian mixtures. We then combine it with Lemma \ref{lemma:tv_implies_param}, which gives a candidate distribution $\G$ that satisfies $\|\G - \G^*\|_{TV} \le \pi_{min}/4$. Lemma \ref{lemma:k_means_converge} provides the final bridge: executing Algorithm \ref{alg:k-means} produces an initialization good enough for the EM algorithm, as guaranteed by Theorem \ref{theorem:gmm_fin_main}. The full proof of \ref{theorem:sample_optimal_learnability} is given in Appendix \ref{Appendix:proof_sample_compression}.

\subsection{Discussion on Initialization}
We have shown that $\tO(d\pi_{min}^{-1} / \epsilon^2)$ samples are sufficient for learning a mixture of well-separated spherical Gaussians. While the combination of proper learning and the EM algorithm gives a sample-optimal algorithm, we note, however, that the algorithm given in \cite{ashtiani2018nearly} is not computationally efficient. In fact, as far as we know, no (orderwise) sample-optimal polynomial time algorithm is known for proper-learning even in the $\Omega(\sqrt{\log k})$-separated regime. Moreover, for general (non-spherical) Gaussian mixtures, work in \cite{diakonikolas2017statistical,diakonikolas2018list} gives statistical-query based lower bounds, though it is an interesting question to explore what these results imply in the spherical setting. Hence, while our work resolves the sample-complexity question, the computational complexity remains open. 

The state-of-the-art algorithms in the $\Omega(\sqrt{\log k})$ separation regime appear in \cite{diakonikolas2018list, hopkins2018mixture, kothari2018robust}. They run in quasi-polynomial time with the required sample complexity $\tO(poly(k^{\log k}, d, 1/\epsilon))$. However, it is still unknown whether it is possible to learn a mixture of well-separated Gaussians in polynomial time with polynomial sample complexity. Achieving polynomial running time with polynomial sample complexity would be a great result that resolves a long lasting open problem in literature.

%% file: Conclusion.tex
We provide  local convergence guarantees for the EM algorithm for a mixture of well-separated spherical Gaussians. We show that EM enjoys desirable local convergence properties in several respects: minimal requirements on the separation condition, optimal sample complexity, and large initialization region. Consequently, our results provide the optimal sample upper bound for learning the parameters of well-separated Gaussian mixture models. Even under structural assumptions or larger sample complexity regimes when other methods apply, EM may still be an appealing local algorithm to amplify the estimation accuracy, as these global algorithms tend to incur large and instance-dependent statistical error. While our analysis is restricted to the well-separated regime, we conjecture that EM locally converges to the true parameters even with smaller separation. It will be an interesting future challenge to establish a similar result in a weaker separation regime. 

%% file: Appendix.tex
\section{Useful Technical Lemmas}
We state some useful lemmas:
\begin{lemma}[Sub-Gaussian tails \cite{vershynin2010introduction}]
    Let $v$ be a sub-Gaussian random vector with parameter $\sigma^2$ in $d$-dimensional space. Then for any unit vector $s \in \mathbb{S}^{d-1}$ and $\alpha > 0$,
    \begin{align*}
        P(\vdot{v}{s} \ge \alpha) \le \exp(-\alpha^2 / 2\sigma^2).
    \end{align*}
\end{lemma}
\begin{lemma}[chi-Square tails \cite{laurent2000adaptive}]
    \label{lemma:chi_square_laurent}
    Let $v$ be a chi-square random variable with $d$ degrees of freedom. Then for any $\alpha > 0$,
    \begin{align*}
        P(v \ge d + 2\sqrt{d\alpha} + 2\alpha) \le \exp(-\alpha), \\
        P(v \le d - 2\sqrt{d\alpha}) \le \exp(-\alpha).
    \end{align*}
\end{lemma}
\begin{lemma}[Lower bound for chi-Square tails \cite{inglot2010inequalities}]
    \label{lemma:chi_square_inglot}
    Let $v$ be a chi-square random variable with degree of freedom $d \ge 2$. Then for any $u \ge d-1$,
    \begin{align*}
        \frac{1 - e^{-2}}{2} \frac{u}{u - d + 2\sqrt{d}} \Eps_d(u) \le P(v \ge u) \le \frac{1}{\sqrt{\pi}} \frac{u}{u - d + 2} \Eps_d (u),
    \end{align*}
    where $\Eps_d(u) = \exp(-(1/2)(u - d - (d-2)\log u + (d-1) \log d) )$.
\end{lemma}

\begin{lemma}[Sub-Gaussian norm \cite{vershynin2010introduction}]
    $v$ is a sub-Gaussian random vector in $d$-dimensional space if and only if for any unit vector $s \in \mathbb{S}^{d-1}$, there exists a finite value $K > 0$ such that
    \begin{align*}
        \sup_{p \ge 1} p^{-1/2} \E[ |\vdot{v}{s}|^p ]^{1/p} \le K.
    \end{align*}
    We denote the sub-Gaussian norm of $v$ as $\|v\|_{\psi_2} \le K$. Furthermore, the tail probability is bounded by
    \begin{align*}
        P(\vdot{v}{s} \ge t) \le \exp(-c t^2 / K^2),
    \end{align*}
    for some universal constant $c > 0$.
\end{lemma}
\begin{lemma}[Sub-exponential norm \cite{vershynin2010introduction}]
    $X$ is a sub-exponential random variable if and only if there exists a finite value $K > 0$ such that
    \begin{align*}
        \sup_{p \ge 1} p^{-1} \E[ |X|^p ]^{1/p} \le K.
    \end{align*}
    We denote the sub-exponential norm of $v$ as $\|v\|_{\psi_1} \le K$. Furthermore, the tail probability is bounded by
    \begin{align*}
        P(X \ge t) \le \exp(-c \min(t/K, t^2 / K^2)),
    \end{align*}
    for some universal constant $c > 0$.
\end{lemma}

\begin{lemma}
    \label{lemma:vp_conditioned_vu}
    Suppose $v \sim \mathcal{N}(0, I_d)$, $u, s \in \mathbb{S}^{d-1}$ are any fixed unit vectors, and $\alpha > 0$ is some constant.
    Then the following holds:
    \begin{align*}
        \E_{v \sim \mathcal{N}(0, I_d)} [ |\vdot{v}{s}|^p | |\vdot{v}{u}| \ge \alpha] \le (2\alpha)^{p} + 4 \alpha \exp(-\alpha^2/2) (2p)^{p/2}.
    \end{align*}
\end{lemma}
\begin{proof}
    Let us first decompose $s = s_u + s_p$, where $s_u$ is parallel to $u$ and $s_p$ is orthogonal to $u$. We can rewrite the target quantity as
    \begin{align*}
        \left( \frac{\E[|\vdot{v}{s_u} + \vdot{v}{s_p}|^p \indic_{\vdot{v}{u} \ge \alpha}]} {P(\vdot{v}{u} \ge \alpha)} \right)^{1/p} &\le \frac{ \E[|\vdot{v}{s_u}|^p \indic_{\vdot{v}{u} \ge \alpha}]^{1/p} + \E[|\vdot{v}{s_p}|^p \indic_{\vdot{v}{u} \ge \alpha}]^{1/p} } {P(\vdot{v}{u} \ge \alpha)^{1/p}} \\
        &= \|s_u\| \left( \E[|\vdot{v}{u}|^p | \vdot{v}{u} \ge \alpha] \right)^{1/p} + \E[|\vdot{v}{s_p}|^p]^{1/p}.
    \end{align*}
    Thus, it boils down to upper-bound $\E_{Z \sim \mathcal{N}(0,1)} [Z^p | Z \ge \alpha]$. This can be bounded by
    \begin{align*}
        \E_Z [Z^p \indic_{Z \ge \alpha}] / P(Z \ge \alpha) &= \E_Z [Z^p \indic_{2\alpha \ge Z \ge \alpha}] / P(Z \ge \alpha) + \E_Z [Z^p \indic_{Z \ge 2\alpha}] / P(Z \ge \alpha) \\
        &\le (2\alpha)^p + \sqrt{\E_Z[Z^{2p}]} \sqrt{P(Z \ge 2\alpha)} / P(Z \ge \alpha) \\
        &\le (2\alpha)^p + 2^{p/2} \Gamma(p+1/2)^{1/2} / \sqrt[4]{\pi} \exp(-\alpha^2) / (\exp(-\alpha^2/2) / (\sqrt{2\pi} 2 \alpha)) \\
        &\le (2\alpha)^p + 4 (2^{p/2}) p^{p/2} \alpha \exp(-\alpha^2/2).
    \end{align*}
\end{proof}

\begin{lemma}
    \label{lemma:chi_square_pnorm_bound}
    Let $v \sim \mathcal{N}(0, I_d)$. Then for any $p \ge 1$,
    \begin{align*}
        \E[\|v\|^p] = 2^{p/2} \Gamma \left((p+d)/2\right) / \Gamma (d/2) \le (d + p)^{p/2}.
    \end{align*}
\end{lemma}
\begin{proof}
    The first equality is standard and given in \cite{yan2017convergence}. From here, we can proceed as
    \begin{align*}
        \Gamma((p+d)/2) &= (d/2 + p/2)(d/2 + p/2 - 1) ... (d/2+\alpha+1) \Gamma(d/2 + \alpha) \\
        &\le ((d+p)/2)^{[p/2]} \Gamma(d/2 + \alpha),
    \end{align*}
    where $\alpha = p/2 - [p/2] \in [0,1)$ and $[p/2]$ is the largest integer that does not exceed $p/2$. We then use Gautschi's inequality for the ratio of Gamma functions \cite{wendel1948note} which states
    \begin{align*}
        \Gamma(x+1) / \Gamma(x+s) \le (x+1)^{1-s}.
    \end{align*}
    Applying this with $x+s = d/2$ and $x+1 = d/2 + \alpha$, we get
    \begin{align*}
        ((d+p)/2)^{[p/2]} \Gamma(d/2 + \alpha)/\Gamma(d/2) &\le ((d+p)/2)^{[p/2]} (d+\alpha)^{\alpha} \le ((d+p)/2)^{p/2}.
    \end{align*}
    Thus, we can conclude 
    \begin{align*}
        \E[\|v\|^p] = 2^{p/2} \Gamma \left((p+d)/2\right) / \Gamma (d/2) \le (d + p)^{p/2}.
    \end{align*}
\end{proof}

The following lemmas are the upper bound on the $L_p$ norm of random variables conditioned on some events. We use them in an important way, as they help us bound expected errors from bad events.
\begin{lemma}
    \label{lemma:vl2_conditioned_vl2}
    Let $v \sim \mathcal{N}(0, I_d)$. Let $p \ge 1$ and $r^2 \ge d$. Then,
    \begin{align*}
        \E[\|v\|^p | \|v\|^2 \ge r^2] \le (2r)^p + 4(d + 2p)^{p/2} \exp(-r^2/8).
    \end{align*}
\end{lemma}
\begin{proof}
    Similarly to the one-dimensional case, we start with
    \begin{align*}
        \E[\|v\|^p | \|v\|^2 \ge r^2] &\le (2r)^p + \sqrt{E[\|v\|^{2p}]} \frac{\sqrt{P(\|v\|^2 \ge 4r^2)}}{P(\|v\|^2 \ge r^2)} \\
        &\le (2r)^p + (d+2p)^{p/2} \frac{\sqrt{P(\|v\|^2 \ge 4r^2)}}{P(\|v\|^2 \ge r^2)}.
    \end{align*}
    We can use the inequalities for lower and upper bounds for the tail probability of a chi-square distribution from Proposition 3.1 in \cite{inglot2010inequalities}. From the inequality, \begin{align*}
        P( \|v\|^2 \ge 4r^2 ) &\le 2 \exp \left( -\frac{1}{2} (4r^2 - d - (d-2) \log(4r^2) + (d-1) \log d)) \right), \\
        P( \|v\|^2 \ge r^2 ) &\ge \frac{1}{2} \exp \left( -\frac{1}{2} (r^2 - d - (d-2) \log(r^2) + (d-1) \log d)) \right).
    \end{align*}
    Using this, we can bound
    \begin{align*}
        \frac{\sqrt{P( \|v\|^2 \ge 4r^2 )}}{P(\|v\|^2 \ge r^2)} &\le 2\sqrt{2} \exp \left( -\frac{1}{2} r^2 - \frac{1}{4} (d + (d-2) \log(r^2)) + \frac{1}{4} ((d-2) \log (4d) + \log d) )) \right), \\
        &\le 2\sqrt{2} \exp(-r^2/8).
    \end{align*}
    Plugging in this relation, the lemma follows.
\end{proof}

\begin{lemma}
    \label{lemma:vp_conditioned_vl2}
    Let $v \sim \mathcal{N}(0, I_d)$. Let $p \ge 1$ and $r^2 = d + 2\sqrt{\alpha d} + 2\alpha$ with $\alpha > 8$. Then for any fixed unit vector $s \in \mathbb{S}^{d-1}$,
    \begin{align*}
        \E[|\vdot{v}{s}|^p | \|v\|^2 \ge r^2] \le (64 \alpha)^{p/2} + 4(8\alpha + 2p)^{p/2}.
    \end{align*}
\end{lemma}
\begin{proof}
    Due to the rotational invariance of standard Gaussian distribution, without loss of generality, we can investigate $\E[|v_1|^p | \|v\|^2 \ge r^2]$. Let $a = r^2 - d$ and $b = a - (d-2) \log(1 + a/d)$. We first look at the case when $\alpha < d/8$ (thus $d > 64$ if $\alpha > 8$). In this case, first observe that $a/d = (2\sqrt{\alpha d} + 2\alpha) / d < 1$ and,
    \begin{align*}
        a^2/d &= 4(\alpha d + \alpha^2 + 2\alpha \sqrt{\alpha d})/d = 4(\alpha + \alpha/8 + \alpha/\sqrt{2}) \le 8\alpha, \\
        b &= a - (d-2) \log(1+ a/d) \le a - (d-2) (a/d - (a/d)^2/6) \\
        &\le 2a/d + (a^2/6d) (d-2)/d \le 2\alpha.
    \end{align*}
    Then we change the quantity as
    \begin{align*}
        \frac{\E[|v_1|^p \E[\indic_{\|v\|^2 \ge r^2 - v_1^2} | v_1] ] }{ P(\|v\|^2 \ge r^2) } &= \frac{\E[|v_1|^p P(\|u\|^2 \ge r^2 - v_1^2) ] }{ P(\|v\|^2 \ge r^2) },
    \end{align*}
    where $u \sim \mathcal{N}(0, I_{d-1})$. Continuing the process,
    \begin{align*}
        \frac{\E[|v_1|^p P(\|u\|^2 \ge r^2 - v_1^2) ] }{ P(\|v\|^2 \ge r^2) } &= \frac{\E[|v_1|^p \indic_{v_1^2 \ge b}] + \E[|v_1|^{p} \indic_{v_1^2 \le b}  P(\|u\|^2 \ge r^2 - v_1^2) ] }{ P(\|v\|^2 \ge r^2) } \\
        &\le \frac{\E[|v_1|^p|v_1^2 \ge b] P(v_1^2 \ge b) + \E[|v_1|^{p} \indic_{v_1^2 \le b} P(\|u\|^2 \ge r^2 - v_1^2) ] }{ P(\|v\|^2 \ge r^2) }. 
    \end{align*}
    
    Now using Lemma \ref{lemma:chi_square_inglot}, we have
    \begin{align*}
        P(\|v\|^2 \ge r^2) &\ge \frac{1-e^{-2}}{2} \frac{r^2}{r^2 - d + 2\sqrt{d}} \exp \left(-\frac{1}{2} (r^2 - d - (d-2)\log(r^2) + (d-1)\log d) \right), \\
        &\ge \frac{1-e^{-2}}{4} \exp \left(-\frac{1}{2} (a - (d-2)\log(r^2/d)) \right),
    \end{align*}
    and
    \begin{align*}
        P(\|u\|^2 \ge r^2 - v_1^2) &\le \frac{1}{\sqrt{\pi}} \frac{r^2 - v_1^2}{r^2 - v_1^2 - d + 1} \exp \left(-\frac{1}{2} (r^2 - v_1^2 - (d-1) - (d-3)\log(r^2 - v_1^2) + (d-2)\log d) \right) \\
        &= \frac{1}{\sqrt{\pi}} \frac{1}{a - v_1^2 + 1} \exp \left(-\frac{1}{2} (1 + a - v_1^2 - (d-2)\log((r^2 - v_1^2) / d)) \right).
    \end{align*}
    Therefore, we have that
    \begin{align*}
        \frac{\E[|v_1|^{p} \indic_{v_1^2 \le b} P(\|u\|^2 \ge r^2 - v_1^2) ]}{ P(\|v\|^2 \ge r^2) } &\le \frac{4 }{\sqrt{\pi}(1 - e^{-2}) \sqrt{2\pi}} \int_{-\sqrt{b}}^{\sqrt{b}} \frac{|v_1|^{p-1} |v_1|}{1+a-v_1^2} \exp \left(-\frac{1}{2} (1 - (d-2) \log((r^2 - v_1^2)/r^2) \right) \\
        &\le \frac{8 b^{(p-1)/2}}{\pi} \int_{0}^{\sqrt{b}} \frac{v_1}{1+a-v_1^2} \exp \left(-\frac{1}{2} + \frac{d-2}{2} \log(1 - v_1^2/r^2) \right) \\
        &\le \frac{8 b^{(p-1)/2}}{\pi} \int_0^{\sqrt{b}} \frac{v_1}{1+a-v_1^2} \le \frac{-4b^{(p-1)/2}}{\pi} \ln(1 + a - v_1^2)|_{0}^{\sqrt{b}} \\
        &\le \frac{4b^{(p-1)/2}}{\pi} \frac{1 + a}{1 + (d-2)\log (1 + a/d)}.
    \end{align*}
    
    Recall $a < d$, $\log(1 + a/d) \ge a/d - (a/d)^2 / 6$. Plugging this in, we have 
    \begin{align*}
        \frac{4b^{(p-1)/2}}{\pi} \frac{1 + a}{1 + (d-2)\log (1 + a/d)} &\le \frac{4b^{(p-1)/2}}{\pi} \frac{1+a}{1 + a(d-2)/d - (a/d)^2/6} \\
        &\le 4b^{(p-1)/2} \le 4 (2\alpha)^{(p-1)/2} \le (2\alpha)^{p/2}.
    \end{align*}
    On the other hand, we know from above that,
    \begin{align*}
        \E[|v_1|^p | v_1^2 \ge b] \le (2\sqrt{b})^p + (2p)^{p/2} \le (8 \alpha)^{p/2} + 4 \sqrt{b}\exp(-b/2)  (2p)^{p/2}.
    \end{align*}
    Furthermore, we know that 
    \begin{align*}
        \frac{P(v_1^2 \ge b)}{P(\|v\|^2 \ge r^2)} \le \frac{2 \exp(-b/2)}{(1-e^{-2})/2 (\sqrt{d}/2) \exp(-b/2)} \le 1/(1-e^{-2}),
    \end{align*}
    where we used that $d > 8\alpha > 64$. Thus, when $\alpha < d/8$, we can conclude that
    \begin{align*}
        \E[|v_1|^p | \|v\|^2 \ge r^2] \le 2 (8\alpha)^{p/2} + 4(2p)^{p/2}.
    \end{align*}
    
    Now let us consider the other side, when $\alpha > d/8$. In this case, we simply use Lemma 3.8. Observe that
    \begin{align*}
        \E[|v_1|^p | \|v\|^2 \ge r^2] &\le \E[\|v\|^p | \|v\|^2 \ge r^2] \le (2r)^p + 4(d + 2p)^{p/2} \exp(-r^2 / 8).
    \end{align*}
    Also, we can observe that when $\alpha > d/8$,
    \begin{align*}
        b &\le a = 2\sqrt{\alpha d} + 2\alpha \le (4\sqrt{2} + 2) \alpha \le 8\alpha, \\
        2\alpha &\le r^2 = d + a \le 16 \alpha. 
    \end{align*}
    Therefore, we can apply these as is to obtain
    \begin{align*}
        \E[|v_1|^p | \|v\|^2 \ge r^2] &\le (64\alpha)^{p/2} + 4(8 \alpha + 2p)^{p/2} \exp(-\alpha/4).
    \end{align*}
    
    Now combining all inequalities, we can conclude that
    \begin{align*}
        \E[|\vdot{v}{s}|^p | \|v\|^2 \ge r^2] \le (64\alpha)^{p/2} + 4 (8\alpha + 2p)^{p/2}.
    \end{align*}
\end{proof}

\begin{lemma}
    \label{lemma:vl2_conditioned_vs}
    Let $v \sim \mathcal{N}(0, I_d)$. Then for any $p \ge 1$ and for any fixed unit vector $s \in \mathbb{S}^{d-1}$,
    \begin{align*}
        \E[\|v\|^p | |\vdot{v}{s}| \ge \alpha]^{1/p} \le (2\alpha) + (2p)^{1/2} + (d+p-1)^{1/2}. 
    \end{align*}
\end{lemma}
\begin{proof}
    It can be easily shown that
    \begin{align*}
        \E[\|v\|^p | \vdot{v}{s} \ge \alpha]^{1/p} &= \E[(v_1^2 + \|v_{2:d}\|^2)^{p/2} | v_1 \ge \alpha]^{1/p} \\
        &\le \E[(|v_1| + \|v_{2:d}\|)^p | v_1 \ge \alpha]^{1/p} \\
        &\le \E[|v_1|^p | v_1 \ge \alpha]^{1/p} + \E[\|v_{2:d}\|^p | v_1 \ge \alpha]^{1/p} \\
        &\le (2\alpha) + (2p)^{1/2} + \E[\|v_{2:d}\|^p]^{1/p} \\
        &\le (2\alpha) + (2p)^{1/2} + (d+p-1)^{1/2}.
    \end{align*}
    Here we used the fact that $\|X\|_2 \le \|X\|_1$ for any vector $X \in R^2$ in the first step, Minkowski's inequality in the second step, and applied the previous lemmas.
\end{proof}

\section{Proof for the Convergence of the Population EM when \texorpdfstring{$D_m \ge 1/2$}{Lg}.}
\label{appendix:population_em_Dm_ge_1over2}
\subsection{Proofs for Fact \ref{fact:fixed_point} and Lemma \ref{lemma:estimation_error}}
\begin{proof} {\it for Fact \ref{fact:fixed_point}.}
    Note that for each sample $X$, $w_i^*(X) = \pi_i^* f_{\D_i}(X) / f_\D(X)$ where $f_\D$ is a p.d.f.\ of the mixture distribution, and $f_{\D_i} (X)$ is a p.d.f.\ of the $i^{th}$ component. Then, 
    \begin{align*}
        \E_\D [w_i^*] &= \int w_i^*(X) f_\D(X) = \int \pi_i^* f_{\D_i} (X) = \pi_i^*, \\
        \E_\D [w_i^* X] &= \int w_i^*(X) X f_\D(X) = \int \pi_i^* X f_{\D_i} (X) = \pi_i^* \mu_i^*. 
    \end{align*}
    Dividing by $\pi_i^*$ recovers $\mu_i^*$. For variances,
    \begin{align*}
        \E_\D [w_i^* \|X - \mu_i^*\|^2] &= \int \pi_i^* \|X - \mu_i^*\|^2 f_{\D_i} (X) = \pi_i^* d {\sigma_i^*}^2.
    \end{align*}
    Dividing by $d \E_\D[w_i^*] = d \pi_i^*$ gives ${\sigma_i^*}^2$.
\end{proof}

\begin{proof} {\it for Lemma \ref{lemma:estimation_error}.}
    For mixing weights, $\pi_i^+ = \E_\D [w_i]$ by construction, $\pi_i^* = \E_\D [w_i^*]$ by Fact \ref{fact:fixed_point}. Hence it is obvious that $\pi_i^+ - \pi_i^* = \E_\D [\Delta_{w_i}]$.
    
    For means, note that $\mu_i^+ = \E_\D[w_iX] / \E_\D[w_i]$. Then,
    \begin{align*}
        \mu_i^+ - \mu_i^* &= \E_\D[w_iX] / \E_\D[w_i] - \mu_i^* = (\E_\D[w_iX] - \E_\D[w_i] \mu_i^*) / \E_\D[w_i] \\
        &= (\E_\D[w_i X] - \E_\D[w_i \mu_i^*]) / \E_\D [w_i] \\
        &= \E_\D[w_i (X - \mu_i^*)] / \E_\D [w_i].
    \end{align*}
    Now note that 
    \begin{align*}
        \E_\D[w_i^* (X - \mu_i^*)] = \E_\D[w_i^* X] - \E[w_i^*] \mu_i^* = \mu_i^* \pi_i^* - \pi_i^* \mu_i^* = 0.
    \end{align*}
    Hence we prove that $\mu_i^+ - \mu_i^* = \E_\D[\Delta_{w_i} (X - \mu_i^*)] / \E_\D [w_i]$.
\end{proof}

\subsection{Proof of Lemma \ref{lemma:weight_good}}
\begin{proof}
    Let us examine the exponents in $w_1$. By definition of the weight constructed in the E-step, we can see that
    \begin{align*}
        w_1 &\le \frac{\pi_1}{\pi_j} \exp \left( -\frac{\|v + \mu_j^* - \mu_1\|^2}{2\sigma_1^2} + \frac{\|v + \mu_j^* - \mu_j\|^2}{2\sigma_j^2} - \frac{d}{2} \log(\sigma_1^2 / \sigma_j^2) \right).
    \end{align*}
    Our goal is to find conditions for good event where $w_1 \le \exp(\cdot)$. The {\it sufficient} condition for this is when the sum of these three terms is smaller than $-C \log (R_{j1}^* k \rho_\pi / \sigma_1^*)$. If this is the case, then the weight given to this sample is less than $O(\sigma_1^*/(k R_{j1}^*))$ which can cancel out errors from the mismatch in labels. The cases are divided based on whether $\sigma_1 \ge \sigma_j$ or $\sigma_1 \le \sigma_j$. We first rearrange the inside of the exponent,
    \begin{align}
        \label{eq:exponent_dissect}
        -&\frac{\|v + \mu_j^* - \mu_1\|^2}{2\sigma_1^2} + \frac{\|v + \Delta_j\|^2}{2\sigma_j^2} - \frac{d}{2} \log(\sigma_1^2/\sigma_j^2) \nonumber \\  
        &= -\frac{\|v + \mu_j^* - \mu_1\|^2}{2\sigma_1^2} + \frac{\|v + \Delta_j\|^2}{2\sigma_j^2} - \frac{d}{2} \log(\sigma_1^2 / \sigma_j^2) \nonumber \\
        &= -\frac{\|v\|^2 + \|\mu_j^* - \mu_1\|^2 + 2 \vdot{v}{\mu_j^* - \mu_1}}{2\sigma_1^2} + \frac{\|v\|^2 + 2\vdot{v}{\Delta_j} + \|\Delta_j\|^2}{2\sigma_j^2} - \frac{d}{2} \log(\sigma_1^2 / \sigma_j^2) \nonumber \\
        &\le \underbrace{-\frac{7{R_{j1}^*}^2} {16\sigma_1^2}}_{I} + \underbrace{ \frac{\|\Delta_j\|^2}{2\sigma_j^2} }_{II}
        + \underbrace{\left( -\vdot{v}{\mu_j^* - \mu_1} / \sigma_1^2  \right)}_{III} + \underbrace{\vdot{v}{\Delta_j} / \sigma_j^2}_{IV} + \underbrace{ \left( -\frac{\|v\|^2}{2} (\frac{1}{\sigma_1^2} - \frac{1}{\sigma_j^2}) -\frac{d}{2} \log(\sigma_1^2 / \sigma_j^2) \right) }_{V}.
    \end{align}
    $II$, $III$ and $IV$ can be controlled in a fairly straight-forward manner. Check that,
    \begin{align*}
        II &\le \frac{\|\Delta_j\|^2}{2\sigma_j^2} \le
        \frac{ {R_{j1}^*}^2}{512 (\sigma_1^* \vee \sigma_j^*)^2} \frac{{\sigma_j^*}^2}{\sigma_j^2} \le 
        \frac{3 {R_{j1}^*}^2}{512 {\sigma_1}^2},
    \end{align*}
    
    \begin{align*}
        P \left(III \ge 7{R_{j1}^*}^2 / 32 {\sigma_1}^2 \right) &= P \left(\vdot{v}{\mu_j^* - \mu_1} \le -7{R_{j1}^*}^2 / 32 \right) \\
        &\le P \left(\vdot{v}{\mu_j^* - \mu_1^*} \le -{R_{j1}^*}^2 / 5 \right) + P \left(\vdot{v}{\Delta_1} \le -{R_{j1}^*}^2 / 64 \right) \\
        &\le 2 \exp \left( - {R_{j1}^*}^2 /(64 {\sigma_j^*}^2) \right),
    \end{align*} 
    \begin{align*}
        P \left(IV \ge {R_{j1}^*}^2 / 20 {\sigma_1}^2 \right) &\le
        P \left(\vdot{v}{\Delta_j} \ge {R_{j1}^*}^2 {\sigma_j}^2 / 20 {\sigma_1}^2 \right) \\
        &\le P \left(\vdot{v}{\Delta_j} \ge {R_{j1}^*}^2 (\sigma_j^* / \sigma_1^*)^2 / 64) \right) \le \exp \left(-{R_{j1}^*}^2 / (64{\sigma_1^*}^2) \right),
    \end{align*} 
    Thus, $I$ + $II$ + $III$ + $IV$ is smaller than $({R_{j1}^*}^2/{\sigma_1}^2) (-7/16 + 7/32 + 1/20 + 3/512) \le -5/32$ with high probability. Now the remaining main challenge is to bound $V$. We should consider cases separately when $\sigma_1 \ge \sigma_j$ and $\sigma_1 \le \sigma_j$. 
    
    Let us first consider $\sigma_1 \ge \sigma_j$. Overall, we want $V \le {R_{j1}^*}^2/(8 \sigma_1^2)$ so that the entire sum inside exponent is small enough to kill this sample. That is, we want that
    
    \begin{align}
        \label{eq:condition_V}
        -\frac{\|v\|^2}{2} \left( \frac{1}{\sigma_1^2} - \frac{1}{\sigma_j^2} \right) - \frac{d}{2} \log(\sigma_1^2 / \sigma_j^2) \le {R_{j1}^*}^2 / (8 {\sigma_1}^2).
    \end{align}
    
    Let us introduce some auxiliary variables to simplify the expression. Let $x = (\sigma_1^2 - \sigma_j^2) / \sigma_1^2$ and $a = {R_{j1}^*}^2 / (8 {\sigma_1}^2)$. Then the above equation can be written as
    \begin{align*}
        \frac{\|v\|^2}{2} (x/\sigma_j^2) + \frac{d}{2} \log(1 - x) \le a
        &\iff \frac{\|v\|^2}{\sigma_j^2} \le \frac{1}{x} (-d \log(1 - x) + 2a) \\
        &\iff \frac{1}{d} \frac{\|v\|^2}{{\sigma_j^*}^2} \le \frac{\sigma_j^2}{{\sigma_j^*}^2} \frac{1}{x} (-\log(1 - x) + 2a/d).
    \end{align*}

    Note that $\|v\|^2/{\sigma_j^*}^2$ is a degree-$d$ chi-square random variable and we can apply standard tail bound for CDF of $\chi_d$ distribution (which is essentially upper tail bound for sub-exponential random variable). The following useful inequality will help us to bound it in more convenient form:
    \begin{align*}
        -\log(1-x) \ge x + x^2 / 2,
    \end{align*}
    so that the sufficient condition for this is,
    \begin{align*}
        \frac{1}{d} \frac{\|v\|^2}{{\sigma_j^*}^2} \le \left( 1 - 1 / (2\sqrt{d}) \right) \left( 1 + \sqrt{2a/d} + a/d \right) \le \frac{\sigma_j^2}{{\sigma_j^*}^2} \left( 1 + \frac{x}{2} + \frac{2a}{dx} \right).
    \end{align*}
    More sufficient condition is
    \begin{align*}
        \frac{1}{d} \frac{\|v\|^2}{{\sigma_j^*}^2} \le 1 + \left(\sqrt{a/d} + a/2d \right), 
    \end{align*}
    where the initialization condition for $|\sigma_j^2 - {\sigma_j^*}^2|/{\sigma_j^*}^2 \le 1 / (4\sqrt{d}) \ll \sqrt{a/d}$ is used for the first inequality. From the Lemma \ref{lemma:chi_square_laurent}, this is true with probability at least $1 - \exp(-a/4) \ge 1 - \exp(-{R_{j1}^*}^2/(64 {\sigma_1^*}^2))$. In summary, when $\sigma_1 \ge \sigma_j$, we have $w_1 \le (\pi_1 / \pi_j) \exp(-{R_{j1}^*}^2 / (64 {\sigma_1^*}^2))$ with probability $1 - 3\exp(-{R_{j1}^*}^2 / (64 {\sigma_1^*}^2))$. Then, $\E_{\D_j} [w_1] \le 3(1 + \pi_1^*/\pi_j^*) \exp(-{R_{j1}^*}^2 / (64 {\sigma_1^*}^2))$.
    
    Now we consider the second case when $\sigma_1 \le \sigma_j$. We again use the same formulation as in \eqref{eq:condition_V}. This time, let $x = (\sigma_j^2 - \sigma_1^2) / \sigma_1^2$ and find the probability for
    \begin{align*}
        -\frac{\|v\|^2}{2\sigma_j^2} x + \frac{d}{2} \log(1 + x) \le a
        &\iff \frac{\|v\|^2}{\sigma_j^2} \ge \frac{1}{x} (d \log(1 + x) - 2a) \\
        &\iff \frac{\|v\|^2}{d {\sigma_j^*}^2} \ge \frac{\sigma_j^2}{{\sigma_j^*}^2} \frac{1}{x} (\log(1 + x) - 2a/d),
    \end{align*}
    where $x$ ranges from 0 to infinity. To control this, divide cases when $0 \le x \le 3/4$ and $x \ge 3/4$. If $0 \le x \le 3/4$, then 
    $$
        \log(1+x) \le x - x^2/2 + x^3/3 \le x - x^2/4.
    $$
    Using this, it is enough to give a probability bound for 
    \begin{align*}
        \frac{\|v\|^2}{d {\sigma_j^*}^2} &\ge (1 + 1/(2 \sqrt{d}) ) (1 - 2\sqrt{a/(2d)}) \ge \frac{\sigma_j^2}{{\sigma_j^*}^2} (1 - x/4 - 2a/(xd)).
    \end{align*}
    Therefore, the sufficient condition is
    \begin{align*}
        \frac{\|v\|^2}{d {\sigma_j^*}^2} \ge (1 - \sqrt{a/2d}).
    \end{align*}
    
    When $3/4 \le x$, note that $\log(1+x)/x \le 3/4$. Therefore, 
    \begin{align*}
        \frac{\|v\|^2}{d {\sigma_j^*}^2} \ge (1 - \sqrt{a/(xd)} ) \ge \frac{\sigma_j^2}{{\sigma_j^*}^2} (1 - 1/4 - 2a/(xd)).
    \end{align*}
    Note that $a/x = {R_{j1}^*}^2/ (8(\sigma_j^2 - \sigma_1^2)) \ge {R_{j1}^*}^2/(16{\sigma_j^*}^2)$. The sufficient condition for $\|v\|^2$ is thus,
    \begin{align*}
        \frac{\|v\|^2}{d {\sigma_j^*}^2} \ge \max \left( 1 - \sqrt{{R_{j1}^*}^2 /(16d {\sigma_1^*}^2)}, 1 - \sqrt{{R_{j1}^*}^2/ (16d{\sigma_j^*}^2)} \right),
    \end{align*}
    which will hold with probability at least $1 - \exp(-{R_{j1}^*}^2 / 64 {\sigma_j^*}^2)$. Note that these are all sufficient conditions to ensure $w_1 \le 3 (\pi_1^* / \pi_j^*) \exp(-{R_{j1}^*}^2 / (64{\sigma_j^*}^2) )$. 
    
    Combining all cases, when events defined in \eqref{event:good_j_neq_1} happen, then $w_1$ is small enough. As we have seen already, this is true with probability at least $1 - 5 \exp(-{R_{j1}^*}^2 / 64(\sigma_1^* \vee \sigma_j^*)^2) $.
\end{proof}

\subsection{Proof of Lemma \ref{corollary:error_bound_Dm_1over2}}
\begin{proof}
    By the Lemma \ref{lemma:weight_good}, the first equation is easy to show. Define $\Eps_{j,good} = \Eps_{j,1} \cap \Eps_{j,2} \cap \Eps_{j,3}$. For the ease of notation, let $\beta = {R_{j1}^*}^2 / 64(\sigma_1^* \vee \sigma_j^*)^2$.
    \begin{align*}
        \E_{\D_j} [w_1] &= \E_{\D_j} \left[ w_1 \indic_{\Eps_{j,good}} \right] + \E_{\D_j} \left[ w_1 \indic_{\Eps_{j,good}^c} \right] \\
        &\le 3 (\pi_1^* / \pi_j^*) \exp(-\beta) + P \left(\indic_{\Eps_{j,good}^c} \right) \le 3 (\pi_1^* / \pi_j^*) \exp(-\beta) + 5\exp(-\beta).
    \end{align*}
    
    For the second equation, 
    \begin{align*}
        |\E_{\D_j} [w_1\vdot{v}{s}] | &= \left| \E_{\D_j} \left[ w_1 \vdot{v}{s} \indic_{\Eps_{j,good}} \right] \right| + \left| \E_{\D_j} \left[ w_1 \vdot{v}{s} \indic_{\Eps_{j,good}^c} \right] \right| \\
        &\le 3(\pi_1^*/\pi_j^*) \exp(-\beta) \E_{\D_j} \left[ |\vdot{v}{s}| \right] + \E_{\D_j} \left[ |\vdot{v}{s}| | \Eps_{j,1}^c \right] P(\Eps_{j,1}^c) \\
        &\qquad + \E_{\D_j} \left[ |\vdot{v}{s}| | \Eps_{j,2}^c \right] P(\Eps_{j,2}^c) + \E_{\D_j} \left[ |\vdot{v}{s}| | \Eps_{j,3}^c \right] P(\Eps_{j,3}^c).
    \end{align*}
    $\E_{\D_j} \left[ |\vdot{v}{s}| | \Eps_{j,1}^c \right]$ can be bounded with Lemma \ref{lemma:vp_conditioned_vu}, with $p = 1$ and $\alpha = R_{j1}^*/5\sigma_j^*$.
    \begin{align*}
        \E_{\D_j} \left[ |\vdot{v}{s}| | \vdot{v}{R_{j1}^*} \ge {R_{j1}^*}^2/5 \right] &\le \sigma_j^* \E_{v \sim \mathcal{N}(0, I_d)} \left[ |\vdot{v}{s}| | |\vdot{v}{u}| \ge \alpha \right] \le \sigma_j^* \left( 2\alpha + \sqrt{2} \right) \le R_{j1}^*.
    \end{align*}
    Similarly, we can bound $\E_{\D_j} [|\vdot{v}{s}| | \Eps_{j,2}^c] P(\Eps_{j,2}^c) \le 2R_{j1}^*$ using the same Lemma \ref{lemma:vp_conditioned_vu} with $p = 1$ and $\alpha = R_{j1}^*/4 \sigma_j^*$. For the third term, we use Lemma \ref{lemma:vp_conditioned_vl2}, with $p = 1$ and $\alpha = {R_{j1}^*}^2/64(\sigma_1^* \vee \sigma_j^*)^2 = \beta$. Then,
    \begin{align*}
        \sigma_j^* \E_{v \sim \mathcal{N}(0,I_d)} [|\vdot{v}{s}| | \|v\|^2 \ge d + 2\sqrt{\alpha d} + 2\alpha] &\le \sigma_j^* \left( (64\alpha)^{1/2} + 4 \exp(-\alpha/2) (8\alpha + 2)^{1/2} \right) \le 2 R_{j1}^*, \\
        \sigma_j^* \E_{v \sim \mathcal{N}(0,I_d)} [|\vdot{v}{s}| | \|v\|^2 \le d - 2\sqrt{\alpha d}] &\le \sigma_j^* \E_{v \sim \mathcal{N}(0,I_d)} [|\vdot{v}{s}|] \le \sigma_j^*,
    \end{align*}
    Collecting these three components, we can conclude that 
    \begin{align*}
        |\E_{\D_j} [w_1 \vdot{v}{s}]| \le (3(\pi_1^*/\pi_j^*)\sigma_j^* + 5R_{j1}^*) \exp(-\beta).
    \end{align*}
    
    The same argument holds for $\Delta_w = w_1 - w_1^*$ since $0 \le w_1, w_1^* \le 3 (\pi_1^*/\pi_j^*) \exp(-\beta)$ ensures $|\Delta_w| \le 3 (\pi_1^*/\pi_j^*) \exp(-\beta)$. 
\end{proof}

\subsection{Proof of Lemma \ref{lemma:error_bounds_Dm_ge_1over2}}
We will focus on $q=2$ case. Due to the separation condition \eqref{eq:pop_separation_condition}, we have $\R{j1}^2 / \sigt{1}{j}^2 \ge C^2 \log (\rho_\sigma/\pi_{min}) \ge 4096 \log (\rho_\sigma/\pi_{min})$. Let $x := \R{j1}^2 / \sigt{1}{j}^2$. One useful fact is, if $x \ge 4096$, then $x \ge 128 \ln x$. Hence,
\begin{align}
    \label{eq:proof_lemma33_one}
    \sum_{j\neq1} \R{j1}^2 \exp(-\R{j1}^2 / 64 \sigt{1}{j}^2) &= \sum_{j\neq1} \sigt{j}{1}^2 x \exp(-x/64) \le \sum_{j\neq1} \sigt{j}{1}^2 \exp(-x/128) \nonumber \\
    &\le \sum_{j\neq1} \sigt{j}{1}^2 (\rho_{\sigma}/\pi_{min})^{-32} \ll c {\sigma_1^*}^2,
\end{align}
for some small constant $c$. The similar result holds for $q = 0$ and $q = 1$. Similarly, 
\begin{align}
    \label{eq:proof_lemma33_two}
    \sum_{j\neq1} \pi_j^* \R{j1}^2 \exp(-\R{j1}^2 / 64 \sigt{1}{j}^2) &= \sum_{j\neq1} \pi_j^* \sigt{j}{1}^2 x \exp(-x/64) \le \sum_{j\neq1} \pi_j^* \sigt{j}{1}^2 \exp(-x/128) \nonumber \\
    &\le \max_{j\neq1} \sigt{j}{1}^2 (\rho_{\sigma}/\pi_{min})^{-32} \ll c {\sigma_1^*}^2 \pi_1^*.
\end{align}
Summing up the result of $\pi_1^* \eqref{eq:proof_lemma33_one} + \eqref{eq:proof_lemma33_two} \le c' \pi_1^* {\sigma_1^*}^2$ gives the Lemma for $q = 2$. The cases for $q = 0, 1$ can be shown similarly.

\subsection{Convergence of Means and Mixing Weights}
\begin{proof}
    First we consider the error that comes from other components. 
    \paragraph{When $j \neq 1$:} We will primarily focus on bounding this quantity by analyzing the errors from each components separately. Then we will give a bound to estimators after one population EM iteration. Note that the Corollary \ref{corollary:error_bound_Dm_1over2} also holds for $\E_{\D_j}[\Delta_w]$ as in the corollary. With the Lemma \ref{lemma:error_bounds_Dm_ge_1over2}, we can bound the errors for mixing weights from other components. For mixing weights,
    \begin{align*}
        \sum_{j\neq1} \pi_j^* \E_{\D_j} [|w_1 - w_1^*|] \le 5 \sum_{j\neq1} \pi_j^* \exp(-\R{j1}^2/64\sigt{1}{j}^2) + 3\pi_1^* \sum_{j\neq1} \exp(-\R{j1}^2/64\sigt{1}{j}^2) \le c\pi_1^*,
    \end{align*}
    for some small constant $c$. 
    
    Similarly, we can bound the errors to the mean estimator from other components. First observe that
    \begin{align*}
        \|\mu_1^+ - \mu_1^*\| &= \|\E_\D [\Delta_w (X - \mu_1^*)]\| / \E_\D[w_1] \\
        &\le \sum_j \pi_j^* \sup_{s \in \mathbb{S}^{d-1}} |\E_{\D_j} [\Delta_w \vdot{X - \mu_1^*}{s}]| \le \sum_j \pi_j^* \sup_{s \in \mathbb{S}^{d-1}} |\E_\D [\Delta_w \vdot{v + \mu_j^* - \mu_1^*}{s}]|.
    \end{align*}
    The errors from other components are thus,
    \begin{align*}
        \sum_{j\neq1} \pi_j^* \left( \sup_{s \in \mathbb{S}^{d-1}} |\E_{\D_j} [\Delta_w \vdot{v}{s}] | + \R{j1} |\E_{\D_j}[\Delta_w]| \right) &\le 10 \sum_{j\neq 1} (\pi_1^* + \pi_j^*) \R{j1} \exp(-\R{j1}/64\sigt{1}{j}^2) \\
        &\le c' \sigma_1^* \pi_1^*,
    \end{align*}
    for some small constant $c'$.

    \paragraph{When $j = 1$:} From the correct component, we expect the weight is mostly close to 1, and 0 only rarely to bad samples. For this case, we can consider the weights given to other components. That is,
    \begin{align*}
        \E_{\D_1} [1 - w_1] &= \sum_{l \neq 1} \E_{\D_1} [w_l] \le \sum_{l\neq1} (3\pi_l^*/\pi_1^* + 5) \exp(-\R{l1}^2/64\sigt{l}{1}^2) \\
        &\le (1/\pi_1^*) \left(3 \sum_{l\neq1} \pi_l^* \exp(-\R{l1}^2/64\sigt{l}{1}^2) + 5 \pi_1^* \sum_{l\neq1} \exp(-\R{l1}^2/64\sigt{l}{1}^2) \right) \\
        &\le (1/\pi_1^*) (c \pi_1^*) \le c,
    \end{align*}
    for some small constant $c$. The same result holds for $\E_{\D_1}[|\Delta_w|] = \E_{\D_1} [|(1 - w_1) - (1 - w_1^*)|]$.
    
    Similarly, for the means, 
    \begin{align*}
        |\E_{\D_1} [(1 - w_1) \vdot{v}{s}]| &\le \sum_{l\neq 1} \E_{\D_1} [|w_l \vdot{v}{s}|] \le (1/\pi_1^*) \sum_{l\neq 1} (3\sigma_l^* \pi_l^* + 5 \R{j1} \pi_1^*) \exp(-\R{j1}^2/64\sigt{l}{1}^2) \\
        &\le (1/\pi_1^*) (c' \pi_1^* \sigma_1^*) \le c' \sigma_1^*.
    \end{align*}
    The same result also holds for $|\E_{\D_1} [\Delta_w \vdot{v}{s}]| = |\E_{\D_1} [((1-w_1) - (1-w_1^*)) \vdot{v}{s}]|$. 
    
    \paragraph{Errors from all components:} Now we can give a bound for the estimation errors after one population EM operation. For mixing weights,
    \begin{align*}
        |\pi_1^+ - \pi_1^*| \le \pi_1^* |\E_{\D_j} [\Delta_w]| + \sum_{j\neq 1} \pi_j^* |\E_{\D_j} [\Delta_w]| \le c_\pi \pi_1^*,
    \end{align*}
    and
    \begin{align*}
        \pi_1^+ \| \mu_1^+ - \mu_1^* \| &\le \pi_1^* \| \E_{\D_1} [\Delta_w (X - \mu_1^*)] \| + \sum_{j\neq1} \pi_j^* \||\E_{\D_j} [\Delta_w (X - \mu_1^*)]\| \le c \pi_1^* \sigma_1^*.
    \end{align*}
    Thus, $|\pi_1^+ - \pi_1^*| \le c_\pi \pi_1^*$, $\|\mu_1^+ - \mu_1^*\| \le c_\mu \sigma_1^*$ for some small constants $c_\pi, c_\mu \le 0.5$.
\end{proof}

\subsection{Convergence of Variances}
\label{appendix:population_em_variance}

\begin{proof} 
    We need some sharper bound on weights and probability of bad events, we need to go through another case study if we also have to estimate $\sigma_1^2$. We need to show that ${\sigma_1^+}^2$ will be very close to ${\sigma_1^*}^2$, {\it i.e.}, $|{\sigma_1^+}^2 - {\sigma_1^*}^2| / {\sigma_1^*}^2 \le 0.5 /\sqrt{d}$.
    First, let us arrange the EM operator for $\sigma_1^+$. 
    \begin{align*}
        {\sigma_1^+}^2 &= \E_\D[w_1 \|X - \mu_1^+\|^2] / (d \pi_1^+) \\
        &= \E_\D[w_1 (\|X - \mu_1^* + \mu_1^* - \mu_1^+\|^2] / (d \pi_1^+) \\
        &= (\E_\D[w_1 (\|X - \mu_1^*\|^2+ \|\mu_1^* - \mu_1^+\|^2] + 2 \vdot{\E_\D [w_1 (X-\mu_1^*)]}{\mu_1^* - \mu_1^+} ) / (d \pi_1^+) \\
        &= (\E_\D[w_1 \|X - \mu_1^*\|^2] - \pi_1^+ \|\mu_1^* - \mu_1^+\|^2) / (d \pi_1^+) \\
        &= \E_\D[w_1 \|X - \mu_1^*\|^2] / (d \pi_1^+) - \|\mu_1^* - \mu_1^+\|^2/d.
    \end{align*}
    We need to further change the expression to get a tight bound for the error. The difference from the ground truth is,
    \begin{align*}
        {\sigma_1^+}^2 - {\sigma_1^*}^2 &= \E_\D[w_1 \|X - \mu_1^*\|^2] / (d \pi_1^+) - {\sigma_1^*}^2 - \|\mu_1^* - \mu_1^+\|^2/d \\
        &= \E_\D[w_1 (\|X - \mu_1^*\|^2 - d{\sigma_1^*}^2)] / (d \pi_1^+) - \|\mu_1^* - \mu_1^+\|^2/d \\
        &= \E_\D[\Delta_w (\|X - \mu_1^*\|^2 - d{\sigma_1^*}^2) ] / (d \pi_1^+) - \|\mu_1^* - \mu_1^+\|^2/d \\
        &= \frac{\sum_j \pi_j^* \E_{\D_j} [\Delta_w (\|\mu_j^* - \mu_1^*\|^2 + 2\vdot{v}{\mu_j^* - \mu_1^*} + \|v\|^2 - d{\sigma_1^*}^2) ]}{d \pi_1^+} - \frac{\|\mu_1^* - \mu_1^+\|^2}{d}, \\
        |{\sigma_1^+}^2 - {\sigma_1^*}^2| &\le \frac{\sum_j \pi_j^* |\E_{\D_j} [\Delta_w]| {R_{j1}^*}^2}{d\pi_1^+} + \frac{\sum_j 2 \pi_j^* |\E_{\D_j} [\Delta_w \vdot{v}{\mu_j^*-\mu_1^*}]|}{d\pi_1^+} \\
        &+ \frac{\sum_j \pi_j^* |\E_{\D_j} [\Delta_w (\|v\|^2 - d{\sigma_j^*}^2)]|}{d\pi_1^+} + \frac{\sum_j \pi_j^* |\E_{\D_j} [\Delta_w] d({\sigma_j^*}^2 - {\sigma_1^*}^2 )|}{d\pi_1^+} + \frac{\|\mu_1^* - \mu_1^+\|^2}{d}.
    \end{align*}
    
    Let us consider the terms one by one. We can use Corollary \ref{corollary:error_bound_Dm_1over2} for the first and second terms. The third term can be bounded with Cauchy-Schwartz inequality.
    \begin{align*}
        |\E_{\D_j} [\Delta_w (\|v\|^2 - d{\sigma_j^*}^2)]| &\le \sqrt{\E_{\D_j} [\Delta_w]} \sqrt{\E_{\D_j} [(\|v\|^2 - d{\sigma_j^*}^2)^2]} \\
        &\le \sqrt{(3\pi_1^*/\pi_j^* + 5) \exp(-\R{j1}^2/64\sigt{j}{1}^2)} \sqrt{2d {\sigma_j^*}^4},
    \end{align*}
    since the variance of chi-Square distribution with $d$ degrees of freedom is $2d$.
    
    The fourth term requires redefinition of good events to get an error that only scales with $\sqrt{d}$ when $d$ grows large. In order to bound this term, we need to give a sharper bound on $\E_{\D_j} [\Delta_w]$. That is, we need the following lemma:
    \begin{lemma}
        \label{lemma:weight_bound_for_variance}
        For $j \neq 1$, if $|{\sigma_j^*}^2 - {\sigma_1^*}^2| \ge 10 R_{j1}^* (\sigma_j^* \vee \sigma_1^*) / \sqrt{d}$, then
        \begin{align*}
            |\E_{\D_j} [\Delta_w]|, |\E_{\D_j} [w_1]| \le O \left( \exp(-d \min(1, t^2) / 256) \exp(- {R_{j1}^*}^2 / 64 ({\sigma_1^*} \vee \sigma_j^*)^2) \right), 
        \end{align*}
        where $t = |{\sigma_j^*}^2 - {\sigma_1^*}^2| / {\sigma_1^*}^2$. This implies,
        \begin{align*}
            d|{\sigma_j^*}^2 - {\sigma_1^*}^2| |\E_{\D_j} [\Delta_w]| \le O \left( \sqrt{d} (1+\pi_1^*/\pi_j^*) R_{j1}^* (\sigma_j^* \vee \sigma_1^*) \exp(- {R_{j1}^*}^2 / 64(\sigma_1^* \vee \sigma_j^*)^2) \right). 
        \end{align*}
    \end{lemma}
    
    The intuition is as follows: what if $({\sigma_j^*}^2 - {\sigma_1^*}^2) < 10 (R_{j1}^* \sigma_1^*)/ \sqrt{d}$? Then we have nothing to worry about, since $\sqrt{d}$ is already canceled out and $\E_{\D_j} [\Delta_w]$ will be small enough to cancel the rest which is $R_{j1}^* \sigma_1^*$. The problem is when $|{\sigma_j^*}^2 - {\sigma_1^*}^2| > 10 (R_{j1}^* \sigma_1^*) / \sqrt{d}$. Here, we expect that the error bound $\E_{\D_j} [|\Delta_w|]$ would crucially depend on the quantity $t = |{\sigma_j^*}^2 - {\sigma_1^*}^2| / {\sigma_1^*}^2 > 10 R_{j1}^*/(\sigma_1^*\sqrt{d})$. Recall that there are three main terms in the exponent of weights as in \eqref{eq:exponent_dissect}:
    \begin{align*}
        III &= -\vdot{v}{\mu_j^* - \mu_1} / \sigma_1^2, \\
        IV &= \vdot{v}{\Delta_j} / \sigma_j^2, \\
        V &= - \frac{\|v\|^2}{2} \left( \frac{1}{\sigma_1^2} - \frac{1}{\sigma_j^2} \right) - \frac{d}{2} \log(\sigma_1^2 / \sigma_j^2). 
    \end{align*}
    
    \begin{proof}
    \paragraph{When $\sigma_1^* \ge \sigma_j^*$ and $|{\sigma_j^*}^2 - {\sigma_1^*}^2| > 10 R_{j1}^* \sigma_1^* / \sqrt{d}$:} Let us refine good events so that each events can take this case into account. Note that this can only happen for very large $d > 64^2 \cdot 100$ given our separation condition. We first show that if ${\sigma_1^*}^2 - {\sigma_j^*}^2 \ge 10 R_{j1}^* \sigma_1^* /\sqrt{d}$, then $x > 4t/5$. To see this, 
    \begin{align*}
        {\sigma_1}^2 - {\sigma_j}^2 &\ge (1 - 1/(2\sqrt{d})) {\sigma_1^*}^2 - (1+1/(2\sqrt{d})) {\sigma_j^*}^2 \\
        &\ge ({\sigma_1^*}^2 - {\sigma_j^*}^2) - 1/(2\sqrt{d}) ({\sigma_1^*}^2 + {\sigma_j^*}^2) \ge 9 R_{j1}^* \sigma_1^* / \sqrt{d}
    \end{align*}
    which then is connected to
    \begin{align*}
        x \ge ({\sigma_1^*}^2 - {\sigma_j^*}^2) / \sigma_1^2 - 1/(2 \sqrt{d}) ({\sigma_1^*}^2 + {\sigma_j^*}^2) / \sigma_1^2 &\ge 4t/5,
    \end{align*}
    given good initialization of $\sigma_1$ and $\sigma_j$.
    
    Instead of just requiring $V \le {R_{j1}^*}^2/(8 \sigma_1^2)$ as in mixing weight case, let us require $V \le -dx^2 / 16 + a$ where $x = (\sigma_1^2 - \sigma_j^2) / \sigma_1^2$ and $a = {R_{j1}^*}^2 / (8\sigma_1^2)$. Then, 
    \begin{align*}
        \frac{1}{d} \frac{\|v\|^2}{{\sigma_j^*}^2} &\le \frac{\sigma_j^2}{{\sigma_j^*}^2} \left( 1 + \frac{x}{2} - \frac{x}{8} + \frac{2a}{dx} \right) = \frac{\sigma_j^2}{{\sigma_j^*}^2} \left(1 + \frac{3x}{8} + \frac{2a}{dx} \right), 
    \end{align*}
    is a sufficient condition to guarantee $V \le -dx^2/16 + a$. The probability of this event is
    \begin{align*}
        P \left(\frac{\|v\|^2}{d {\sigma_j^*}^2} \ge \frac{\sigma_j^2}{{\sigma_j^*}^2} \left( 1 + \frac{3x}{8} + \frac{2a}{dx} \right) \right)
        &\le P \left(\frac{\|v\|^2}{d {\sigma_j^*}^2} \ge \left(1 - \frac{1}{2\sqrt{d}} \right) \left( 1 + \frac{3x}{8} + \frac{2a}{dx} \right) \right) \\
        &\le \exp \Big( -d x / 32 - {R_{j1}^*} / (64 {\sigma_1^*}^2) \Big).
    \end{align*}
    
    For quantities $III$ and $IV$, we will require
    \begin{align*}
        P &\left(III \ge dx^2 / 32 + 7 {R_{j1}^*}^2 / (32 \sigma_1^2) \right) \le P \left(\vdot{v}{\mu_j^* - \mu_1} \ge \sqrt{d}x \sqrt{d} (\sigma_1^2 - \sigma_j^2) / 32 + 7 {R_{j1}^*}^2 / 32 \right) \\
        &\le P \left(\vdot{v}{\mu_j^* - \mu_1^*} \ge \sqrt{d}x \sqrt{d} (\sigma_1^2 - \sigma_j^2) / 40 + {R_{j1}^*}^2 / 5 \right) + P\left(\vdot{v}{\Delta_1} \ge \sqrt{d}x \sqrt{d} (\sigma_1^2 - \sigma_j^2) / 160 + {R_{j1}^*}^2 / 32 \right) \\
        &\le \exp \Big(-dx^2/2 \underbrace{(\sqrt{d} (\sigma_1^2 - \sigma_j^2) / 40 \sigma_j^* R_{j1}^*)^2}_{(i)} - {R_{j1}^*}^2 /(64{\sigma_j^*}^2) \Big) \\
        &\quad + \exp \Big(-dx^2/2 \underbrace{(\sqrt{d} (\sigma_1^2 - \sigma_j^2) / 10 \sigma_j^* R_{j1}^*)^2}_{(i')} - {R_{j1}^*}^2 /(16 {\sigma_j^*}^2) \Big).
    \end{align*}
    For $IV$,
    \begin{align*}
        P \Big(\vdot{v}{\Delta_j} &\ge dx^2 \sigma_j^2 / 128 + {R_{j1}^*}^2 \sigma_j^2 / (20 \sigma_1^2) \Big) \\ 
        &\le P \Big(Z \ge dx^2 (\sigma_1^* \vee \sigma_j^*) ({\sigma_j}^2 / {\sigma_j^*}^2) / (8R_{j1}^*) + 4 {R_{j1}^*} ({\sigma_j}^2 / {\sigma_j^*}^2) / ( 5 \sigma_1^2) \Big) \\
        &\le \exp \Big(-dx^2/2 \underbrace{(\sqrt{d} (\sigma_1^2 - \sigma_j^2) / 16 \sigma_1^* R_{j1}^*)^2}_{(i'')} - {R_{j1}^*}^2 /(64{\sigma_1^*}^2) \Big),
    \end{align*}
    When three events happens at the same time, we are guaranteed that $w_1 \le \exp(-dx^2/64 - {R_{j1}^*}^2/(64{\sigma_1^*}^2))$. Meanwhile, by the assumption on this case study we have $(i), (i'), (i'') \ge 1/32$. That is, we get $w_1 \le \exp(-dt^2 / 128 - {R_{j1}^*}^2 / (64{\sigma_1^*}^2))$ with probability $1 - 4\exp(-dt^2 / 64 - {R_{j1}^*}^2 / (64 {\sigma_1^*}^2))$. Under the same events, we can show the same for $w_1^*$. Therefore,
    \begin{align*}
        \E_{\D_j} [\Delta_w] \le 10 (\pi_1^*/\pi_j^* + 1) \exp(-dt^2/64 - {R_{j1}^*}^2 / (64 {\sigma_1^*}^2)). 
    \end{align*}
    This is enough to bound the fourth term (note that if $\sqrt{d} t > 8R_{j1}^* /\sigma_1^* \gg 512$, then $dt^2 \ge 64\log(\sqrt{d}t)$ will be guaranteed). 
    The consequence of this relation is that
    \begin{align*}
        |\E_{\D_j} [\Delta_w] d({\sigma_j^*}^2 - {\sigma_1^*}^2)| \le {\sigma_1^*}^2 \sqrt{d} \exp(-{R_{j1}^*}^2 / (64 {\sigma_1^*}^2)). 
    \end{align*}
    
    \paragraph{When $\sigma_1^* \le \sigma_j^*$ and $|{\sigma_j^*}^2 - {\sigma_1^*}^2| > 10 R_{j1}^* \sigma_j^* / \sqrt{d}$:} We can go through the other side (when $\sigma_1 \le \sigma_j$) similarly. For this case, let $x = ({\sigma_j}^2 - {\sigma_1}^2) / {\sigma_1}^2$. In this case, we consider the case when $t = ({\sigma_j^*}^2 - {\sigma_1^*}^2) / {\sigma_1^*}^2 > 10 R_{j1}^* \sigma_j^* / ({\sigma_1^*}^2 \sqrt{d})$, and we show that $x \ge 4t/5$ as before. Note that still the case can only happen when $d > 640^2$.
    \begin{align*}
        \sigma_j^2 - \sigma_1^2 &\ge (1 - 1/2 \sqrt{d}) {\sigma_j^*}^2 - (1 + 1/2\sqrt{d}) {\sigma_1^*}^2 \\
        &= ({\sigma_j^*}^2 - {\sigma_1^*}^2) - 1/(2\sqrt{d}) ({\sigma_j^*}^2 + {\sigma_1^*}^2) \\
        &\ge 9 R_{j1}^*\sigma_j^* / \sqrt{d}, \\
        x &\ge 4t/5 \ge 8 R_{j1}^* \sigma_j^* / ({\sigma_1^*}^2 \sqrt{d}).
    \end{align*}
    
    We define the condition for $V$ as
    \begin{align*}
        P(V \ge -dx^2 / 32 + {R_{j1}^*}^2 / (8\sigma_1^2) ) &= P(-\|v\|^2 x/ (2{\sigma_j}^2) + \frac{d}{2} \log(1+x) \ge -dx^2/32 + a ) \\
        &\le P(-\|v\|^2 x/{\sigma_j}^2 + d (x - x^2/4) \ge -dx^2/16 + 2a ) \\
        &= P(\|v\|^2 /{\sigma_j}^2 \le d (1 - 3x/16 -  2a / (d x) )) \\
        &\le \exp(-dx / 128 - {R_{j1}^*}^2 / (64 {\sigma_1^*}^2) ), 
        \qquad \text{for} \quad 0 \le x \le 3/4,
    \end{align*}
    We also similarly compute the bad probabilities for other quantities $III$ and $IV$.
    \begin{align*}
        P(&III \ge dx^2/64 + 7 {R_{j1}^*}^2 / (32 \sigma_1^2) ) \\
        &\le P(\vdot{v}{\mu_j^* - \mu_1^*} \ge \sqrt{d}x \sqrt{d}(\sigma_j^2 - \sigma_1^2) / 80 + {R_{j1}^*}^2 / 5) + P( \vdot{v}{\Delta_1} \ge \sqrt{d}x \sqrt{d}(\sigma_j^2 - \sigma_1^2) / 320 + {R_{j1}^*}^2 / 64) \\
        &\le \exp( -dx^2/2 \underbrace{(\sqrt{d}(\sigma_j^2 - \sigma_1^2)/(80 R_{j1}^* \sigma_j^*))^2}_{(i)} - {R_{j1}^*}^2 / (32 {\sigma_j^*}^2)) \\
        &+ \exp( -dx^2/2 \underbrace{(\sqrt{d}(\sigma_j^2 - \sigma_1^2)/(20 R_{j1}^* \sigma_j^*))^2}_{(i')} - {R_{j1}^*}^2 / (64 {\sigma_j^*}^2)),
    \end{align*}
    \begin{align*}
        P(IV \ge dx^2/128 + {R_{j1}^*}^2 / (20 \sigma_1^2)) &\le P(\vdot{x}{\Delta_j} \sigma_1^2 / \sigma_j^2 \ge dx^2 \sigma_1^2 / 16 + 4 {R_{j1}^*}^2 / 5 ) \\
        &\le P(Z \ge \sqrt{d}x (\sqrt{d} (\sigma_j^2 - \sigma_1^2)/(16 R_{j1}^* \sigma_1^*)) + {R_{j1}^*}^2 / (4 {\sigma_1^*}^2) ) \\
        &\le \exp(-dx^2/2 \underbrace{(\sqrt{d} (\sigma_j^2 - \sigma_1^2) / (16 R_{j1}^* \sigma_1^*) )^2}_{(i'')} - {R_{j1}^*}^2 / (64 {\sigma_1^*}^2) )
    \end{align*}
    Since we assumed $|\sigma_j^2 - \sigma_1^2| \ge 9\R{j1} \sigma_j^* / \sqrt{d}$ , we can see that $(i), (i'), (ii') \ge 1/100$. Thus, similarly we can get $w_1 \le 3 (\pi_1^* / \pi_j^*) \exp(-dt^2 / 256 - {R_{j1}^*}^2 / (64{\sigma_1^*}^2))$ with probability $1 - \exp(-dt^2 / 256 - {R_{j1}^*}^2 / (64 {\sigma_1^*}^2))$. Using the same argument, $\E_{\D_j} [\Delta_w]$ can be bounded by $\exp(-dt^2/256 - {R_{j1}^*}^2/(64 {\sigma_1^*}^2))$.
    
    When $3/4 \le x$, we target:
    \begin{align*}
        P(V \ge -dx / 16 + {R_{j1}^*}^2 / (8{\sigma_1}^2) ) &= P(-\|v\|^2 x/ (2{\sigma_j}^2) + (d/2) \log(1+x) \ge -dx/8 + {R_{j1}^*}^2/(8{\sigma_1}^2)) \\
        &\le P(\|v\|^2 /{\sigma_j}^2 \le d (\log(1+x)/x + 1/8) - {R_{j1}^*}^2/(4{\sigma_1}^2)/x ) \\
        &\le P(\|v\|^2 /{\sigma_j}^2 \le d (1 - 1/8 - {R_{j1}^*}^2/(4 d ({\sigma_j}^2 - \sigma_1^2) ) )) \\
        &\le \exp(-d/256 - {R_{j1}^*}^2/(64 {\sigma_j^*}^2)),
    \end{align*}
    Note that when $x = (\sigma_j^2 - \sigma_1^2) / \sigma_1^2 > 6(R_{j1}^* \sigma_j^*) / ({\sigma_1^*}^2 \sqrt{d})$, it is true that $d \gg 512 \log \sqrt{d}$. For $III$ and $IV$, when $x \ge 3/4$, we find a probability for
    \begin{align*}
        P(III \ge dx / 32 + 7 {R_{j1}^*}^2 / (32{\sigma_1}^2)) &\le P(\vdot{v}{\mu_j^* - \mu_1^*} \ge \sqrt{d}\sqrt{d}(\sigma_j^2 - \sigma_1^2) / 64 + {R_{j1}^*}^2 / 5) \\
        &+ P( \vdot{v}{\Delta_1} \ge \sqrt{d} \sqrt{d}(\sigma_j^2 - \sigma_1^2) / 64 + {R_{j1}^*}^2 / 64) \\
        &\le \exp( -d/2 \underbrace{(\sqrt{d}(\sigma_j^2 - \sigma_1^2)/(64 R_{j1}^* \sigma_j^*))^2}_{(i)} - {R_{j1}^*}^2 / (64 {\sigma_j^*}^2)) \\
        &+ \exp( -d/2 \underbrace{(\sqrt{d}(\sigma_j^2 - \sigma_1^2)/(4 R_{j1}^* \sigma_j^*))^2}_{(i)} - {R_{j1}^*}^2 / (64 {\sigma_j^*}^2)),
    \end{align*}
    and,
    \begin{align*}
        P(IV \ge dx/64 + {R_{j1}^*}^2 / (20 \sigma_1^2)) &\le P(ZR_{j1}^* \sigma_j^* \ge dx {\sigma_1}^2/4 + 4{R_{j1}^*}^2 /5) \\
        &\le \exp(-d/2 \underbrace{(\sqrt{d}(\sigma_j^2 - \sigma_1^2) / (4R_{j1}^* \sigma_j^*))^2}_{(i')} - {R_{j1}^*}^2 / (64{\sigma_1^*}^2)),
    \end{align*}
    Again, the similar result holds for $\E_{\D_j} [\Delta_w]$. Therefore,
    \begin{align*}
        \E_{\D_j} [\Delta_w] \le 3 (1 + \pi_1^*/\pi_j^*) \exp(-d/256 - {R_{j1}^*}^2 / (64 {\sigma_j^*}^2) ), 
    \end{align*}
    Collecting all cases yields the Lemma. 
    \end{proof}
    
    \paragraph{Errors from own component $j = 1$:} Note that if $j = 1$, first, second, and fourth terms are gone automatically. For the third term, $\E_{\D_1} [\Delta_w] \le c$ for some small $c$ as we have seen several times, and $\E_{\D_1} [(\|v\|^2 - d{\sigma_1^*}^2)^2] \le 2 d {\sigma_1^*}^4$. The fifth term is less than $O( {\sigma_1^*}^2 /d)$ as we have already seen that next estimates for means are already within $c_\mu \sigma_1^*$. Hence the error from own component is less than $c_1$ for some small constant $c_1$.
    
    \paragraph{Errors from all components:} Now we can collect every terms to give a bound the error of $\sigma_1^+$, 
    \begin{align*}
        {\sigma_1^+}^2 - {\sigma_1^*}^2 &\le c_1{\sigma_1^*}^2 + \frac{ \sum_{j \neq 1} \pi_j^* \R{j1}^2 \E_{\D_j} [\Delta_w] } {d\pi_1^+} + \frac{2 \sum_{j \neq 1} \pi_j^* \R{j1} \E_{\D_j} [\Delta_w] }{d\pi_1^+} \\
        &+ \frac{\sum_{j\neq 1} 5 {\sigma_j^*}^2 (\pi_1^* + \pi_j^*) \sqrt{d} \exp(-\R{j1}^2/128\sigt{1}{j}^2) }{d\pi_1^+} \\
        &+ \frac{\sum_{j \neq 1} 50 (\pi_1^* + \pi_j^*) \sqrt{d} \R{j1}\sigt{j}{1} \exp(-\R{j1}^2/64\sigt{1}{j}^2)}{d\pi_1^+} + \frac{c_\mu^2 {\sigma_1^*}^2}{d},
    \end{align*}
    which gives $|{\sigma_1^+}^2 - {\sigma_1^*}^2| \le c_\sigma {\sigma_1^*}^2 / \sqrt{d},$ for some small constant $c_\sigma < 0.5$ given good enough SNR condition in \eqref{eq:pop_separation_condition}.
\end{proof}

\section{Proof for Finite-Sample EM}
\label{appendix:finite_sample}
We define some additional notations. We use $\Eps_j$ to denote the event that the $i^{th}$ sample comes from $j^{th}$ component. Define $\Eps_{j,good} := \Eps_{j,1} \cap \Eps_{j,2} \cap \Eps_{j,3}$ where $\Eps_{j, \cdot}$ are as defined in \eqref{event:good_j_neq_1}. For the simplicity in notation, we now use looser upper bound $\rho_\pi$ for $(1 \vee \pi_1^* / \pi_j^*)$. We use $\lesssim$ when the inequality holds up to some universal constants. Finally, under the modified condition in \eqref{eq:finite_separation_cond}, we will use slightly modified version of Lemma \ref{lemma:error_bounds_Dm_ge_1over2}:
\begin{lemma}
    \label{lemma:error_bounds_finite_sample}
    For well-separated mixture of Gaussians, for $q \in \{0, 1, 2\}$,
    \begin{align}
        \label{ineq:finite_sample_sum_weight_error}
        \rho_\pi \sum_{j\neq 1} \R{j1}^q \exp \left(-\R{j1}^2 / (128c) \sigt{1}{j}^2 \right) &\le c_q {\sigma_1^*}^q \pi_{min},
    \end{align}
    for some small constants $c_q$ given separation condition as in \eqref{eq:finite_separation_cond}. 
\end{lemma}
\begin{proof}
    The proof is similar to that of Lemma \ref{lemma:error_bounds_Dm_ge_1over2}. Note that $\R{j1}^2 / \sigt{1}{j}^2 \ge C^2 c^2 \log (\rho_\sigma / \pi_{min})$ where the universal constant $C$ is such that $C^2 \ge 4096$. Let $x := \R{j1}^2 / \sigt{1}{j}^2$. Then, since $\log(x) / x$ is decreasing in $x$ whenever $x \ge e$, 
    \begin{align*}
        \frac{\log(x)}{x} &\le \frac{\log(C^2 c^2 \log(\rho_\sigma / \pi_{min}))}{C^2 c^2 \log(\rho_\sigma / \pi_{min})} \le \frac{1}{256c}. 
    \end{align*}
    Applying this to the equation  \eqref{ineq:finite_sample_sum_weight_error} with $q=2$, 
    \begin{align*}
        \rho_\pi \sum_{j\neq 1} \R{j1}^2 \exp \left(-\R{j1}^2 / (128c) \sigt{1}{j}^2 \right) &\le \rho_\pi \sum_{j\neq 1} \sigt{1}{j}^2 \exp \left(-\R{j1}^2 / (256c) \sigt{1}{j}^2 \right) \\
        &\le \rho_\pi \sum_{j\neq 1} \sigt{1}{j}^2 (\rho_\sigma / \pi_{min})^{-32} \ll c_2 {\sigma_1^*}^2 \pi_{min},
    \end{align*}
    which gives the lemma with some small constant $c_2$. Similar claims hold for $q = 0, 1$. 
\end{proof}
We also restate here the Proposition \ref{lemma:indic_prob_decompose}.
\begin{proposition}[Restatement of Proposition \ref{lemma:indic_prob_decompose}]
    Let $X$ be a random $d$-dimensional vector, and $A$ be an event in the same probability space with $p = P(A) > 0$. Define random variable $Y = X|A$, {\em i.e.}, $X$ conditioned on event $A$, and $Z = \indic_{X \in A}$. Let $X_i, Y_i, Z_i$ be the i.i.d samples from corresponding distributions. Then, the following holds, 
    \begin{align}
        P\Bigg( \Big\|\frac{1}{n} \sum_{i=1}^{n} X_i \indic_{X_i \in A} - &\E[X \indic_{X \in A}] \Big\| \ge t \Bigg) \le \max_{m \le n_e} P \left(\frac{1}{n} \left\|\sum_{i=1}^{m} (Y_i - \E[Y]) \right\| \ge t_1 \right) \nonumber \\
        & + P\left(\|\E[Y]\| \left|\frac{1}{n} \sum_{i=1}^n Z_i - p\right| \ge t_2 \right)
        + P\left(\left|\sum_{i=1}^n Z_i\right| \ge n_e+1\right).
    \end{align}
    for any $0 \le n_e \le n$ and $t_1 + t_2 = t$.
\end{proposition}

\subsection{Concentration in Mixing Weights}
\begin{proof}
We give a concentration result for mixing weights first. We can first check that
\begin{align*}
    \tpi_1^+ - \pi_1^+ &= \frac{1}{n} \sum_{i=1}^n w_{1,i} - \E_\D [w_1].
\end{align*}
Then, single $w_{1,i}$ can be decomposed using indicators. Then,
\begin{align*}
    w_{1,i} = w_{1,i} \indic_{\Eps_{1}} + \sum_{j\neq 1} w_{1,i} \indic_{\Eps_j \cap \Eps_{j,good}} + w_{1,i} \indic_{\Eps_j \cap \Eps_{j,good}^c}.
\end{align*}
Now let us apply proposition \ref{lemma:indic_prob_decompose} step by step.
\paragraph{With $\Eps_1$:} Note that $P(\Eps_1) = \pi_1^*$. We can pick $n_e = 2 n \pi_1^*$. By multiplicative version of concentration inequality for Bernoulli random variable, the second and third terms will be safely killed. Also, we note that $w_{1,i}$ is bounded random variable. Therefore,
\begin{align*}
    P( 1/n |\sum_{i=1}^{n_e} (w_{1,i} - \E_{\D_1} [w_{1,i}]) | \ge t ) \le \exp(-2 n_e (nt / n_e)^2) = \exp(-2 n^2 / n_e t^2).
\end{align*}
Thus, $t = O\left(\sqrt{n_e/n} \sqrt{\ln(k^2T/\delta)/n} \right) = O(\sqrt{\pi_1^* \ln(k^2T/\delta)/n})$ gives $\delta/(k^2T)$ error bound. 

\paragraph{With $\Eps_j \cap \Eps_{j, good}$:} When a good sample comes from $j^{th} \neq 1$ component, the weight given to first component is very small, {\it i.e.}, $w_{1,i} \indic_{\Eps_j \cap \Eps_{j, good}} \le 5\rho_\pi \exp(-{R_{j1}^*}^2 / (64 (\sigma_j^* \vee \sigma_1^*)^2))$. Thus, it is a bounded random variable, therefore its statistical error can be bounded by
\begin{align*}
    t = O \left(\rho_\pi \sqrt{\pi_j^*} \sqrt{\ln(k^2T / \delta)/n} \right) \exp \left(- {R_{j1}^*}^2 / 64 (\sigma_j^* \vee \sigma_1^*)^2 \right),
\end{align*}
with probability at least $1 - (k^2 T / \delta)$. 

\paragraph{With $\Eps_{j, good}^c$:} This is a very special case, since the chance of this event to happen is $p := 5 \pi_j^* \exp(-{R_{j1}^*}^2 / 64(\sigma_1^* \vee \sigma_j^*)^2 )$, {\it i.e.}, exponentially small. We first need to bound the number of samples that have fallen into this bad event with high probability. We divide the case as when $n \ge p^{1/c}$ and $n \le p^{1/c}$ for some constant $c > 2$. 

Let us first consider when $n \ge p^{1/c}$. Recall the Bernstein's inequality, which states for Bernoulli random variable that
\begin{align*}
    P \left( \Big|\frac{1}{n} \sum_{i=1}^n Z_i - p \Big| \ge t \right) \le \exp(- nt^2 / (2p + 2/3 t)).
\end{align*}
Solving the right hand side to get a $\delta/(k^2T)$ probability bound, we get
\begin{align*}
    n_e = O(c \ln (k^2T/\delta) + c \sqrt{np \ln (k^2T/\delta)}). 
\end{align*}
Then using the proposition, we can decide how large the sum of $1/n \sum_{i=1}^{n} w_{1,i} \indic_{\Eps_{j, good}^c}$, which will be
\begin{align*}
    t = O \left( \sqrt{p \vee 1/n} \sqrt{\ln^2(k^2T/\delta) /n} \right) = O \left( p^{1/2c} \sqrt{\ln^2(k^2T/\delta) /n} \right) = \tilde{O}\left(\pi_{min}^{32} \sqrt{1/n} \right),
\end{align*}
with probability $1 - n^c$. 

On the other side, if $n \le p^{1/c}$, then we will have $n_e = 0$ with probability at least $1 - np \ge 1 - p^{1 - 1/c}$. Note that $p^{1 - 1/c} \le O(\pi_{min}^{32} / n^{c - 2})$ given SNR condition as in the theorem. Thus, in this case, with probability at least $1 - 1/(n^{c-2} k^{32})$, we have 
\begin{align*}
    \left|\frac{1}{n} \sum_{i=1}^{n} w_{1,i} \indic_{\Eps_{j,good}^c} - \E_{\D} [w_{1,i} \indic_{\Eps_{j,good}^c}] \right| \le \E_{\D_j} [w_{1,i} \indic_{\Eps_{j,good}^c}] \le \exp(-{R_{j1}^*}^2 / (64 (2c) (\sigma_1^* \vee \sigma_j^*)^2)).
\end{align*}

\paragraph{Collect all errors:} Now we can collect all items we found for each cases. Taking union bound over all $O(k)$ items, with probability $1 - O(\delta/kT) - O(1/(n^{c-2} k^{31}))$, 
\begin{align*}
    \left|\frac{1}{n} \sum_{i=1}^{n} w_{1,i} - \E_{\D} [w_{1,i}] \right| &\le O \left(\sqrt{\pi_1^* \ln(k^2T / \delta) / n} \right) +  O\left( \rho_\pi \sqrt{\ln^2 (k^2T / \delta) / n} \right) \sum_{j\neq1} \exp(-\R{j1}^2/(128c\sigt{j}{1}^2)) \\
    &\le O(\pi_1^* \epsilon),
\end{align*}
where we used Lemma \ref{lemma:error_bounds_finite_sample}. Thus, $|\tpi_1^+ - \pi_1^*| / \pi_1^* \le |\tpi_1^+ - \pi_1^+| / \pi_1^* + |\pi_1^+ - \pi_1^*| / \pi_1^* \le \epsilon + \gamma D_m$. Thus, after $T = O(\log(1/ \epsilon) )$ iteration, we get $|\tpi_1^{(T)} - \pi_1^*| \le \pi_1^* \epsilon$ with probability $1 - O(\delta/k) - O(T/(n^{c-2}k^{31}))$. We can take union bound over all $O(k)$ components to get the result for all components with probability $1 - O(\delta) - O(\log (1/\epsilon) / n^{c-2} k^{30})$.
\end{proof}

\subsection{Concentration in Means}
\begin{proof}
Now let us look at the iteration for means. First, we should observe that
\begin{align*}
    \tmu_1^+ - \mu_1^* &= \left(\frac{1}{n} \sum_{i=1}^n w_{1,i}(X - \mu_1^*) \right) / \left(\frac{1}{n} \sum_{i=1}^n w_{1,i} \right) \\
    &= \left(\frac{1}{n} \sum_{i=1}^n w_{1,i} (X_i - \mu_1^*) - \E_{\D} [w_1 (X - \mu_1^*)] + \E_\D[w_1 (X - \mu_1^*)] - \E_\D[w_1^* (X - \mu_1^*)] \right) / \tpi_1^+ \\
    &= \left( \underbrace{\frac{1}{n} \sum_{i=1}^n w_{1,i} (X_i - \mu_1^*) - \E_{\D} [w_1 (X - \mu_1^*)]}_{e_\mu} + \underbrace{\E_\D[\Delta_w (X - \mu_1^*)]}_{B_\mu} \right) / \tpi_1^+. 
\end{align*}
$B_\mu$ is decreasing as we have seen for population EM, we focus on the fluctuation of the sum of random variables $W = w_1 (X - \mu_1^*)$. We further decompose this random variable using disjoint events. That is,
\begin{align*}
    W_i = W_i \indic_{\Eps_1} + \sum_{j\neq 1}^n \left(W_i \indic_{\Eps_{j, good}} + W_i \indic_{\Eps_{j,1}^c} + W_i \indic_{\Eps_{j,1} \cap \Eps_{j,2}^c} + W_i \indic_{\Eps_{j,1} \cap \Eps_{j,2} \cap \Eps_{j,3}^c}\right).
\end{align*}
Now for each decomposed sample, we compute $\psi_2$ or $\psi_1$ norm conditioned on each event, and sum everything at the end. 

\paragraph{With $j \neq 1, \Eps_{j,good}$:} Let $Y_i = W_i \indic_{\Eps_{j, good}} | \Eps_{j}$ and $Z_i = \indic_{j}$. Then using the proposition, 
\begin{align*}
    P \Bigg( \Bigg\| 1/n \sum_i^n W_i \indic_{j,good} - &\E_\D [W \indic_{j, good}] \Bigg\| \ge t \Bigg) \le P \left(\frac{1}{n} \left\| \sum_i^{n_e} Y_i - \E_{\D_j} [Y] \right\| \ge t_1 \right) \\
    &+ P \left(\E_{\D_j} [Y] \left|\frac{1}{n} \sum_i^n Z_i - \pi_j^* \right| \ge t_2 \right) +  P \left(\sum_{i=1}^n Z_i \ge n_e + 1 \right).
\end{align*}
Now we find a sub-Gaussian norm of $Y_i$, which can be computed as
\begin{align*}
    \|Y\|_{\psi_2} &= \sup_{p\ge1} p^{-1/2} \E_{\D_j} [|w_1 \indic_{j, good} \vdot{v + \mu_j^* - \mu_1^*}{s}|^p ]^{1/p} \\
    &\le 10 \rho_\pi \exp(-{R_{j1}^*}^2 / 64 {\sigma_1^*}^2) \sup_{p\ge1} p^{-1/2} (\E_{\D_j} [|\vdot{v}{s}|^p]^{1/p} + R_{j1}^*) \\
    &\le 20 \rho_\pi (R_{j1}^* + \sigma_1^*) \exp(-{R_{j1}^*}^2 / 64{\sigma_1^*}^2).
\end{align*}
Meanwhile, we can set $n_e = n\pi_j^* + O(\sqrt{np \ln(k^2T/\delta)}) \le 2n\pi_j^*$ as previously to get a high probability bound for the number of samples from $j^{th}$ component. Using standard 1/2-covering argument for $d$-dimensional sub-Gaussian vector, we have
\begin{align*}
    P \left( \left\| \sum_{i=1}^{n_e} Y_i - \E_{\D_j} [Y] \right\| \ge n t_1 \right) &\le \exp \left(- \frac{(nt_1)^2}{n_e \|Y\|_{\psi_2}^2} + Cd \right),
\end{align*}
for some universal constant $C$. That is, $t_1 = O\left( \|Y\|_{\psi_2} \sqrt{\frac{n_e}{n}} \sqrt{\frac{d + \log(k^2T/\delta)}{n}} \right)$ with probability at least $1 - \delta / (k^2 T)$.

\paragraph{With $j \neq 1, \Eps_{j,1}^c$:} We can use the same trick with $Y_i = W_i | \Eps_{j} \cap \Eps_{j,1}^c$. Note that $p := P(\Eps_j \cap \Eps_{j,1}^c) \le \pi_j^* \exp(-{R_{j1}^*}^2 / 64{\sigma_1^*}^2)$. Sub-Gaussian norm of $Y_i$ in this case can be bounded with using one of the lemmas.
\begin{align*}
    \|Y\|_{\psi_2} &= \sup_{p\ge1} p^{-1/2} \E_{\D_j} [|w_1 \vdot{v + \mu_j^* - \mu_1^*}{s}|^p | \vdot{v}{u} \ge R_{j1}^*/5]^{1/p} \\
    &\le \sup_{p\ge1} p^{-1/2} (R_{j1}^* + \E_{\D_j} [|\vdot{v}{s}|^p | \vdot{v}{u} \ge R_{j1}^*/5]^{1/p}) \\
    &\le R_{j1}^* + \sup_{p\ge1} p^{-1/2} \sigma_j^* (2R_{j1}^*/5\sigma_j^* + (2p)^{1/2}) \\
    &\le 2R_{j1}^*,
\end{align*}
where $u$ is a unit vector in direction $\mu_j^* - \mu_1^*$. Since the probability of the bad event is very small, we divide the cases into $n \ge p^{-1/c}$ and $n \le p^{-1/c}$ for some $c > 2$, as we have done for handling bad events for mixing weights. From Bernstein's inequality, if $n \ge p^{-1/c}$, with probability at least $1 - \delta/k^2T$,
\begin{align*}
    n_e \le O \left( \ln (k^2T / \delta) + \sqrt{np \ln (k^2T / \delta)} \right).
\end{align*}
Using this, we can give a good bound for $t_1$ with high probability,
\begin{align*}
    t_1 &= O\left(\|Y\|_{\psi_2} \sqrt{n_e/n} \sqrt{(d + \ln(k^2T / \delta))/n} \right) \\
    &\le O\left(R_{j1}^* \sqrt{1/n \vee \sqrt{p/n}} \sqrt{\ln (k^2T / \delta)} \sqrt{(d + \ln (k^2T\delta))/n} \right) \\
    &\le \tilde{O} \left(R_{j1}^* p^{1/2c} \sqrt{d / n} \right),
\end{align*}
and
\begin{align*}
    t_2 = O(R_{j1}^* \ln(k^2T/\delta) / n) = O \left(R_{j1}^* p^{1/2c} \sqrt{\ln^2 (k^2T/\delta) / n} \right),
\end{align*}
getting a similar scale of fluctuation.

For the other case when $n \le p^{-1/c}$, we can again get $n_e = 0$ with probability at least $1 - p^{1-1/c}$, which is again greater than $1 - (1/(n^{c-2} k^{32})$. In this case, $t_1 = 0$ and $t_2 = \E[Y] p \le 2R_{j1}^* p^{1 - 1/c} / n \lesssim \sigma_1^* \pi_{min} / n$, which is again sufficiently small. 

In all cases, we have that the fluctuation conditioned on this bad event is $\tilde{O} (\sigma_1^* \pi_{min} \sqrt{d/n})$ with probability at least $1 - \delta/(k^2T) - 1/(n^{c-2}k^{32})$. 

\paragraph{With $j \neq 1, \Eps_{j,1} \cap \Eps_{j,2}^c$:} Let $Y_i = W_i \indic_{\Eps_{j,1}} | \Eps_{j,2}^c$. Then $p:= P(\Eps_{j,2}^c) \le 2 \pi_j^* \exp(-{R_{j1}^*}^2 / 64 (\sigma_1^* \vee \sigma_j^*)^2)$. We can again follow the same path as we have done for other bad events. The key step is to get a sub-Gaussian norm.
\begin{align*}
    \|Y\|_{\psi_2} &= \sup_{p\ge1} p^{-1/2} \E_{\D_j} [|w_1 \indic_{\Eps_{j,1}} \vdot{v + \mu_j^* - \mu_1^*}{s}|^p | \Eps_{j,2}^c ]^{1/p} \\
    &\le R_{j1}^* + \sup_{p\ge1} p^{-1/2} \E_{\D_j} [|w_1 \vdot{v}{s}|^p | \vdot{v}{\Delta_1} \ge {R_{j1}^*}^2 / 64]^{1/p} \\
    &\qquad + \sup_{p\ge1} p^{-1/2} \E_{\D_j} [|w_1 \vdot{v}{s}|^p | \vdot{v}{\Delta_j} \ge (\sigma_j^*/\sigma_1^*)^2 {R_{j1}^*}^2 / 64]^{1/p} \\
    &\le R_{j1}^* + 2\sigma_j^* \sup_{p\ge1} p^{-1/2} \E_{v \sim \mathcal{N}(0,I_d)} 
    [ |\vdot{v}{s}|^p | \vdot{v}{u} \ge {R_{j1}^*} / 4(\sigma_j^* \vee \sigma_1^*)]^{1/p} \\
    &\le R_{j1}^* + 2\sup_{p\ge1} p^{-1/2} (\sigma_j^* (2R_{j1}^*/4(\sigma_j^* \vee \sigma_1^*) + \sqrt{2p})) \\
    &\le 3R_{j1}^*. 
\end{align*}
The rest of the step is similar to the previous case. We can thus again get a results that the deviation in this case is also $\tilde{O} \left(R_{j1}^* \exp(-{R_{j1}^*}^2/ (64\cdot (2c) (\sigma_1^* \vee \sigma_j^*)^2)) \sqrt{d/n} \right)$ with probability at least $1 - \delta/(k^2T) - 1/(n^{c-2}k^{32})$.

\paragraph{With $j \neq 1, \Eps_{j,1} \cap \Eps_{j,2} \cap \Eps_{j,3}^c $:} Let $Y_i = W_i \indic_{\Eps_{j,1} \cap \Eps_{j,2}} | \Eps_{j,3}^c$. Again, $p:= P(\Eps_{j,3}^c) \le 2 \pi_j^* \exp(-{R_{j1}^*}^2 / 64 (\sigma_1^* \vee \sigma_j^*)^2)$ again. This time, we can invoke lemma 3.9 with $\alpha = {R_{j1}^*}^2 / 64(\sigma_1^* \vee \sigma_j^*)^2$, to get
\begin{align*}
    \|Y\|_{\psi_2} &= \sup_{p\ge1} p^{-1/2} \E_{\D_j} [|w_1 \indic_{\Eps_{j,1}} \vdot{v + \mu_j^* - \mu_1^*}{s}|^p | \Eps_{j,3}^c ]^{1/p} \\
    &\le R_{j1}^* + \sup_{p\ge1} p^{-1/2} \E_{\D_j} [|w_1 \vdot{v}{s}|^p | \|v\|^2/{\sigma_j^*}^2 \ge d + 2\sqrt{d\alpha} + 2\alpha]^{1/p} \\
    &\qquad + \sup_{p\ge1} p^{-1/2} \E_{\D_j} [|w_1 \vdot{v}{s}|^p | \|v\|^2/{\sigma_j^*}^2 \le d - 2\sqrt{d\alpha}]^{1/p} \\
    &\le R_{j1}^* + c_1 \sigma_j^* + \sigma_j^* \sup_{p\ge1} p^{-1/2} \E_{\D_j} [|\vdot{v}{s}|^p | \|v\|^2 \ge d + 2\sqrt{d\alpha} + 2\alpha]^{1/p}  \\
    &\le R_{j1}^* + c_1\sigma_j^* + \sigma_j^* \sup_{p\ge1} p^{-1/2} ((64\alpha)^{1/2} + 4^{1/p} (8\alpha + p)^{1/2}) \\
    &\le 3R_{j1}^*. 
\end{align*}
The rest of the step is similar to the previous case. We can thus again get a results that the deviation in this case is also $\tilde{O} \left(R_{j1}^* \exp(-{R_{j1}^*}^2/ (64 \cdot 2c (\sigma_1^* \vee \sigma_j^*)^2)) \sqrt{d/n} \right)$ with probability at least $1 - \delta/(k^2T) - 1/(n^{c-2}k^{32})$.

\paragraph{With $\Eps_1$:} Let $Y_i = W_i | \Eps_{1}$, $p := P(\Eps_{1}) = \pi_1^*$. $\|Y\|_{\psi_2}$ can be easily verified such that
\begin{align*}
    \|Y\|_{\psi_2} &= \sup_{p \ge 1} p^{-1/2} \E_{\D_1} [|w_1 \vdot{v}{s}|^p]^{1/p} \le K \sigma_1^*,
\end{align*}
for some constant K. We can set $n_e = 2n\pi_1^*$ as usual, to get the statistical error by those samples are
\begin{align*}
    \tilde{O}\left( \sigma_1^* \sqrt{\pi_1^*} \sqrt{d/n} \right),
\end{align*}
with probability at least $1 - \delta/(k^2T)$.

Summing up every terms, the entire error is
\begin{align*}
    e_\mu &\lesssim \sigma_1^* \left( \sqrt{\pi_1^*} \sqrt{d/n} + \sqrt{d/n} \rho_\pi \sum_{j \neq 1} \R{j1} \exp\left(-\R{j1}^2/(128c\sigt{j}{1}^2) \right) \right) \\
    &\lesssim \sigma_1^* \pi_{min} \epsilon,
\end{align*}
with probability $1 - O(\delta/kT + 1/n^{c-2} \cdot 1/k^{31}))$. In consequence, 
\begin{align*}
    \|\tmu_1^+ - \mu_1^*\| \le \pi_{min} \sigma_1^* (\epsilon + \gamma \|\tmu_1 - \mu_1^*\| / \sigma_1^* ) / \tpi_1^+ \le \sigma_1^* (\epsilon + \gamma D_m ).   
\end{align*}
Similarly to mixing weights, after taking union bound over all $k$ components and $T = O(\log(1/\epsilon))$ iterations, we get $\|\tmu_1^{(T)} - \mu_1^*\| \le \sigma_1^* \epsilon$ with probability $1 - O(\delta + T/n^{c-2} \cdot 1/k^{30})$. 
\end{proof}

\subsection{Concentration in Variances}
\begin{proof}

With finite samples, the finite-sample EM iteration for variance is
\begin{align*}
    \tsigma_1^+{}^2 - {\sigma_1^*}^2 &= \left( \sum_{i=1}^n w_{1,i} \|X_i - \tmu_1^+\|^2 \right) / \left(d \sum_{i=1}^n w_{1,i} \right) - {\sigma_1^*}^2 \\
    &= \left( 1/n \sum_{i=1}^n w_{1,i} \|X_i - \mu_1^*\|^2) / (d \tpi_1^+) \right) - {\sigma_1^*}^2 - \|\tmu_1^+ - \mu_1^*\|^2 / d \\
    &= \underbrace{\left( \frac{1}{n} \sum_{i=1}^n w_{1,i} (\|X_i - \mu_1^*\|^2 - d{\sigma_1^*}^2) - \E[w_1 (\|X - \mu_1^*\|^2 - d{\sigma_1^*}^2)] \right)}_{e_\sigma}  / (d\tpi_1^+) \\
    &\qquad + \frac{\E[\Delta_w (\|X - \mu_1^*\|^2 - d{\sigma_1^*}^2)]}{d\tpi_1^+} - \frac{\|\tmu_1^+ - \mu_1^*\|^2}{d}.
\end{align*}

In order to be more precise, we may need to target $\tO(\sqrt{1/nd})$ for statistical precision. But it is enough to proceed more roughly, since $\tilde{O}(\sqrt{1/n})$ is enough to guarantee $\epsilon/\sqrt{d}$ statistical error with $n = \Omega(d)$ samples (we need this for estimating means). Let us define $W_i = w_{1,i} (\|X_i - \mu_1^*\|^2 - d{\sigma_1^*}^2)$ and use the decomposition strategy as we have done for $\mu$.

\paragraph{With $j \neq 1, \Eps_{j,good}$:} Let $Y_i = W_i \indic_{j,good} | \Eps_j$, $p:= \pi_j^*$. First task is, similarly, to find a sub-exponential norm (since now $Y$ are sum of squared variables). We first compute it,
\begin{align*}
    \|Y\|_{\psi_1} &= \sup_{p\ge1} p^{-1} \E_{\D_j} [|w_1 \indic_{j, good} (\|v + \mu_j^* - \mu_1^*\|^2 - d{\sigma_1^*}^2)|^p]^{1/p} \\
    &\le |w_1\indic_{j,good}| \sup_{p\ge1} p^{-1} \Bigg( {R_{j1}^*}^{2} + \E_{\D_j} [\|v\|^{2p}]^{1/p} + 2 \E_{\D_j} [|\vdot{v}{\mu_j^* - \mu_1^*}|^p]^{1/p} + \E_{\D_j} [|d{\sigma_1^*}^2|^p ]^{1/p} \Bigg) \\
    &\le 10 \rho_\pi \exp \left( -\frac{{R_{j1}^*}^2}{64(\sigma_1^* \vee \sigma_j)^2} \right) \sup_{p\ge1} p^{-1} \left( ({R_{j1}^*}^2 + d{\sigma_1^*}^2) + K R_{j1}^* \sigma_j^* \sqrt{p} + K' d {\sigma_j^*}^2 p \right) \\
    &\le C\rho_\pi \exp \left( - {R_{j1}^*}^2 / 64 (\sigma_1^* \vee \sigma_j^*)^2 \right) \left( {R_{j1}^*}^2 + d{\sigma_1^*}^2 + d{\sigma_j^*}^2 + R_{j1}^* \sigma_j^* \right) \\
    &\le C \rho_\pi \left( \R{j1}^2 + d \sigt{j}{1}^2 \right) \exp \left( - {R_{j1}^*}^2 / 64 (\sigma_1^* \vee \sigma_j^*)^2 \right).
\end{align*}
Here we bound this term with the tail bound for sub-exponential random variable with $\psi_1$-norm $\|Y\|_{\psi_1}$. Note that for sub-exponential random variable, from \cite{vershynin2010introduction}, 
\begin{align*}
    \left| \sum_{i=1}^{n_e} (Y - \E[Y]) \right| \le \|Y\|_{\psi_1} O \left( \sqrt{ n_e \log(1/\delta')} + \log(1/\delta') \right),
\end{align*}
with probability $1-\delta'$. In order to decide the statistical fluctuation, we just need to pick the maximum among $\sqrt{p/n}$ and $1/n$, which is in effect same to the case when $Y$ is sub-Gaussian. For this event, we can set $n_e = 2n\pi_j^*$ as before to bound the number of samples. Thus, we get bound the statistical error as
\begin{align*}
    \tilde{O} \left( \rho_\pi \sqrt{\pi_j^*} \sqrt{1/n} \right) \left( \R{j1}^2 + d\sigt{j}{1}^2 \right) \exp \left(- \R{j1}^2 / 64\sigt{j}{1}^2 \right),
\end{align*}
with probability at least $1 - \delta/(k^2T)$. Note that $d$ term will be canceled out with division by $d$ at the end.

\paragraph{With $j \neq 1, \Eps_{j,1}^c$:} Similarly, we find sub-exponential norm of $Y = W | \Eps_j \cap \Eps_{j,1}^c$.
\begin{align*}
    \|Y\|_{\psi_1} &= \sup_{p\ge1} p^{-1} \E_{\D_j} [|w_1 (\|v + \mu_j^* - \mu_1^*\|^2 - d{\sigma_1^*}^2)|^p | \vdot{v}{\mu_j^* - \mu_1^*} \ge R_{j1}^*/5 ]^{1/p} \\
    &\le \sup_{p\ge1} p^{-1} \Bigg( ({R_{j1}^*}^2 + d{\sigma_1^*}^2) + \E_{\D_j} [\|v\|^{2p} | \vdot{v}{\mu_j^* - \mu_1^*} \ge R_{j1}^*/5]^{1/p} \\
    &\qquad \qquad \qquad + 2 \E_{\D_j} [|\vdot{v}{\mu_j^* - \mu_1^*}|^p | \vdot{v}{\mu_j^* - \mu_1^*} \ge R_{j1}^*/5]^{1/p} \Bigg) \\
    &\le {R_{j1}^*}^2 + d{\sigma_1^*}^2 + \sup_{p\ge1} p^{-1} \left( 2 R_{j1}^* \sigma_j^* (2{R_{j1}^*}/(5\sigma_j^*) + \sqrt{2p}) + (8 {R_{j1}^*}^2/25 + 4p + 2\pi^{1/p} (d+p-1) ) \right) \\
    &\le {R_{j1}^*}^2 + d{\sigma_1^*}^2 + \left( {R_{j1}^*}^2 + 5{\sigma_j^*}^2 + 3d {\sigma_j^*}^2 \right) \le 3 {R_{j1}^*}^2 + d {\sigma_1^*}^2 + 3d{\sigma_j^*}^2.
\end{align*}
Rest of the procedure follows similarly to the cases handled bad cases on means. We will get the statistical error of
\begin{align*}
    \tilde{O} \left( \left({R_{j1}^*}^2 + d(\sigma_1^* \vee {\sigma_j^*})^2 \right) p^{1/2c} \sqrt{1/n} \right) = \tO \left( \sqrt{1/n} \right) \left({R_{j1}^*}^2 + d(\sigma_1^* \vee {\sigma_j^*})^2 \right) \exp \left(-\R{j1}^2 / (128c \sigt{j}{1}^2) \right),
\end{align*}
with probability $1 - 1/(n^{c-2} k^{32})$. 

\paragraph{With $j \neq 1, \Eps_{j,2}^c$:} For this, we can follow exactly same procedure to $\Eps_{j,1}^c$ case to get the same result. 

\paragraph{With $j \neq 1, \Eps_{j,3}^c$:} We need to bound the $p^{th}$ norm conditioned on $\|v\|^2 / {\sigma_j^*}^2 \ge r^2$ where $r^2 := d + 2\sqrt{d\alpha} + 2\alpha$, where $\alpha = {R_{j1}^*}^2 / 64(\sigma_1^* \vee \sigma_j^*)^2$. 
\begin{align*}
    \|Y\|_{\psi_1} &= \sup_{p \ge 1} p^{-1} \left( {R_{j1}^*}^2 + d{\sigma_1^*}^2 + \E_{\D_j} [\|v\|^{2p} | \|v\|^2 / {\sigma_j^*}^2 \ge r^2]^{1/p} + 2 R_{j1}^* \E_{\D_j} [|\vdot{v}{s}|^p | \|v\|^2/{\sigma_j^*}^2 \ge r^2]^{1/p} \right) \\
    &\le {R_{j1}^*}^2 + d{\sigma_1^*}^2 \\
    &\qquad \qquad + \sup_{p\ge1} p^{-1} \Bigg({\sigma_j^*}^2 \E_{v \sim \mathcal{N}(0,I_d)} [\|v\|^{2p} | \|v\|^2 \ge r^2]^{1/p} + 2R_{j1}^* \sigma_j^* \E_{v \sim \mathcal{N}(0,I_d)} [|\vdot{v}{s}|^{p} | \|v\|^2 \ge r^2]^{1/p} \Bigg),
\end{align*}
where $s$ is unit vector in direction $\mu_j^* - \mu_1^*$. We can invoke Lemma \ref{lemma:vl2_conditioned_vl2} to get
\begin{align*}
    \E_{v \sim \mathcal{N}(0,I_d)} [\|v\|^{2p} | \|v\|^2 \ge r^2]^{1/p} &\le (4r^2) + 4^{1/p} (d+4p) \exp(-r^2/8p)\\
    &\le 4d + (R_{j1}^*/(\sigma_1^* \vee \sigma_j^*)) \sqrt{d} + ({R_{j1}^*}^2 / 8 (\sigma_1^* \vee \sigma_j^*)^2) + 4^{1/p} (d + 4p) \exp(-d^2/8p) \\
    &\le 5d + 5p + (R_{j1}^*/\sigma_j^*)\sqrt{d} + {R_{j1}^*}^2 / (8{\sigma_j^*}^2).
\end{align*}
We can also invoke Lemma \ref{lemma:vp_conditioned_vl2} to get
\begin{align*}
    \E_{v \sim \mathcal{N}(0,I_d)} [|\vdot{v}{s}|^{p} | \|v\|^2 \ge r^2]^{1/p} &\le (64\alpha)^{1/2} + 4^{1/p} (8\alpha + 2p)^{1/2} \exp(-\alpha/2p) \\
    &\le 2 R_{j1}^* / (\sigma_1^* \vee \sigma_j^*) + 4p^{1/2}.
\end{align*}
Now we can further continuing to bound sub-exponential norm as
\begin{align*}
    \|Y\|_{\psi_1} &\le {R_{j1}^*}^2 + d {\sigma_1^*}^2 + \sup_{p\ge1} p^{-1} \Bigg({R_{j1}^*}^2/8 + R_{j1}^*\sigma_j^* \sqrt{d} + 5d{\sigma_j^*}^2 + 5p{\sigma_j^*}^2 + 2R_{j1}^* \sigma_j^* + 4p^{1/2} {\sigma_j^*}^2 \Bigg) \\
    &\le 4 {R_{j1}^*}^2 + d {\sigma_1^*}^2 + 15 d {\sigma_j^*}^2 + 3R_{j1}^* \sigma_j^* \sqrt{d}.
\end{align*}
Hence following the same procedure for bad events, statistical fluctuation will be again smaller than 
\begin{align*}
    \tO \left(\sqrt{1/n} \left({R_{j1}^*}^2 + d(\sigma_1^* \vee {\sigma_j^*})^2 \right) \exp \left(-\R{j1}^2 / (128c \sigt{j}{1}^2) \right) \right).
\end{align*} 

\paragraph{With $j = 1$:} Finally, we need to handle this case. We recall Lemma \ref{lemma:chi_square_pnorm_bound}, sub-exponential norm will be less than 
\begin{align*}
    \|Y\|_{\psi_1} &= \sup_{p\ge1} p^{-1} \E_{\D_1} [|w_1 (\|v\|^2 - d {\sigma_1^*}^2)|^p]^{1/p} \\
    &\le 3 d{\sigma_1^*}^2.
\end{align*}
We can use Proposition \ref{lemma:indic_prob_decompose} with $n_e = 2n\pi_1^*$. Similar to $\Eps_{j,good}$ case, the statistical fluctuation is $\tilde{O} (d {\sigma_1^*}^2 \sqrt{\pi_1^*} \sqrt{1/n})$ with probability at least $1 - \delta/(k^2T)$. 

Now collecting all $O(k)$ error terms, 
\begin{align*}
    e_\sigma \lesssim \sqrt{1/n} \left( d {\sigma_1^*}^2 \sqrt{\pi_1^*} + \rho_\pi \sum_{j\neq1} (\R{j1}^2 + d\sigt{j}{1}^2) \exp \left(-\R{j1}^2 / (128c \sigt{j}{1}^2) \right) \right) \le \sqrt{d} {\sigma_1^*}^2 \epsilon {\pi_1^*}, 
\end{align*}
with probability at least $1 - \delta/(kT) - 1/(n^{c-2}\cdot k^{31})$. Now we can conclude that,
\begin{align*}
    |\tsigma_1^+{}^2 - {\sigma_1^*}^2| &\le e_\sigma / (d\pi_1^*) + (\pi_1^+ / \tpi_1^+) \left( {\sigma_1^+}^2 - {\sigma_1^*}^2 + \|\mu_1^+ - \mu_1^*\|^2 / d \right) -  \|\tmu_1^+ - \mu_1^*\|^2 / d \\
    &\le {\sigma_1^*}^2 (\epsilon/\sqrt{d} + (1+\epsilon) \gamma D_m/ \sqrt{d} + 3 \epsilon D_m / d ) \\ 
    &\le {\sigma_1^*}^2 (\epsilon' + \gamma D_m) / \sqrt{d},
\end{align*}
with some constant rescaling of $\epsilon$ to $\epsilon'$.
\end{proof}

\section{Proofs for Section \ref{section:sample_optimality} }
\subsection{Proof for Lemma \ref{lemma:k_means_converge}}
\label{Appendix:k_means_converge}
\begin{proof}
    Let us first check the correctness of $\pi_i$ and $\mu_i$. This proof is reminiscent of the analysis on population EM when $D_m \ge 1/2$. The step 1 and 2, which are essentially the stpes of the k-mean algorithm, can be also considered as a variant of E-step and M-step, with a rule (for $1^{st}$ component):
    \begin{align}
        \label{eq:k_means}
        \mbox{(E'-step)}: \ w_1 &= \indic_{\|X - \mu_1\|^2 \le \|X - \mu_j\|^2, \forall j \neq 1}, \nonumber \\
        \mbox{(M'-step)}: \ \pi_1^+ &= \E_\D[w_1], \nonumber \\
                        \mu_1^+ &= \E_\D[w_1 X] / \E_\D[w_1].
    \end{align}
    Let us follow the proof strategy of population EM. As before, note that $\pi_1^* = \E_\D[w_1^*]$, where $w_1^*$ is a weight constructed at E-step with the standard EM algorithm. Regardless of different weight assignment rules, the estimation error after one step can be represented as
    \begin{align*}
        &\pi_1^+ - \pi_1^* = \E_{\D} [w_1] - \E_{\D} [w_1^*] = \E_{\D} [\Delta_w], \\
        &\mu_1^+ - \mu_1^* = \E_{\D} [w_1(X - \mu_1^*)] / \E_{\D} [w_1] = \E_{\D} [\Delta_w (X - \mu_1^*)] / \E_{\D} [w_1],
    \end{align*}
    which is exactly the same as in standard population EM. We similarly split the errors from other components and own component. 
    
    \paragraph{When $j \neq 1$:} Let $v = X - \mu_j^*$ and define good event as 
    \begin{align*}
        \Eps_j' = \{ \vdot{v}{\mu_j - \mu_1} \ge -{R_{j1}^*}^2 / 4 \}.
    \end{align*}
    Since $\|\mu_j - \mu_1\| \le \|\mu_j^* - \mu_1^*\| + \|\mu_j^* - \mu_j\| + \|\mu_1^* - \mu_1\| \le 3R_{j1}^*/2$, we have $P(\Eps_j^c) \le \exp(- {R_{j1}^*}^2 / (72 {\sigma_j^*}^2))$. Observe that, 
    \begin{align*}
        \|X - \mu_1\|^2 \le \|X - \mu_j\|^2 
        \iff &2\vdot{v}{\mu_j^* - \mu_1} + \|\mu_j^* - \mu_1\|^2 \le 2\vdot{v}{\mu_j^* - \mu_j} + \|\mu_j^* - \mu_j\|^2 \\
        \iff & 2\vdot{v}{\mu_j - \mu_1} \le  \|\mu_j^* - \mu_j\|^2 - \|\mu_j^* - \mu_1\|^2 \\
        \implies & \vdot{v}{\mu_j - \mu_1} \le -{R_{j1}^*}^2 / 4.
    \end{align*}
    Note that by the initialization condition, $\|\mu_j^* - \mu_j\| \le R_{j1}^*/4$ and $\|\mu_j^* - \mu_1\| \ge 3R_{j1}^*/4$. That is, if $\indic_{\Eps_j'} = 1$, then $w_1 = 0$. We can conclude that $\E_{\D_j} [w_1] \le \exp(- {R_{j1}^*}^2 / (72 {\sigma_j^*}^2))$ for all $j \neq 1$. Now using Lemma \ref{lemma:weight_good}, we can also see that 
    \begin{align*}
        |\E_{\D_j} [\Delta_w]| \le \E_{\D_j} [|w_1 - w_1^*|] &\le 5 (1 + \pi_1^*/\pi_j^*) \exp(-{R_{j1}^*}^2 / 64(\sigma_1^* \vee \sigma_j^*)^2) + \exp(- {R_{j1}^*}^2 / (72 {\sigma_j^*}^2)).
    \end{align*}
    Summing up all errors from $j \neq 1$, 
    \begin{align*}
        \sum_{j\neq 1} \pi_j^* |\E_{\D_j} [\Delta_w]| &\le 6 \sum_{j \neq 1} (\pi_1^* + \pi_j^*) \exp \left(-{R_{j1}^*}^2 / 128(\sigma_1^* \vee \sigma_j^*)^2 \right),
    \end{align*}
    which can be bounded by some small constants $c_1 < 0.01$ with Lemma \ref{lemma:error_bounds_finite_sample}, given good separation condition.
    
    Similarly, the errors to means are also small: for any unit vector $s \in \mathbb{S}^{d-1}$,
    \begin{align*}
        |\E_{\D_j} [\Delta_w \vdot{v}{\mu_j^* - \mu_1^* + s}]| &\le R_{j1}^* |\E_{\D_j} [\Delta_w]| + |\E_{\D_j} [\Delta_w \vdot{v}{s}]| \\
        &\lesssim R_{j1}^* (1 + \pi_1^* / \pi_j^*) \exp \left( -\R{j1}^2 / 128\sigt{1}{j}^2 \right),
    \end{align*}
    where we applied the same technique to bound as in Corollary \ref{corollary:error_bound_Dm_1over2}. Summing up over $j \neq 1$ and applying Lemma \ref{lemma:error_bounds_finite_sample} gives the similar result, $\sum_{j\neq 1} \pi_j^* \E_{\D_j} [\Delta_w \vdot{X - \mu_1^*}{s}] \le c_2 \sigma_1^* \pi_1^*$ for some small constant $c_2$. 
    
    \paragraph{When $j = 1$:} Recall that when we compute errors from its own component, we bounded $\E_{\D_1} [1 - w_1]$. 
    \begin{align*}
        \E_{\D_1} [1 - w_1] = \sum_{l \neq 1} \E_{\D_1} [w_l] \le \sum_{l \neq 1} \exp(-\R{l1}^2 / 72{\sigma_1^*}^2 ) \le c_1,
    \end{align*}
    for some small constant $c_1 < 0.01$. Meanwhile, in the population EM, we have shown that $\E_{\D_1} [1 - w_1^*] \le c_2$ for small constant $c_2$. Hence, $\E_{\D_1} [\Delta_w] = \E_{\D_1} [(1 - w_1) - (1 - w_1^*)] \le c_3$ for small constant $c_3$. Similarly, we can bound the errors for means,
    \begin{align*}
        |\E_{\D_1} [\Delta_w \vdot{v}{s}]| \le c_4 \sigma_1^*,
    \end{align*}
    for some small constant $c_4$. 
    
    Collecting errors from all components gives that $|\pi_1^+ - \pi_1^*| \le c_\pi \pi_1^*$ and $\| \mu_1^+ - \mu_1^* \| \le c_\mu \sigma_1^*$ for some small constants $c_\pi < 0.5, c_\mu < 4$. 
    
    \paragraph{Population to Finite-Sample:} We can reproduce the proof for finite-sample EM with modified rule \eqref{eq:k_means}. To see this, observe that 
    \begin{align*}
        \tpi_1^+ - \pi_1^* &= \frac{1}{n} \sum_{i=1}^n w_{1,i} - \E_{\D} [w_1], \\   
        \tmu_1^+ - \mu_1^* &= \left( \frac{1}{n} \sum_{i=1}^n w_{1,i} (X - \mu_1^*) - \E_{\D} [w_1 (X - \mu_1^*)] + \E_{\D} [\Delta_w (X - \mu_1^*)] \right) / \left( \frac{1}{n} \sum_{i=1}^n w_{1,i} \right), \\   
    \end{align*}
    which is exactly in the same format as when we used standard EM iteration. Note that the proof of concentration in finite-sample EM holds for any different rule of assigning weights in E-step, as long as the probability of bad events is exponentially small. In this case it is as small as $\exp(-\R{j1}^2 / 72 {\sigma_j^*}^2)$. Hence the same procedure in Appendix \ref{appendix:finite_sample} can give a desired finite-sample error bound with high probability.
    
    \paragraph{Estimating ${\sigma_i^*}^2$:} From the analysis of estimating the mixing weights, we can conclude that the elements in each cluster $C_i$ are mostly from $i^{th}$ component and only a few fraction of elements are from the other components (say, less than $1\%$). Furthermore, only less than $1\%$ of samples from $i^{th}$ component are missing. Thus, each cluster $C_i$ can be considered as $2\%$-corrupted data from the $i^{th}$ component. In order to retrieve $\sigma_i^2$ such that $|\sigma_i^2 - {\sigma_i^*}^2| \le 0.5 {\sigma_i^*}^2 / \sqrt{d}$, we can consider taking a median-like quantity among pairwise distances of samples. 
    
    In each cluster, let the elements be in some fixed order which is pre-defined before we see the entire dataset. First let us consider the case when there is no corruption in each cluster. That is, each cluster has the true samples from its own component. Without loss of generality, let us focus on the first cluster $C_1$. Let the elements in $C_1$ as $X_1, X_2, ..., X_m$ where $m = |C_1|$. Since all $X_i \sim \mathcal{N} (\mu_1^*, {\sigma_1^*}^2 I_d)$, distribution of $X_i - X_{i+1}$ follows $\mathcal{N}(0, 2{\sigma_1^*}^2 I_d)$. Hence, $\|X_i - X_{i+1}\|^2 / (2{\sigma_1^*}^2)$ is a chi-square random variable with $d$ degrees of freedom. 
    
    Let $F(x)$ be the cdf function of a chi-square distribution with $d$ degrees of freedom. Consider cdf value $x_l := F(d - \sqrt{d}/2)$ and $x_r := F(d + \sqrt{d}/2)$. We can numerically check that $x_r - \alpha_d \ge 0.1$ and $\alpha_d - x_l \ge 0.1$ where $\alpha_d = F(d)$ as defined in the Algorithm \ref{alg:k-means} (for large $d$, the pdf of chi-square distribution is very well-approximated by normal distribution). 
    
    Now let us define $r_i = \|X_{2i} - X_{2i-1}\|^2 / (2 {\sigma_1^*}^2)$ and $r_i' = \|X_{2i+1} - X_{2i}\|^2 / (2 {\sigma_1^*}^2)$ for $i = 1, 2, ..., m/2 - 1$. Then let $Z, Z'$ be the portion of $r_i$s such that $r_i \le d - \sqrt{d}/2$ and $r_i' \le d - \sqrt{d}/2$ respectively. By standard concentration of Bernoulli random variable, both $Z$ and $Z'$ are well concentrated around $x_l$ with probability at least $1 - \delta/k$, given $m = \Omega (\log (k/\delta))$ samples (this holds since we generate $n = \Omega(\pi_{min}^{-1} \log(k/\delta))$ samples from mixture distribution). Note that the key point here is, there is no probabilistic dependency between $r_i$s for all $i$, and similarly between $r_i'$s for all $i$. 
    
    Finally, we return to the 2\% corrupted data from the first component. In this set, we see all adjacent pairs $\|X_{i+1} - X_i\|^2 / (2 {\sigma_1^*}^2 )$ for all $i = 1, 2, ..., |C_i| - 1$. This is because due to the insertion of wrong samples and deletion of authentic samples, the parity of original index might have changed. By looking at all adjacent pairs, we can look at both $r_i$ and $r_i'$. Note that 2\% corruption can at most corrupt 4\% of original $r_i$s and $r_i'$s respectively. Fortunately, we have 10\% margin from $\alpha_d$. That is, in the corrupted set, it is still guaranteed that $\alpha_d^{th}$ value is greater than $d - \sqrt{d}/2$, which is a standard argument for median-type estimators. The similar argument holds for the other direction. 
    
    In conclusion, if we take $\alpha_d^{th}$ value among distances of all adjacent pairs in cluster $C_1$, that value is within ${\sigma_1^*}^2 [2d - \sqrt{d}, 2d + \sqrt{d}]$ range with high probability. We get a desired initialization parameter for variances by dividing the quantity by $2d$.  
\end{proof}

\subsection{Proof for Lemma \ref{lemma:tv_implies_param}}
\label{Appendix:tv_implies_param}
We define some additional notations that will be used in this section. We use $P_{\G} (\cdot)$ to denote the probability of some event when underlying distribution is the candidate $\G$. Similarly, $P_{\G^*} (\cdot)$ denotes the probability when underlying distribution is true mixture $\G^*$. We use $P_{\G} (\cdot | X \sim j^{th})$ to denote the probability of event when $X$ comes from $j^{th}$ component in candidate distribution $\G$. $P_{\G^*} (\cdot | X \sim j^{th})$ is defined in a similar way. We use $R_{j1}$ to denote $\| \mu_j - \mu_1^* \|$ for $j \neq 1$.

\begin{proof}
    Suppose the conclusion is not true, {\it i.e.}, $\exists i \in [k]$ s.t. $\|\mu_i^* - \mu_j\| / \sigma_i^* \ge 16\sqrt{\log(1/\pi_{min})}$, $\forall j \in [k]$. Without loss of generality, let $\mu_1^*$ is far from all $\mu_j$ by at least $16 \sqrt{\log(1/\pi_{min})}$. We consider the cases when $d \ge 128 \log(1/\pi_{min})$ and $d \le 128 \log(1/\pi_{min})$. 
    
    \paragraph{Case I. $d \ge 128 \log(1/\pi_{min})$:} We define an event,
    \begin{align*}
        \Eps = \{ \|X - \mu_1^*\|^2/{\sigma_1^*}^2 + d \log {\sigma_1^*}^2 \le \|X - \mu_j\|^2/\sigma_j^2 + d \log {\sigma_j}^2, \quad \forall j \in [k] \}.
    \end{align*}
    Our goal is to show that $P_{\G^*} (\Eps) \ge 3\pi_{min}/4$ and $P_\G (\Eps) \le \pi_{min} / 2$. Then, by the definition of total variation distance, $\|\G - \G^*\|_{TV} \ge |P_{\G^*} (\Eps) - P_\G (\Eps)| \ge \pi_{min}/4$. 
    
    \paragraph{Probability from true distribution:} Let us first show $P_{\G^*} (\Eps) \ge 3\pi_{min}/4$. It suffices to show that $P_{\G^*} (\Eps | X \sim 1^{st}) \ge 3/4$. Thus, we are considering the event
    \begin{align}
        \label{eq:tv_event}
        \|v\|^2 / {\sigma_1^*}^2 + d \log({\sigma_1^*}^2) \le \|v + \mu_1^* - \mu_j\|^2 / \sigma_j^2 + d \log({\sigma_j}^2),
    \end{align}
    where $v \sim \mathcal{N}(0, {\sigma_1^*}^2 I)$. Similarly to we have seen in previous proofs for EM, we divide the cases into when $\sigma_1^* \ge \sigma_j$ and $\sigma_1^* \le \sigma_j$. 
    
    When $\sigma_1^* \ge \sigma_j$, let $x = ({\sigma_1^*}^2 - {\sigma_j}^2) / {\sigma_j}^2$. After rearranging \eqref{eq:tv_event}, we get
    \begin{align*}
        -\frac{\|v\|^2}{{\sigma_1^*}^2}x + d \log(1 + x) \le \|\mu_1^* - \mu_j\|^2 / \sigma_j^2 + 2\vdot{v}{\mu_1^* - \mu_j} / \sigma_j^2,
    \end{align*}
    Then the probability of $P_{\G^*} (\Eps^c | X \sim 1^{st})$ is less than 
    \begin{align*}
        P \Bigg(\vdot{v}{\mu_1^* - \mu_j} \le -\frac{\|\mu_1^* - \mu_j\|^2}{4} \Bigg) &+ P\left(-\frac{\|v\|^2}{{\sigma_1^*}^2} x + d \log(1+x) \ge \frac{\|\mu_1^* - \mu_j\|^2}{2\sigma_j^2} \right) \\
        &\le \exp( -{R_{j1}}^2 / 32{\sigma_1^*}^2 ) + P\left( \frac{\|v\|^2}{d{\sigma_1^*}^2} \le \frac{1}{x} \left(\log (1+x) - \frac{R_{j1}^2}{2 d \sigma_j^2} \right) \right).
    \end{align*}
    Using the similar trick as before, first consider when $0 \le x \le 3/4$. Then,
    \begin{align*}
        P\left( \frac{\|v\|^2}{d{\sigma_1^*}^2} \le \frac{\log (1+x)}{x} - \frac{R_{j1}^2}{2 d \sigma_j^2 x} \right) &\le P\left( \frac{\|v\|^2}{d{\sigma_1^*}^2} \le 1 - \frac{x}{4} - \frac{R_{j1}^2}{2 d \sigma_j^2 x} \right) \\
        &\le P\left( \frac{\|v\|^2}{d{\sigma_1^*}^2} \le 1 - 2 \sqrt{\frac{R_{j1}^2}{2d \sigma_j^2} } \right) \le \exp(-R_{j1}^2/2\sigma_j^2).
    \end{align*}
    When $x \ge 3/4$, 
    \begin{align*}
        P\left( \frac{\|v\|^2}{d{\sigma_1^*}^2} \le \frac{\log (1+x)}{x} - \frac{R_{j1}^2}{2 d \sigma_j^2 x} \right) &\le P\left( \frac{\|v\|^2}{d{\sigma_1^*}^2} \le 1 - \frac{1}{4} - \frac{R_{j1}^2}{2 d ({\sigma_1^*}^2 - \sigma_j^2)} \right) \\
        &\le P\left( \frac{\|v\|^2}{d{\sigma_1^*}^2} \le 1 - 2 \sqrt{\frac{R_{j1}^2}{8 d ({\sigma_1^*}^2 - \sigma_j^2)} } \right) \le \exp(-R_{j1}^2/8{\sigma_1^*}^2).
    \end{align*}
    In either case, $P_{\G^*} (\Eps^c | X \sim 1^{st}) \le 2\exp(-R_{j1}^2 / 32{\sigma_1^*}^2) \le 2 \pi_{min}^8 \le 1/4$. 
    
    When $\sigma_1^* \le \sigma_j$, let $x = ({\sigma_j}^2 - {\sigma_1^*}^2) / {\sigma_j}^2$ and we similarly rearrange \eqref{eq:tv_event}. 
    \begin{align*}
        \frac{\|v\|^2}{{\sigma_1^*}^2}x + d \log(1 - x) \le \|\mu_1^* - \mu_j\|^2 / \sigma_j^2 + 2\vdot{v}{\mu_1^* - \mu_j} / \sigma_j^2.
    \end{align*}
    We need to divide the cases when $\sigma_j^2 \le 8 {\sigma_1^*}^2$ and $\sigma_j^2 \ge 8 {\sigma_1^*}^2$. When $\sigma_j^2 \le 8 {\sigma_1^*}^2$, we proceed similarly to previous cases,
    \begin{align*}
        P(\Eps^c | X \sim 1^{st}) &\le P \left(\vdot{v}{\mu_1^* - \mu_j} \le -\frac{\|\mu_1^* - \mu_j\|^2}{4} \right) + P\left(\frac{\|v\|^2}{{\sigma_1^*}^2} x + d \log(1-x) \ge \frac{\|\mu_1^* - \mu_j\|^2}{2\sigma_j^2} \right) \\
        &\le \exp( -{R_{j1}}^2 / 32{\sigma_1^*}^2 ) + P\left( \frac{\|v\|^2}{d{\sigma_1^*}^2} \ge \frac{1}{x} \left(-\log (1-x) + \frac{R_{j1}^2}{2 d \sigma_j^2} \right) \right) \\
        &\le \exp( -{R_{j1}}^2 / 32{\sigma_1^*}^2 ) + P\left( \frac{\|v\|^2}{d{\sigma_1^*}^2} \ge 1 + \frac{x}{2} + \frac{R_{j1}^2}{2 d x \sigma_j^2} \right) \\
        &\le \exp( -{R_{j1}}^2 / 32{\sigma_1^*}^2 ) + P\left( \frac{\|v\|^2}{d{\sigma_1^*}^2} \ge 1 + 2\sqrt{\frac{R_{j1}^2}{8d \sigma_j^2}} + 2 \frac{R_{j1}^2}{8 d x \sigma_j^2} \right) \\
        &\le \exp( -{R_{j1}}^2 / 32{\sigma_1^*}^2 ) + \exp( -R_{j1}^2 / 8{\sigma_j}^2 ) \le \exp( -{R_{j1}}^2 / 32{\sigma_1^*}^2 ) + \exp( -{R_{j1}}^2 / 64{\sigma_1^*}^2 ) \\
        &\le \pi_{min}^8 + \pi_{min}^4 \le 1/4.
    \end{align*}
    When $\sigma_j^2 \ge 8 {\sigma_1^*}^2$, we first note that $x \ge 7/8$. Thus, $-\log(1-x)/x \ge 2.376$. We can then bound the term simply as
    \begin{align*}
        P(\Eps^c | X \sim 1^{st}) &\le P \left(\vdot{v}{\mu_1^* - \mu_j} \le -\frac{\|\mu_1^* - \mu_j\|^2}{4} \right) + P\left(\frac{\|v\|^2}{{\sigma_1^*}^2} x + d \log(1-x) \ge \frac{\|\mu_1^* - \mu_j\|^2}{2\sigma_j^2} \right) \\
        &\le \pi_{min}^8 + P\left( \frac{\|v\|^2}{d{\sigma_1^*}^2} \ge 1 + 1.376 \right).
    \end{align*}
    Using standard tail probability for chi-Square distribution, we get
    \begin{align*} 
         P\left( \frac{\|v\|^2}{d{\sigma_1^*}^2} \ge 1 + 1.376 \right) \le \exp( -1.376d / 8) \le \pi_{min}^{20},
    \end{align*}
    since we are considering the case when $d \ge 128 \log(1/\pi_{min})$. Therefore, we get $P(\Eps^c | X \sim 1^{st}) \le 1/4$ in all cases. It concludes that $P_{\G^*} (\Eps) \ge 3\pi_{min}/4$.
    
    \paragraph{Probability from candidates:} Now we show that $P_\G (\Eps) \le \pi_{min}/2$. Toward this goal, we need to show that for each $j \in [k]$, $P_\G (\Eps | X \sim j^{th}) \le \pi_{min}/2$. The corresponding event becomes
    \begin{align}
        \label{eq:tv_event2}
        \|v + \mu_j - \mu_1^*\|^2 / {\sigma_1^*}^2 + d \log({\sigma_1^*}^2) \le \|v\|^2 / {\sigma_j}^2 + d \log({\sigma_j}^2).
    \end{align}
    
    Let us first consider the case when $\sigma_j \le \sigma_1^*$. Now we set $x = ({\sigma_1^*}^2 - \sigma_j^2) / {\sigma_1^*}^2$. Then the rearrangement of \eqref{eq:tv_event2} gives
    \begin{align*}
        \|\mu_1^* - \mu_j\|^2 / {\sigma_1^*}^2 + 2\vdot{v}{\mu_j - \mu_1^*} / {\sigma_1^*}^2 \le \frac{\|v\|^2}{{\sigma_j}^2}x + d \log(1 - x).
    \end{align*}
    We bound the probability of this event similarly to previous cases. 
    \begin{align*}
        P(\Eps | X \sim j^{th}) &\le P \left(\vdot{v}{\mu_j - \mu_1^*} \le -\frac{\|\mu_1^* - \mu_j\|^2}{4} \right) + P\left(\frac{\|v\|^2}{{\sigma_j}^2} x + d \log(1-x) \ge \frac{\|\mu_1^* - \mu_j\|^2}{2{\sigma_1^*}^2} \right) \\
        &\le \exp( -{R_{j1}}^2 / 32{\sigma_j}^2 ) + P\left( \frac{\|v\|^2}{d{\sigma_j}^2} \ge \frac{1}{x} \left(-\log (1-x) + \frac{R_{j1}^2}{2 d {\sigma_1^*}^2} \right) \right) \\
        &\le \exp( -{R_{j1}}^2 / 32{\sigma_1^*}^2 ) + P\left( \frac{\|v\|^2}{d{\sigma_j}^2} \ge 1 + \frac{R_{j1}^2}{2 d x \sigma_j^2} \right) \\
        &\le \exp( -{R_{j1}}^2 / 32{\sigma_1^*}^2 ) + \exp(-R_{j1}^2 / (16 ({\sigma_1^*}^2 - \sigma_j^2)) ) \\
        &\le \pi_{min}^8 + \pi_{min}^{16} \le \pi_{min}/2.
    \end{align*}
    
    When $\sigma_j \ge \sigma_1^*$, the proof should be more delicate. First we rearrange \eqref{eq:tv_event2} to see that
    \begin{align*}
        \|\mu_1^* - \mu_j\|^2 / {\sigma_1^*}^2 + 2\vdot{v}{\mu_j - \mu_1^*} / {\sigma_1^*}^2 \le -\frac{\|v\|^2}{{\sigma_j}^2}x + d \log(1 + x),
    \end{align*}
    is the event to bound, where $x = (\sigma_j^2 - {\sigma_1^*}^2) / {\sigma_1^*}^2$. Here, we will consider three cases, $0 \le x \le 3/4$, $3/4 \le x \cap R_{j1}^2 / \sigma_j^2 \ge 32 \log(1/\pi_{min})$, and $R_{j1}^2 / \sigma_j^2 \le 32 \log(1/\pi_{min})$. First, if $0 \le x \le 3/4$, then $\log(1+x) \le x - x^2/4$ and $\sigma_j^2 \le 2 {\sigma_1^*}^2$, thus 
    \begin{align*}
        P(\Eps | X \sim j^{th}) &\le P \left(\vdot{v}{\mu_j - \mu_1^*} \le -\frac{7}{16} \|\mu_1^* - \mu_j\|^2 \right) + P\left(\frac{-\|v\|^2}{{\sigma_j}^2} x + d \log(1+x) \ge \frac{\|\mu_1^* - \mu_j\|^2}{8{\sigma_1^*}^2} \right) \\
        &\le \exp( -49 {R_{j1}}^2 / 512 {\sigma_j}^2 ) + P\left( \frac{\|v\|^2}{d{\sigma_j}^2} \le \frac{1}{x} \left(\log (1+x) - \frac{R_{j1}^2}{8 d {\sigma_1^*}^2} \right) \right) \\
        &\le \exp( -{R_{j1}}^2 / 32{\sigma_1^*}^2 ) + P\left( \frac{\|v\|^2}{d{\sigma_j}^2} \le 1 - \frac{x}{4} - \frac{R_{j1}^2}{8 d x {\sigma_1^*}^2} \right) \\
        &\le \exp( -{R_{j1}}^2 / 32{\sigma_1^*}^2 ) + \exp(-R_{j1}^2 / (32 {\sigma_1^*}^2 ) ) \le \pi_{min}^8 + \pi_{min}^{8} \le \pi_{min} / 2.
    \end{align*}
    If $3/4 \le x$ and $R_{j1}^2 / \sigma_j^2 \ge 32 \log(1/\pi_{min})$, then $\log(1+x)/x \le 3/4$, thus
    \begin{align*}
        P(\Eps | X \sim j^{th}) &\le P \left(\vdot{v}{\mu_j - \mu_1^*} \le -\frac{7}{16} \|\mu_1^* - \mu_j\|^2 \right) + P\left(\frac{-\|v\|^2}{{\sigma_j}^2} x + d \log(1+x) \ge \frac{\|\mu_1^* - \mu_j\|^2}{8{\sigma_1^*}^2} \right) \\
        &\le \exp( -49 {R_{j1}}^2 / 512 {\sigma_j}^2 ) + P\left( \frac{\|v\|^2}{d{\sigma_j}^2} \le \frac{1}{x} \left(\log (1+x) - \frac{R_{j1}^2}{8 d {\sigma_1^*}^2} \right) \right) \\
        &\le \exp( -6 \log(1/\pi_{min}) ) + P\left( \frac{\|v\|^2}{d{\sigma_j}^2} \le 1 - \frac{1}{4} - \frac{R_{j1}^2}{8 d x {\sigma_1^*}^2} \right) \\
        &\le \exp( -6 \log(1/\pi_{min}) ) + P\left( \frac{\|v\|^2}{d{\sigma_j}^2} \le 1 - 2 \sqrt{ \frac{1}{64} + \frac{R_{j1}^2}{64 d (\sigma_j^2 - {\sigma_1^*}^2)} } \right) \\
        &\le \exp( -6 \log(1/\pi_{min}) ) + \exp(- d/64 - R_{j1}^2 / 64\sigma_j^2 ) \le \pi_{min}^6 + \pi_{min}^{2.5} \le \pi_{min} / 2.
    \end{align*}
    Finally, if $R_{j1}^2/\sigma_j^2 \le 32 \log(1/\pi_{min})$, we take a different path. First of all, this can only happen when $\sigma_j^2 \ge 8 {\sigma_1^*}^2$ and $x \ge 7$, and $R_{j1}^2/\sigma_j^2 \le d/4$. Using rotational invariance property of Gaussian with (scale of) identity covariance, without loss of generality, we can set $v = v_1 \hat{e}_{j1} + v_{2:d}$ where $\hat{e}_{j1}$ is a unit vector in direction $\mu_j - \mu_1^*$, and $v_{2:d}$ is the rest $d-1$ dimensional orthogonal component. Then, we can rearrange the event as 
    \begin{align*}
        &(v_1 + R_{j1})^2 / {\sigma_1^*}^2 -  v_1^2 / {\sigma_j}^2 \le -\frac{\|v_{2:d}\|^2}{{\sigma_j}^2} x + d \log(1 + x),  \\
        \implies &\left( \frac{1}{{\sigma_1^*}^2} - \frac{1}{\sigma_j^2} \right) \left(v_1 - \left( \frac{1}{{\sigma_1^*}^2} - \frac{1}{\sigma_j^2} \right)^{-1} \frac{R_{j1}}{{\sigma_1^*}^2} \right)^2 - \left( \frac{1}{{\sigma_1^*}^2} - \frac{1}{\sigma_j^2} \right)^{-1} \frac{R_{j1}^2}{{\sigma_1^*}^4} + \frac{R_{j1}^2}{{\sigma_1^*}^2} \\
        &= \frac{x}{\sigma_j^2} \left(v_1 - \frac{\sigma_j^2}{x} \frac{R_{j1}}{{\sigma_1^*}^2} \right)^2 - \frac{R_{j1}^2}{\sigma_j^2 - {\sigma_1^*}^2} \le -\frac{\|v_{2:d}\|^2}{{\sigma_j}^2} x + d \log(1 + x) \\
        \implies &\frac{\|v_{2:d}\|^2}{{\sigma_j}^2} \le d \frac{\log(1 + x)}{x} + \frac{1}{x} \frac{R_{j1}^2}{\sigma_j^2 - {\sigma_1^*}^2} \\
        \implies & \frac{\|v_{2:d}\|^2}{(d-1) {\sigma_j}^2} \le \left(1 + \frac{1}{d-1} \right) \left(\frac{\log 8}{7} + \frac{2}{49} \right) \le \frac{1}{2}.
    \end{align*}
    And, the probability of $P(\|v_{2:d}\|^2 / \sigma_j^2 \le (d-1)/2) \le \exp(-(d-1)/16) \le \pi_{min}^4$ upper-bounds the probability $P_\G (\Eps | X \sim j^{th})$. 
    
    Combining all cases, we can conclude that $P_\G (\Eps) \le \pi_{min}/2$. Hence, we get $\|\G^* - \G\|_{TV} \ge \pi_{min}/4$ as desired when $d \ge 128\log(1/\pi_{min})$. 
    
    \paragraph{Case II. $d \le 128 \log(1/\pi_{min})$:} In this case, we consider the following two events. 
    \begin{align*}
        \Eps_1 &= \{ \|X - \mu_1^*\| \le \sigma_1^* \sqrt{2d} \}, \\
        \Eps_2 &= \{ \|X - \mu_1^*\| \le 3 \sigma_1^* \sqrt{2d} \}.
    \end{align*}
    We will show that either $P_{\G^*}(\Eps_1) - P_{\G} (\Eps_1) \ge \pi_{min} / 4$ or $P_{\G}(\Eps_1^c \cap \Eps_2) - P_{\G^*} (\Eps_1^c \cap \Eps_2) \ge \pi_{min} / 4$. As a first step, we show that $P_{\G^*} (\Eps_1 | X \sim 1^{st}) > 3/4$, and $P_{\G^*} (\Eps_2 | X \sim j^{th}) \ll \pi_{min}/100$ for $j = 2, ..., k$. 
    
    It is easy to see that 
    \begin{align*}
        P_{\G^*} (\Eps_1 | X \sim 1^{st}) = P_{v \sim \mathcal{N}(0, I)} (\|v\|^2 \le 2d) \ge 3/4,
    \end{align*}
    using the pre-computed cdf value of chi-square distribution with degree $d$ (we can numerically check it for small $d$, and we can approximate it with approximation to normal distribution for large $d$). For other components, 
    \begin{align*}
        P_{\G^*} (\Eps_2 | X \sim j^{th}) &= P_{v \sim \mathcal{N}(0, {\sigma_j^*}^2 I)} (\|v + \mu_j^* - \mu_1^*\|^2 \le 18 \sigma_1^* d) \\
        &\le P_{v \sim \mathcal{N}(0, {\sigma_j^*}^2I)} ((\vdot{v}{\hat{e}_{j1}} + R_{j1}^*)^2 \le 18 \sigma_1^*d) \\
        &\le P_{v \sim \mathcal{N}(0, {\sigma_j^*}^2I)} (\vdot{v}{\hat{e}_{j1}} \le 3\sigma_1^* \sqrt{2d} - R_{j1}^*) \\
        &\le P_{v \sim \mathcal{N}(0, {\sigma_j^*}^2I)} (\vdot{v}{\hat{e}_{j1}} \le -R_{j1}^* / 4) \\
        &\le \exp(-{R_{j1}^*}^2 / 32 {\sigma_j^*}^2) \le \pi_{min}^{64},
    \end{align*}
    where $\hat{e}_{j1}$ is a unit vector in direction $\mu_j^* - \mu_1^*$. Note that $3 \sigma_1^* \sqrt{2d} \le 48 \sigma_1^* \sqrt{\log(1/\pi_{min})} \le 3R_{j1}^*/4$. Combining two facts, it is easy to see that $P_{\G^*} (\Eps_1) \ge 3\pi_{min}/4$ and $P_{\G^*} (\Eps_1^c \cap \Eps_2) \le \pi_{min}/4$. 
    
    Now we show that either $P_\G (\Eps_1) \le \pi_{min}/2$ or $P_\G (\Eps_1^c \cap \Eps_2) \ge \pi_{min}/2$ is true. Suppose $P_\G (\Eps_1) \ge \pi_{min}/2$. Observe that $\|\mu_j - \mu_1^*\| / \sigma_1^* \ge 16\sqrt{ \log(1/\pi_{min}) } \ge \sqrt{2d}$ for all $j$. That is, all $\mu_j$ are outside of the sphere that $\Eps_1$ considers. Therefore, if we imagine a bigger ball of radius $3\sqrt{2d}$, for any $j$, there exists a ball of radius $\sqrt{2d}$ in $\Eps_2 \cap \Eps_1^c$ where the contribution from $j^{th}$ component is larger than the ball considered in $\Eps_1$. Since this is true for all $j$, we can conclude that $P_\G (\Eps_1^c \cap \Eps_2) \ge \pi_{min}/2$. 
    
    In conclusion, $\|\G - \G^*\|_{TV} \ge \pi_{min}/4$ if any one of $\mu_i^*$ cannot find a good initailizer in candidate parameters. Note that this result does not assume any separation condition in candidate distributions. Neither, this lower bound for TV distance does not depend on any other parameters but $\pi_{min}$. 
\end{proof}

\subsection{Proof of Theorem \ref{theorem:sample_optimal_learnability}}
\label{Appendix:proof_sample_compression}
\begin{proof}
Lemma \ref{lemma:tv_implies_param} indicates that if $\|\G - \G^*\|_{TV} \le \pi_{min} / 4$, then we have initializers that satisfy the requirement \eqref{eq:minimal_initialization} up to some permutation in $\G$. Note that since the true mixture distribution $\G^*$ satisfies the separation condition \eqref{eq:sample_opt_separation_cond}, when \eqref{eq:tv_to_initializetion} holds true, each $i^{th}$ component has its unique $j$ for initial mean $\mu_j$: if $\|\mu_j - \mu_i^*\|$ is less than $16 \sigma_i^* \sqrt{\log (1/\pi_{min})} < \frac{1}{4} \|\mu_i^* - \mu_{i'}^*\|$ for $i' \neq i$, then $\mu_j$ is at least $\frac{3}{4} \|\mu_{i'}^* - \mu_i^*\| > 48 \sigma_{i'}^* \sqrt{\log (1/\pi_{min})}$ far apart from other $\mu_{i'}^*$. Hence, one $\mu_j$ can only be associated with only one $\mu_i^*$ and vice versa.

We first show how to get the sample-optimal guarantee for the proper-learning of spherical Gaussian mixtures using the sample-compression scheme introduced in \cite{ashtiani2018nearly}. The compressibility of a distribution is (informally) defined as follows:
\begin{definition}[Informal Definition of Compressibility in \cite{ashtiani2018nearly}]
    For any $\epsilon > 0$, a distribution $\F$ is called $(\tau(\epsilon), t(\epsilon), m(\epsilon))$-{\it compressible} if the following holds: if $m(\epsilon)$ samples are drawn from $\F$, we can carefully select $\tau(\epsilon)$ samples (among $m(\epsilon)$ samples) and additional $t(\epsilon)$-bits such that a pre-defined systematic procedure (decoder) takes them as inputs, and returns a distribution $\hat{\F}$ that satisfies $\|\hat{\F} - \F\|_{TV} \le \epsilon$ with high probability. 
\end{definition}
See Definition 3.1 and 3.2 in their work \cite{ashtiani2018nearly} for more details. Their key result (see Theorem 3.5 in \cite{ashtiani2018nearly}) states that if a distribution is $(\tau(\epsilon), t(\epsilon), m(\epsilon))$-compressible, then $\tO(m(\epsilon) + (\tau(\epsilon) + t(\epsilon)) / \epsilon^2)$ samples suffice to learn a $\epsilon$-close distribution in TV distance. The optimal sample upper bound $\tO(kd/\epsilon^2)$ for learning a mixture of axis-aligned Gaussians then follows by (i) showing that a single axis-aligned Gaussian is $(O(d), O(d \log(d/\epsilon)), O(d))$-compressible, and (ii) using their Lemma 3.7 in \cite{ashtiani2018nearly} to conclude that a mixture of $k$ axis-aligned Gaussians is $(O(kd), O(kd \log(d/\epsilon)), \tO(kd/\epsilon))$-compressible.

Given their argument for a mixture of axis-aligned Gaussians, it is straight-forward to get the same result for spherical Gaussians. We only need to show that a single spherical Gaussian is also $(O(d), O(d \log(d/\epsilon)), O(d))$-compressible. Then we can use the same argument using their Lemma 3.7 to conclude that a mixture of $k$ spherical Gaussians is also $(O(kd), O(kd \log(d/\epsilon)), \tO(kd/\epsilon))$-compressible, hence $\tO(kd/\epsilon^2)$ samples suffice to learn a $\epsilon$-close distribution in TV distance.

Suppose a single axis-aligned Gaussian with mean $\mu = (\mu_1, \mu_2, ..., \mu_d) \in \mathbb{R}^d$ and covaraince $\Sigma = diag(\sigma_1, \sigma_2, ..., \sigma_d)$. The decoder they construct for an axis-aligned Gaussian outputs $\hat{\mu}$ and $\hat{\Sigma} = diag(\hat{\sigma}_1, \hat{\sigma}_2, ..., \hat{\sigma}_d)$ such that 
\begin{align*}
  |\mu_i - \hat{\mu}_i| \le \sigma_i \epsilon / d, \quad  | \sigma_i - \hat{\sigma}_i | \le \sigma_i \epsilon / d, \qquad \forall i \in [k],
\end{align*}
which hence guarantees that $\|\mathcal{N}(\mu, \Sigma) - \mathcal{N}(\hat{\mu}, \hat{\Sigma})\|_{TV} \le \epsilon$. For a spherical Gaussian, we can use the same decoder by considering it as an axis-aligned Gaussian, and simply pick $\sigma_1$ as a common scale factor of an identity matrix for a spherical Gaussian. Therefore, a spherical Gaussian is also compressible with the same parameters $(O(d), O(d \log(d/\epsilon)), O(d))$. Combining with Lemma 3.7 and Theorem 3.5 in \cite{ashtiani2018nearly}, we obtain a $\tO(kd/\epsilon^2)$ sample-complexity guarantee for the proper-learning of spherical Gaussian mixtures.

Now we can first get the candidate distribution $\G$ using the algorithm given in \cite{ashtiani2018nearly} with $\tO(kd\pi_{min}^{-2})$ samples to get $\|\G - \G^*\|_{TV} \le \pi_{min}/4$. Then we can run Algorithm \ref{alg:k-means} and then the EM algorithm using $\tO(d\pi_{min}^{-1}/\epsilon^2)$ samples. This gives the algorithm for Theorem \ref{theorem:sample_optimal_learnability}.
\end{proof}

\section{Deferred Proof: Convergence of Population EM when \texorpdfstring{$D_m \le 1/2$}{Lg}.}
\label{appendix:population_em_Dm_le_1over2}
    We define a target error $D_m = \max_j \left(\|\mu_j - \mu_j^*\|/\sigma_j^*, |\pi_j - \pi_j^*| / \pi_j^*, \sqrt{d} |{\sigma_j}^2 - {\sigma_j^*}^2| / {\sigma_j^*}^2 \right) \le 1/2$.
    \begin{proof}
    First of all, we differentiate our EM operator with respect to all variables being estimated ($\sigma_i^2$ are considered as a single variable). For instance,
    \begin{align*}
        \pi_1^+ - \pi_1^* &= \E_\D [w_1] - \E_\D[w_1^*] = \E_\D[\Delta_w^u], \\
        \mu_1^+ - \mu_1^* &= \E_\D [\Delta_w^u (X - \mu_1^*) ] / \E_\D[w_1], \\
        {\sigma_1^+}^2 - {\sigma_1^*}^2 &= \E_\D [\Delta_w^u (\|X - \mu_1^*\|^2 - d{\sigma_1^*}^2)] / (d\pi_1^+) - \|\mu_1^+ - \mu_1^*\|^2/d,
    \end{align*}
    where
    \begin{align*}
        \Delta_w^u &= -w_1^u(1-w_1^u) (X - \mu_1^u)^T (\mu_1 - \mu_1^*) / {\sigma_1^u}^2 
        + \sum_{l \neq 1} w_1^u w_l^u (X - \mu_l^u)^T (\mu_l - \mu_l^*) / {\sigma_l^u}^2 \\
        &\quad -w_1^u(1-w_1^u) (\pi_1 - \pi_1^*) / \pi_1^u
        + \sum_{l \neq 1} w_1^u w_l^u (\pi_l - \pi_l^*) / \pi_l^u \\
        &\quad -w_1^u(1-w_1^u) (+\|X - \mu_1^u\|^2 / (2{\sigma_1^u}^4) - d / (2{\sigma_1^u}^2))  (\sigma_1^2 - {\sigma_1^*}^2) \\
        &\quad + \sum_{l \neq 1} w_1^u w_l^u (+\|X - \mu_l^u\|^2 / (2 {\sigma_l^u}^4) - d / (2{\sigma_l^u}^2)) (\sigma_l^2 - {\sigma_l^*}^2),
    \end{align*}
    where $w_1^u$ is a weight constructed with $\mu_1^u := \mu_1^* + u(\mu_1 - \mu_1^*)$ for some $u \in [0,1]$, and other $u$ scripted variables are defined similarly. In addition to previous technical lemmas, we state a few more helper lemmas. 
    The same result holds for $\E_{\D_j}[\Delta_w]$ as in this corollary. With the above lemma, we can bound the errors for mixing weights from other components. We need one more lemma for bounding the sum of errors.
    \begin{lemma}
        \label{lemma:sum_pnorm_error}
        For values $q \in \{0,1,2,3,4\}$, the following summations are bounded by
        \begin{align}
            \label{ineq:sum_pnorm_error_dv1}
            \sum_{j\neq 1} (\pi_1^* + \pi_j^*) (\R{j1}/\sigma_1^*)^q \exp(-\R{j1}^2 / 128\sigt{1}{j}^2) &\le c_q \pi_1^*,
        \end{align}
        \begin{align}
            \label{ineq:sum_pnorm_error_dvj}
            \sum_{j\neq 1} (\pi_1^* + \pi_j^*) (\R{j1}/\sigma_j^*)^q \exp(-\R{j1}^2 / 128\sigt{1}{j}^2) &\le c_q \pi_1^*,
        \end{align}
        with some small constant $c_q$ given large enough universal constant $C$. 
    \end{lemma}
    The proof is similar to Lemma \ref{lemma:error_bounds_finite_sample}, and we will allow much larger universal constant $C \ge 128$ in the separation condition \eqref{eq:pop_separation_condition} to lighten the algebraic burden. The proof will be given in Appendix \ref{appendix:auxiliary_lemma_for_E}. 
    
    From this point, due to the heavy calculation and algebra, we give up tracking most constants in error bounds. Bounds will be often given in $O(\cdot)$ notation, but we note that the hidden constants will not be too large.
    
    \subsection{Convergence of Mixing Weights}
    Let us start with the simplest targets. Similarly to $D_m \ge 1/2$, we consider errors from other components first. 
    
    \paragraph{Errors from other components $j \neq 1$:} We first bound 
    \begin{align*}
        e_{j1} &= |\E_{\D_j} [w_1^u (1-w_1^u) (X - \mu_1^u)^T (\mu_1 - \mu_1^*)] / {\sigma_1^u}^2 | \\
        &\le 2 \| \E_{\D_j} [w_1^u (1-w_1^u) (X - \mu_1^u)] / \sigma_1^*\| \|(\mu_1 - \mu_1^*) / \sigma_1^*\| \\
        &\le 2D_m \sup_{s \in \mathbb{S}^{d-1}} \E_{\D_j} [w_1^u \vdot{X - \mu_1^u}{s} ]  / \sigma_1^*.
    \end{align*}
    Using the Lemma \ref{lemma:error_bounds_Dm_ge_1over2},
    \begin{align*}
        \E_{\D_j} [|w_1^u \vdot{X - \mu_1^u}{s}|] &\le O \left((1 + \pi_1^*/\pi_j^*) R_{j1}^* \exp(-{R_{j1}^*}^2 / 64\sigt{1}{j}^2) \right), 
    \end{align*}
    which yields $e_{j1} \le D_m O\left((1+\pi_1^*/\pi_j^*) (\R{j1}/{\sigma_1^*}) \exp(-\R{j1}^2/64\sigt{1}{j}^2) \right)$.
    
    The second target is 
    \begin{align*}
        e_{j2} &= | \E_{\D_j} [\sum_{l \neq 1} w_1^u w_l^u (X - \mu_l^u)^T (\mu_l - \mu_l^*)] / {\sigma_l^u}^2] | \\
        &\le 2 \sum_{l\neq1} \|\E_{\D_j} [w_1^u w_l^u (X - \mu_l^u)]\| \|\mu_l - \mu_l^*\| / {\sigma_l^*}^2  \\
        &\le 2 D_m \sum_{l\neq1} \|\E_{\D_j} [w_1^u w_l^u (X - \mu_l^u) /  {\sigma_l^*}]\| \\
        &\le 2D_m \sqrt{\E_{\D_j}[w_1^u]} \sum_{l\neq1} \underbrace{ \sqrt{\sup_{s \in \mathbb{S}^{d-1}} \E_{\D_j} [w_l^u \vdot{X - \mu_l^u}{s}^2  / {\sigma_l^*}^2 ]} }_{I}.
    \end{align*}
    Main challenge is to show that $\sum_{l\neq1} \sqrt{\sup_{s \in \mathbb{S}^{d-1}} \E_{\D_j} [w_l^u \vdot{X - \mu_l^u}{s}^2 ] / {\sigma_l^*}^2} = O(1)$. As this will appear several times, we state a helping lemma.
    \begin{lemma}
        \label{lemma:helper_second_term}
        The summation of term $I$ over $l \neq j$ is bounded by
        \begin{align*}
            \sum_{l\neq j} \sqrt{\sup_{s \in \mathbb{S}^{d-1}} \E_{\D_j} [w_l^u \vdot{X - \mu_l^u}{s}^2 ] / {\sigma_l^*}^2} \le c,
        \end{align*}
        for some small constant $c$.
        When $l = j$, $I = O(1)$.
    \end{lemma}
    The proof of lemma is given in Section \ref{appendix:auxiliary_lemma_for_E}. By the lemma, the summation over the entire term is $O(1)$. The term $\sqrt{\E_{\D_j}[w_1^u]}$ are less than 
    \begin{align*}
        O \left( \sqrt{(1+\pi_1^*/\pi_j^*) \exp(-\R{j1}^2/64\sigt{j}{1}^2)} \right).
    \end{align*}
    Therefore, we can conclude that
    \begin{align*}
        e_{j2} \le D_m O\left( (1 + \pi_1^* / \pi_j^*) \exp(-{R_{j1}^*}^2 / 128 (\sigma_j^* \vee \sigma_1^*)^2 ) \right).
    \end{align*}
    
    The third target is 
    \begin{align*}
        e_{j3} &= |\E_{\D_j} [w_1^u (1-w_1^u) (\pi_1 - \pi_1^*)] / {\pi_1^u} | \\
        &\le 2 D_m | \E_{\D_j} [w_1^u] | \le D_m O \left( (1 + \pi_1^* / \pi_j^*) \exp(-{R_{j1}^*}^2 / 64{\sigma_1^*}^2) \right).
    \end{align*}
    
    The fourth target is
    \begin{align*}
        e_{j4} &= |\E_{\D_j} [\sum_{l\neq1} w_1^u w_l^u (\pi_1 - \pi_1^*)] / {\pi_1^u} | \\
        &\le 2 D_m | \E_{\D_j} [\sum_{l\neq1} w_1^u w_l^u] | \le 2 D_m \E_{\D_j} [w_1^u] \le D_m O \left( (1 +\pi_1^*/\pi_j^*) \exp \left(-{R_{j1}^*}^2 / 64{\sigma_1^*}^2 \right) \right),
    \end{align*}
    since $\sum_{l\neq1} w_l^u \le 1$.
    
    The fifth target is
    \begin{align*}
        e_{j5} &= \left| \E_{\D_j} \left[ w_1^u (1 - w_1^u) (\|X - \mu_1^u\|^2 - d {\sigma_1^u}^2) / (2{\sigma_l^u}^2) \right] (\sigma_1^2 - {\sigma_1^*}^2) / {\sigma_1^u}^2 \right| \\
        &\le \left(2 D_m / \sqrt{d} \right) \left| \E_{\D_j} \left[ w_1^u(1 - w_1^u) (\|X - \mu_1^u\|^2 - d {\sigma_1^u}^2) / (2{\sigma_1^u}^2)  \right] \right| \\
        &\le \left( 2 D_m/ {\sigma_1^*}^2 \sqrt{d} \right) | \E_{\D_j} [ w_1^u (1 - w_1^u) ( \|v\|^2 - d {\sigma_j^*}^2 + 2 \vdot{v}{\mu_j^* - \mu_1^u} \\
        &\qquad \qquad \qquad \qquad \qquad + \| \mu_j^* - \mu_1^u \|^2 + d ({\sigma_j^*}^2 - {\sigma_1^*}^2) + d ({\sigma_1^*}^2 - {\sigma_1^u}^2) ) ] | \\
        &\le \left( 2 D_m/ {\sigma_1^*}^2 \sqrt{d} \right) \E_{\D_j} \Bigg[ \left| w_1^u (\|v\|^2 - d {\sigma_j^*}^2) \right| + 2 \left| w_1^u \vdot{v}{\mu_j^* - \mu_1^u} \right| \\
        &\qquad \qquad \qquad \qquad \qquad + 2 \left| w_1^u {R_{j1}^*}^2 \right| + \left| w_1^u d ({\sigma_j^*}^2 - {\sigma_1^*}^2) \right| + w_1^u d \left( {\sigma_1^*}^2 D_m / \sqrt{d} \right) \Bigg].
    \end{align*}
    We have seen similar terms in $D_m \ge 1/2$ case. Each term we can bound as
    \begin{align*}
        \E_{\D_j} [ |w_1^u(\|v\|^2 - d{\sigma_j^*}^2)| ] &\le \sqrt{\E_{\D_j} [|w_1^u|]} \sqrt{\E_{\D_j} [ (\|v\|^2 - d{\sigma_j^*}^2)^2 ]} \\
        &\le 4 {\sigma_j^*}^2 \sqrt{d (1 + \pi_1^* / \pi_j^*)} \exp(- {R_{j1}^*}^2/128 (\sigma_1^* \vee \sigma_j^*)^2),
    \end{align*}
    \begin{align*}
        \E_{\D_j} [ |w_1^u \vdot{v}{\mu_j^* - \mu_1^u}| ] &\le O\left( (1 + \pi_1^* / \pi_j^*) R_{j1}^* \sigma_j^* \exp(-{R_{j1}^*}^2 / 64(\sigma_1^* \vee \sigma_j^*)^2 ) \right),
    \end{align*}
    \begin{align*}
        d({\sigma_j^*}^2 - {\sigma_1^*}^2) \E_{\D_j} [ |w_1^u| ] &\le O \left( {R_{j1}^*} (\sigma_j^* \vee \sigma_1^*) (1 + \pi_1^* / \pi_j^*) \sqrt{d} \exp(-{R_{j1}^*}^2 / 64(\sigma_1^* \vee \sigma_j^*)^2) \right),
    \end{align*}
    and the rest terms can also be easily bounded. Thus, we have shown that 
    \begin{align*}
        e_{j5} \le D_m O \left(R_{j1}^* (\sigma_j^* \vee \sigma_1^*) / {\sigma_1^*}^2 \cdot (1 + \pi_1^* / \pi_j^*) \exp(-{R_{j1}^*}^2 / 64(\sigma_1^* \vee \sigma_j^*)^2) \right).
    \end{align*}
    
    Finally, we control the last term.
    \begin{align*}
        &e_{j6} = \left| \E_{\D_j} \left[ \sum_{l\neq 1} w_1^u w_l^u (\|X - \mu_l^u\|^2 - d {\sigma_l^u}^2) / (2{\sigma_l^u}^2) \cdot (\sigma_l^2 - {\sigma_l^*}^2) / {\sigma_l^u}^2 \right] \right| \\
        &\le \sum_{l\neq 1} \left( 2 D_m/ ({\sigma_l^*}^2 \sqrt{d}) \right) | \E_{\D_j} [ w_1^u w_l^u ( \|X - \mu_l^u\|^2 - d {\sigma_j^*}^2 + d ({\sigma_j^*}^2 - {\sigma_l^u}^2) ] | \\
        &\le (2D_m / \sqrt{d}) \sum_{l\neq 1} \sqrt{\E_{\D_j} [w_1^u]} \\
        &\times \underbrace{\Bigg( \sqrt{\E_{\D_j} [w_l^u (\|v\|^2 - d {\sigma_j^*}^2)^2 ]} + 2 \sqrt{\E_{\D_j} [w_l^u \vdot{v}{\mu_j^* - \mu_l^u}^2]} + 2 \R{jl}^2 \sqrt{\E_{\D_j} [w_l^u]} + d({\sigma_j^*}^2 - {\sigma_l^u}^2) \sqrt{\E_{\D_j}[w_l^u]} \Bigg) /{\sigma_l^*}^2 }_{II}.
    \end{align*}
    As $II$ will appear frequently, we state a helping lemma.
    \begin{lemma}
        \label{lemma:helper_sixth_term}
        The summation of $II$ over all $l \neq j$ is much less than $\sqrt{d}$. That is,
        \begin{align*}
            \sum_{l\neq j} \left( 1/{\sigma_l^*}^2 \right) &\Bigg( \sqrt{\E_{\D_j} [w_l^u (\|v\|^2 - d {\sigma_j^*}^2)^2 ]} + 2 \sqrt{\E_{\D_j} [w_l^u \vdot{v}{\mu_j^* - \mu_l^u}^2]} \\
            &\qquad + 2 \R{jl}^2 \sqrt{\E_{\D_j} [w_l^u]} + d({\sigma_j^*}^2 - {\sigma_l^u}^2) \sqrt{\E_{\D_j}[w_l^u]} \Bigg) \le c\sqrt{d},
        \end{align*}
        for some small constant $c$. 
        When $l = j$, $II = O(\sqrt{d})$.
    \end{lemma}
    We will also prove this lemma in the last Section \ref{appendix:auxiliary_lemma_for_E}. Thus, the entire summation is $O(\sqrt{d})$. Now we can conclude that 
    \begin{align*}
        e_{j6} &\le (2D_m / \sqrt{d}) \sqrt{\E_{\D_j}[w_1^u]} O\left(\sqrt{d} \right) \le O\left( (1 + \pi_1^* / \pi_j^*) D_m \exp(-\R{j1}^2 / 128 \sigt{j}{l}^2) \right).
    \end{align*}
    
    Collecting all components, we can conclude that 
    \begin{align*}
        \E_{\D_j} [w_1^u] \le D_m O \left( (\R{j1}^2/{\sigma_1^*}^2) (1 + \pi_1^* / \pi_j^*) \exp \left( -{R_{j1}^*}^2 / 128 (\sigma_j^* \vee \sigma_1^*)^2 \right) \right).
    \end{align*}
    The summation over all other components $j \neq 1$ can be bounded thus with Lemma \ref{lemma:sum_pnorm_error}.

    \paragraph{Computing errors from $j = 1$:}
    Reproducing the equation (), we start from
    \begin{align*}
        e_{11} &\le 2D_m \sup_{s \in \mathbb{S}^{d-1}} \E_{\D_1} [w_1^u(1-w_1^u) \vdot{v + \mu_1^* - \mu_1^u}{s}] / \sigma_1^* \\
        &\le 2D_m \sup_{s \in \mathbb{S}^{d-1}} (\E_{\D_1} [(1-w_1^u) \vdot{v}{s}] / \sigma_1^* + \E_{\D_1} [1-w_1^u] D_m).
    \end{align*}
    Observe that $\E_{\D_1}[1 - w_1^u] = \sum_{l\neq1} \E_{\D_1} [w_l^u] \le \sum_{l \neq 1} 8 (\pi_l^* / \pi_1^*) \exp(-{R_{l1}^*}^2 / 64(\sigma_1^* \vee \sigma_l^*)^2)$. This is smaller than $1$ by Lemma \ref{lemma:sum_pnorm_error}. Similarly, we can see that
    \begin{align*}
        \E_{\D_1} [(1-w_1^u) \vdot{v}{s}] &= \sum_{l\neq1} \E_{\D_1} [w_l^u \vdot{v}{s}] \le O\left( \sum_{l\neq 1} \sigma_1^* \sqrt{1 + \pi_l^* / \pi_1^*} \exp \left(-R_{l1}^2 / 128 (\sigma_1^*\vee \sigma_l^*)^2 \right) \right).
    \end{align*}
    This gives that $e_{11} \le D_m \sum_{l\neq1} O\left( (1 + \pi_l^*/\pi_1^*) \exp(-R_{l1}^2/128(\sigma_1^* \vee \sigma_l^*)^2) \right) \ll D_m$. 
    
    The second error term is 
    \begin{align*}
        e_{12} &\le 2D_m \sum_{l\neq 1} \sup_{s \in \mathbb{S}^{d-1}} \E_{\D_1} [w_l^u \vdot{v + \mu_1^* - \mu_l^u}{s}] / {\sigma_l^*} \\
        &\le D_m  \sum_{l\neq 1} O \left( (\R{l1}/\sigma_l^*) (1 + \pi_l^* / \pi_1^*) \exp(-\R{l1}^2/64\sigt{l}{1}^2) \right),
    \end{align*}
    which is guaranteed to be $e_{12} \ll D_m$. 
    
    The third and fourth terms are smaller than $2 D_m \sum_{l\neq 1} \E_{\D_1} [w_l^u] \ll D_m$. Now we need to deal with fifth and sixth terms, which again require some algebraic manipulation. We can start from ()...
    \begin{align*}
        e_{15} \le (2D_m/{\sigma_1^*}^2 \sqrt{d}) \E_{\D_1} \left[(1-w_1^u) \left( |\|v\|^2 - d{\sigma_1^*}^2| + 2|\vdot{v}{\mu_1^* - \mu_1^u}| + D_m^2 + d({\sigma_1^*}^2 - {\sigma_1^u}^2) \right) \right].
    \end{align*}
    For each item, we can say that
    \begin{align*}
        \E_{\D_1} [(1-w_1^u) (\|v\|^2 - d{\sigma_1^*}^2)] &\le \sqrt{\E_{\D_1} [1 - w_1^u]} \sqrt{\E_{\D_1} [(\|v\|^2 - d{\sigma_1^*}^2)^2]} \\
        &\le O \left( \sqrt{\sum_{l\neq1} (1 + \pi_l^*/\pi_1^*) \exp(-{R_{l1}^*}^2 / 64(\sigma_1^* \vee \sigma_l^*)^2 )} \sqrt{2 d {\sigma_1^*}^4} \right),
    \end{align*}
    and for two other terms,
    \begin{align*}
        \E_{\D_1} [(1-w_1^u) \vdot{v}{\mu_1^* - \mu_1^u}] &\le D_m \sigma_1^* \sqrt{\E_{\D_1} [1 - w_1^u]} \sqrt{\E_{\D_1} [\vdot{v}{s}^2]} \\
        &\le D_m {\sigma_1^*}^2 O \left( \sqrt{\sum_{l\neq1} (1 + \pi_l^*/\pi_1^*) \exp(-{R_{l1}^*}^2 / 64(\sigma_1^* \vee \sigma_l^*)^2 )} \right),
    \end{align*}
    and 
    \begin{align*}
        \E_{\D_1} [(1-w_1^u) d({\sigma_1^*}^2 - {\sigma_1^u}^2)] &\le \left( {\sigma_1^*}^2 D_m \sqrt{d} \right) O \left( \sum_{l\neq1} (1 + \pi_l^*/\pi_1^*) \exp(-{R_{l1}^*}^2 / 64(\sigma_1^* \vee \sigma_l^*)^2 ) \right).
    \end{align*}
    Thus $e_{15} \le 2 D_m / ({\sigma_1^*}^2 \sqrt{d}) \cdot c \sqrt{d} {\sigma_1^*}^2 \le c' D_m$ for small constants $c, c'$. 
    
    Finally, the sixth term can be similarly bounded as
    \begin{align*}
        e_{16} &\le \sum_{l\neq1} (2D_m / {\sigma_l^*}^2 \sqrt{d}) |\E_{\D_1} [ w_1^u w_l^u (\|X - \mu_l^u\|^2 - d {\sigma_1^u}^2)]| \\
        &\le \sum_{l\neq1} (2D_m / {\sigma_l^*}^2 \sqrt{d}) \left| \E_{\D_1} \left[ w_l^u \left( (\|v\|^2 - d {\sigma_1^*}^2) + 2\vdot{v}{\mu_1^* - \mu_l^u} + \|\mu_1^* - \mu_l^u\|^2 + d ({\sigma_1^*}^2 - {\sigma_1^u}^2) \right) \right] \right| \\
        &\le D_m \sum_{l\neq1} (2 / {\sigma_l^*}^2 \sqrt{d}) \Bigg( |\E_{\D_1} [ w_l^u (\|v\|^2 - d {\sigma_1^*}^2)]| \\
        & \qquad \qquad \qquad \qquad \qquad + 2D_m\sigma_1^* |\E_{\D_1} [w_l^u \vdot{v}{s}]| + \E_{\D_1} [w_l^u] D_m^2 {\sigma_1^*}^2 + \sqrt{d} D_m {\sigma_1^*}^2 \E_{\D_1} [w_l^u] \Bigg) \\
        &\le D_m \sum_{l\neq1} (2 / {\sigma_l^*}^2 \sqrt{d}) (5 + 3 \pi_l^*/\pi_1^*) \exp(-\R{l1}^2 / 128 \sigt{l}{1}^2) \left( \sqrt{2d} {\sigma_1^*}^2 + 2D_m {\sigma_1^*}^2 + D_m^2 {\sigma_1^*}^2 + \sqrt{d} D_m {\sigma_1^*}^2 \right) \\
        &\le 20 D_m /\pi_1^* \sum_{l\neq1} ({\sigma_1^*}^2 / {\sigma_l^*}^2) (\pi_1^* + \pi_l^*) \exp(-\R{l1}^2 / 128 \sigt{l}{1}^2) \le c D_m 
    \end{align*}
    for small constant $c$. Collecting all components, we can conclude that $\E_{\D_1} [\Delta_w^u] \le c' D_m$. 
    
    \paragraph{Errors from all components:} Collecting the errors from other components and own components, now we can conclude that $\E_{\D} [\Delta_w^u] = \sum_j \pi_j^* \E_{\D_j} [\Delta_w^u] \le c_\mu \pi_1^* D_m$ for some small $c_\mu < 1$. 
    
    \subsection{Convergence of Means}
    
    \paragraph{Computing errors from $j \neq 1$:} Let us proceed in a very similar way we did for mixing weights. Let us first handle $-w_1^u(1-w_1^u)(X - \mu_1^u)^T (\mu_1 - \mu_1^*) / \sigma_1^2$.
    \begin{align*}
        e_{j1} &= \|\E_{\D_j} [w_1^u (1-w_1^u) (X - \mu_1^u)^T (\mu_1 - \mu_1^*) (X - \mu_1^*)] / {\sigma_1^u}^2\| \\
        &\le 2 \|\E_{\D_j} [w_1^u (1-w_1^u) (X - \mu_1^*) (X - \mu_1^u)^T] / \sigma_1^*\|_{op} \|(\mu_1 - \mu_1^*) / \sigma_1^*\| \\
        &\le 2D_m \sup_{s \in \mathbb{S}^{d-1}} \E_{\D_j} [w_1^u (1-w_1^u) \vdot{X - \mu_1^*}{s} \vdot{X - \mu_1^u}{s}]  / \sigma_1^*.
    \end{align*}
    Therefore, it is enough to show that for any fixed unit vector $s$, 
    \begin{align*}
        |\E_{\D_j} [w_1^u (1-w_1^u) \vdot{X - \mu_1^*}{s} \vdot{X - \mu_1^u}{s}],
    \end{align*}
    is exponentially small. We state one more helper lemma that bounds
    \begin{lemma}
        \label{lemma:w1_vsl2_small_jneq1}
        For $j \neq 1$, 
        \begin{align}
            \label{eq:w1_vsl2_small_jneq1}
            \E_{\D_j} [w_1^u \vdot{v}{s}^2] \le 5 (1 + \pi_1^* / \pi_j^*) \R{j1}^2 \exp \left( -\R{j1}^2 / 64\sigt{j}{1}^2 \right).
        \end{align}
    \end{lemma}
    \begin{proof}
        We can reproduce the proof of Corollary \ref{corollary:error_bound_Dm_1over2}. Let $\beta = \R{j1}^2/64\sigt{j}{1}^2$. Then,
        \begin{align*}
            |\E_{\D_j} [w_1\vdot{v}{s}^2] | &= \left| \E_{\D_j} \left[ w_1 \vdot{v}{s}^2 \indic_{\Eps_{j,good}} \right] \right| + \left| \E_{\D_j} \left[ w_1 \vdot{v}{s}^2 \indic_{\Eps_{j,good}^c} \right] \right| \\
            &\le 3(\pi_1^*/\pi_j^*) \exp(-\beta) \E_{\D_j} \left[ |\vdot{v}{s}|^2 \right] + \E_{\D_j} \left[ |\vdot{v}{s}|^2 | \Eps_{j,1}^c \right] P(\Eps_{j,1}^c) \\
            &\qquad + \E_{\D_j} \left[ |\vdot{v}{s}|^2 | \Eps_{j,2}^c \right] P(\Eps_{j,2}^c) + \E_{\D_j} \left[ |\vdot{v}{s}|^2 | \Eps_{j,3}^c \right] P(\Eps_{j,3}^c).
        \end{align*}
        $\E_{\D_j} \left[ |\vdot{v}{s}| | \Eps_{j,1}^c \right]$ can be bounded with Lemma \ref{lemma:vp_conditioned_vu}, with $p = 2$ and $\alpha = R_{j1}^*/5\sigma_j^*$.
        \begin{align*}
            \E_{\D_j} \left[ |\vdot{v}{s}|^2 | \vdot{v}{R_{j1}^*} \ge {R_{j1}^*}^2/5 \right] &\le \sigma_j^* \E_{v \sim \mathcal{N}(0, I_d)} \left[ |\vdot{v}{s}|^2 | |\vdot{v}{u}| \ge \alpha \right] \le {\sigma_j^*}^2 \left( 4\alpha^{2} + 4 \right) \le {R_{j1}^*}^2.
        \end{align*}
        Similarly, we can bound $\E_{\D_j} [|\vdot{v}{s}|^2 | \Eps_{j,2}^c] P(\Eps_{j,2}^c) \le 2R_{j1}^*$ using the same Lemma \ref{lemma:vp_conditioned_vu} with $p = 2$ and $\alpha = R_{j1}^*/4 \sigma_j^*$. For the third term, we use Lemma \ref{lemma:vp_conditioned_vl2}, with $p = 2$ and $\alpha = {R_{j1}^*}^2/64(\sigma_1^* \vee \sigma_j^*)^2 = \beta$. Then,
        \begin{align*}
            \sigma_j^* \E_{v \sim \mathcal{N}(0,I_d)} [|\vdot{v}{s}|^2 | \|v\|^2 \ge d + 2\sqrt{\alpha d} + 2\alpha] &\le {\sigma_j^*}^2 \left( (64\alpha) + 4 \exp(-\alpha/2) (8\alpha + 2) \right) \le 2 {R_{j1}^*}^2, \\
            {\sigma_j^*}^2 \E_{v \sim \mathcal{N}(0,I_d)} [|\vdot{v}{s}|^2 | \|v\|^2 \le d - 2\sqrt{\alpha d}] &\le {\sigma_j^*}^2 \E_{v \sim \mathcal{N}(0,I_d)} [|\vdot{v}{s}|^2] \le {\sigma_j^*}^2,
        \end{align*}
        Collecting these three components, we can conclude that 
        \begin{align*}
            |\E_{\D_j} [w_1 \vdot{v}{s}^2]| \le (3(\pi_1^*/\pi_j^*){\sigma_j^*}^2 + 5 {R_{j1}^*}^2) \exp(-\beta).
        \end{align*}
        This yields the equation \eqref{eq:w1_vsl2_small_jneq1}.
    \end{proof}
    
    Then we can proceed as
    \begin{align*}
        |\E_{\D_j} [w_1^u (1-w_1^u) \vdot{X - \mu_1^*}{s} &\vdot{X - \mu_1^u}{s}]| \le \E_{\D_j} \left[w_1^u \left( \vdot{v}{s}^2 + 2 |\vdot{v}{s}| \R{j1} + 2\R{j1}^2 \right) \right] \\
        &\le O\left( (1 + \pi_1^* / \pi_j^*) \R{j1}^2 \exp(-\R{j1}^2/64 \sigt{j}{1}^2 ) \right),
    \end{align*}
    which yields $e_{j1} \le D_m \sigma_1^* O\left((1 + \pi_1^*/\pi_j^*) (\R{j1}/\sigma_1^*)^2 \exp(-\R{j1}^2/64\sigt{j}{1}^2) \right)$.
    
    Similarly, we bound the second term.
    \begin{align*}
        e_{j2} &\le 2 D_m \left|\E_{\D_j} \left[\sum_{l\neq1} w_1^u w_l^u \vdot{X - \mu_1^*}{s} \vdot{X - \mu_l^u}{s} / {\sigma_l^*} \right] \right| \\
        &\le 2D_m \sqrt{\E_{\D_j} \left[ {w_1^u} \vdot{X-\mu_1^*}{s}^2 \right]} \sum_{l\neq1} \sqrt{ \E_{\D_j} [{w_l^u} \vdot{X - \mu_l^u}{s}^2 ] / {\sigma_l^*}^2 }.
    \end{align*}
    Using the Lemma \ref{lemma:helper_second_term} as in mixing weights, we get
    \begin{align*}
        e_{j2} \le D_m \sigma_1^* O\left( (1 + \pi_1^*/\pi_j^*) (\R{j1}/\sigma_1^*) \exp(-\R{j1}^2/128\sigt{j}{1}^2) \right).
    \end{align*}
    
    Third term and fourth term are straight-forward to bound.
    \begin{align*}
        e_{j3} &\le 2 D_m \left|\E_{\D_j} \left[w_1^u \vdot{X - \mu_1^*}{s}\right] \right| \le D_m O \left(R_{j1}^* \exp(-\R{j1}^2/64\sigt{j}{1}^2) \right),
    \end{align*}
    \begin{align*}
        e_{j4} &\le 2 D_m \left|\E_{\D_j} \left[\sum_{l\neq1} w_1^u w_l^u \vdot{X - \mu_1^*}{s}\right] \right| \le 2 D_m \left|\E_{\D_j} \left[w_1^u \vdot{X - \mu_1^*}{s}\right] \right|,
    \end{align*}
    which is again smaller than $D_m O \left(R_{j1}^* \exp(-\R{j1}^2/64\sigt{j}{1}^2) \right)$.
    
    The challenging fifth and sixth term is also bounded using similar algebra.
    \begin{align*}
        e_{j5} &= \left| \E_{\D_j} \left[ w_1^u (1 - w_1^u) \vdot{X-\mu_1^*}{s} (\|X - \mu_1^u\|^2 - d {\sigma_1^u}^2) / (2{\sigma_1^u}^2) \right] (\sigma_1^2 - {\sigma_1^*}^2) / {\sigma_1^u}^2 \right| \\
        &\le \left(2 D_m / \sqrt{d} \right) \left| \E_{\D_j} \left[ w_1^u \vdot{X - \mu_1^*}{s} (\|X - \mu_1^u\|^2 - d {\sigma_1^u}^2) / (2{\sigma_1^u}^2)  \right] \right| \\
        &\le \left( 2 D_m/ {\sigma_1^*}^2 \sqrt{d} \right) \Bigg| \sqrt{\E_{\D_j} [ w_1^u \vdot{X - \mu_1^*}{s}^2 ]} \left(\sqrt{\E_{\D_j} [(\|v\|^2 - d {\sigma_j^*})^2]} + 2 \sqrt{\E_{\D_j} [\vdot{v}{\mu_j^* - \mu_1^u}^2]} \right) \\
        &\qquad \qquad \qquad \qquad \qquad + 2\R{j1}^2 |\E_{\D_j} [w_1^u \vdot{X - \mu_1^*}{s}]| +  d ({\sigma_j^*}^2 - {\sigma_1^u}^2) \sqrt{\E_{\D_j}[w_1^u]} \sqrt{\E_{\D_j} [\vdot{X - \mu_1^*}{s}^2]} \Bigg| \\
        &\le \left( 2 D_m/ {\sigma_1^*}^2 \sqrt{d} \right) (1+\pi_1^*/\pi_j^*) O\Bigg( \sqrt{\R{j1}^2 \exp(-\R{j1}^2/64\sigt{j}{1}^2)} \left( \sqrt{2d} + 4 \R{j1}\sigma_j^* \right) \\
        &\qquad \qquad \qquad \qquad \qquad + 2\R{j1}^3 \exp(-\R{j1}^2/64\sigt{j}{1}^2) + \sqrt{d} \R{j1}^2 (\sigma_j^* \vee \sigma_1^*) \exp(-\R{j1}^2/128\sigt{j}{1}^2) \Bigg) \\ 
        &\le D_m \sigma_1^* O \left((1+\pi_1^*/\pi_j^*) (\R{j1}/\sigma_1^*)^3 \exp(-\R{j1}^2/128\sigt{j}{1}^2) \right).
    \end{align*}
    Finally, 
    \begin{align*}
        e_{j6} &= \left| \E_{\D_j} \left[ \sum_{l\neq1} w_1^u w_l^u \vdot{X-\mu_1^*}{s} (\|X - \mu_l^u\|^2 - d {\sigma_l^u}^2) / (2{\sigma_l^u}^2) \right] (\sigma_l^2 - {\sigma_l^*}^2) / {\sigma_l^u}^2 \right| \\
        &\le \left( 2 D_m/ \sqrt{d} \right) \sum_{l\neq 1} (1/{\sigma_l^*}^2) \sqrt{\E_{\D_j} [ w_1^u \vdot{X - \mu_1^*}{s}^2 ]} \\
        &\qquad \times \Bigg(\sqrt{\E_{\D_j} [w_l^u (\|v\|^2 - d {\sigma_j^*})^2]} + 2 \sqrt{\E_{\D_j} [w_l^u \vdot{v}{\mu_j^* - \mu_l^u}^2} + 2\R{jl}^2 \sqrt{\E_{\D_j}[w_l^u] } +  d ({\sigma_j^*}^2 - {\sigma_l^u}^2) \sqrt{\E_{\D_j}[w_l^u]} \Bigg).
    \end{align*}
    Using Lemma \ref{lemma:helper_sixth_term}, the summation is less than $c\sqrt{d}$ and hence
    \begin{align*}
        e_{j6} \le D_m \sigma_1^* O\left( (1 + \pi_1^*/\pi_j^*) (R_{j1}^*/\sigma_1^*) \exp(-\R{j1}^2/128\sigt{j}{1}^2) \right).
    \end{align*}
    So all guarantee that 
    \begin{align}
        \label{eq:mean_error_sum_jneq1_Dm_le_12}
        \sum_{j\neq1} \pi_j^* \| \E_{\D_j} [\Delta_w^u (X - \mu_1^*)] \| &\le D_m \sigma_1^* \sum_{j\neq 1}  O\left( (\pi_1^* + \pi_j^*) (R_{j1} /{\sigma_1^*})^3 \exp(-\R{j1}^2/128\sigt{j}{1}^2) \right).
    \end{align}
    
    \paragraph{Errors from $j = 1$:} We repeat the process of bounding six terms as always.
     \begin{align*}
        e_{11} &\le 2 D_m \sup_{s \in \mathbb{S}^{d-1}} \E_{\D_1} [w_1^u (1-w_1^u) \vdot{X - \mu_1^*}{s} \vdot{X - \mu_1^u}{s}]  / \sigma_1^* \\
        &\le 2 D_m \sup_{s \in \mathbb{S}^{d-1}} \E_{\D_1} [(1-w_1^u) (\vdot{v}{s}^2 + \vdot{v}{s} D_m \sigma_1^*)]  / \sigma_1^*.
    \end{align*}
    Then,
    \begin{align*}
        \E_{\D_1} [(1-w_1^u) \vdot{v}{s}^2] &= \sum_{l\neq 1} \E_{\D_1} [w_l^u \vdot{v}{s}^2] \le \sum_{l\neq 1} O\left((1 + \pi_l^*/\pi_1^*) \R{l1}^2 \right) \exp(-\R{l1}^2/64\sigt{l}{1}^2) \\
        &\le (1/\pi_1^*) \sum_{l\neq 1} (\pi_1^* + \pi_l^*) \R{l1}^2 \exp(-\R{l1}^2/64\sigt{l}{1}^2)
        \le c {\sigma_1^*}^2,
    \end{align*}
    for some small constant $c$ with the Lemma \ref{lemma:sum_pnorm_error}.
    
    The second term will be similarly,
    \begin{align*}
        e_{12} &\le 2 D_m \E_{\D_1}  \left[ \sum_{l\neq 1} |w_1^u w_l^u \vdot{v}{s} \vdot{X - \mu_l^u}{s}| / \sigma_l^* \right]  \\
        &\le 2 D_m {\sigma_1^*} \sum_{l\neq1} \sqrt{\E_{\D_1} \left[w_l^u \vdot{X - \mu_l^u}{s}^2 \right] / {\sigma_l^*}^2 } \le c D_m \sigma_1^*, 
    \end{align*}
    for some small constant $c$. 
    
    Third and fourth term is easy to handle,
    \begin{align*}
        e_{13}, e_{14} &\le D_m \sum_{l\neq1} O \left(R_{l1}^* \exp(-\R{l1}^2/64\sigt{l}{1}^2) \right),
    \end{align*}
    which is again much less than $D_m$. 
    
    For the fifth term, we again start from
    \begin{align*}
        e_{15} &= \left| \E_{\D_1} \left[ w_1^u (1 - w_1^u) \vdot{X-\mu_1^*}{s} (\|X - \mu_1^u\|^2 - d {\sigma_1^u}^2) / (2{\sigma_1^u}^2) \right] (\sigma_1^2 - {\sigma_1^*}^2) / {\sigma_1^u}^2 \right| \\
        &\le \left(2 D_m / \sqrt{d} \right) \left| \E_{\D_1} \left[ (1 - w_1^u) \vdot{v}{s} (\|v + \mu_1^* - \mu_1^u\|^2 - d {\sigma_1^u}^2) / (2{\sigma_1^u}^2)  \right] \right| \\
        &\le \left( 2 D_m/ {\sigma_1^*}^2 \sqrt{d} \right) \Bigg| \sqrt{\E_{\D_1} [ (1 - w_1^u) \vdot{v}{s}^2 ]} \left(\sqrt{\E_{\D_1} [(\|v\|^2 - d {\sigma_1^*}^2)^2]} + 2 \sqrt{\E_{\D_1} [\vdot{v}{\mu_1^* - \mu_1^u}^2]} \right) \\
        &\qquad \qquad \qquad \qquad \qquad + 2D_m^2 {\sigma_1^*}^2 |\E_{\D_1} [(1 - w_1^u) \vdot{v}{s}]| +  D_m \sqrt{d} {\sigma_1^*}^2 |\E_{\D_1} [(1 - w_1^u) \vdot{v}{s}]| \Bigg|,
    \end{align*}
    The first term in the above can be bounded with Lemma \ref{lemma:w1_vsl2_small_jneq1},
    \begin{align*}
        \E_{\D_1} [(1 - w_1^u) \vdot{v}{s}^2] = \sum_{l\neq1} \E_{\D_1} [w_l^u \vdot{v}{s}^2] &\le {\sigma_1^*}^2 O \left( \sum_{l\neq1} (1 + \pi_l^* / \pi_1^*) \R{l1}^2 \exp(-\R{l1}^2/64\sigt{l}{1}^2) \right) \\
        &\le c {\sigma_1^*}^2.
    \end{align*}
    and 
    \begin{align*}
        &\E_{\D_1} [(\|v\|^2 - d{\sigma_1^*}^2)^2] = 2d {\sigma_1^*}^4, \ \E_{\D_1} [\vdot{v}{\mu_1^* - \mu_1^u}^2] \le D_m^2 {\sigma_1^*}^4.
    \end{align*}
    Similarly, we have that $\E_{\D_1} [(1 - w_1^u) \vdot{v}{s}] \le c\sigma_1^*$. Collecting all components, we can bound this fifth term $e_{15} \ll D_m \sigma_1^*$. 
    
    The bound for the final term follows similarly.
    \begin{align*}
        e_{16} &= \left| \E_{\D_1} \left[ \sum_{l\neq1} w_1^u w_l^u \vdot{X-\mu_1^*}{s} (\|X - \mu_l^u\|^2 - d {\sigma_l^u}^2) / (2{\sigma_l^u}^2) \right] (\sigma_l^2 - {\sigma_l^*}^2) / {\sigma_l^u}^2 \right| \\
        &\le \left( 2 D_m/ \sqrt{d} \right) \sum_{l\neq 1} (1/{\sigma_l^*}^2) \sqrt{\E_{\D_1} [w_l^u \vdot{v}{s}^2 ]} \\
        &\ \times \Bigg(\sqrt{\E_{\D_1} [w_l^u (\|v\|^2 - d {\sigma_j^*})^2]} + 2 \sqrt{\E_{\D_1} [w_l^u \vdot{v}{\mu_1^* - \mu_l^u}^2} + 2\R{jl}^2 \sqrt{\E_{\D_1}[w_l^u] } +  d ({\sigma_1^*}^2 - {\sigma_l^u}^2) \sqrt{\E_{\D_1}[w_l^u]} \Bigg).
    \end{align*}
    We can again use Lemma \ref{lemma:helper_sixth_term} and get $e_{j6} \le c D_m {\sigma_1^*}$ for some small constant $c$.
    
    Finally, we collect all error terms to conclude that $\E_{\D_1} [\Delta_w^u (X - \mu_1^*)] \le c D_m \sigma_1^*$ for some small constant $c$. Now with the equation \eqref{eq:mean_error_sum_jneq1_Dm_le_12}, 
    \begin{align*}
        \| \mu_1^+ - \mu_1^* \| &= (\sum_j \pi_j^* \E_{\D_j} [\Delta_w^u (X - \mu_1^*)]) / \pi_1^+ \\
        &\le D_m \sigma_1^* O\left( \sum_{j\neq1} (\pi_1^* + \pi_j^*) (\R{j1}/\sigma_1^*)^3 \exp(-\R{j1}^2/128\sigt{j}{1}^2) \right) / \pi_1^+ \\
        &\le c D_m \sigma_1^* \pi_1^* / \pi_1^+ \le c' D_m \sigma_1^*,
    \end{align*}
    for some small constant $c'$, where we used Lemma \ref{lemma:sum_pnorm_error} with $q = 3$.
    
    \subsection{Convergence of the Variance}
    The most challenging part is again to show the convergence of variance estimators. We start with each term by term as other quantities.
    
    \paragraph{Errors from other components $j \neq 1$:} 
    Let us start from the first error term.
    \begin{align*}
        e_{j1} &= |\E_{\D_j} [w_1^u (1-w_1^u) (X - \mu_1^u)^T (\mu_1 - \mu_1^*) (\|X - \mu_1^*\|^2 - d {\sigma_1^*}^2) ] / {\sigma_1^u}^2 | \\
        &\le 2 \left\|\E_{\D_j} [w_1^u (\|X - \mu_1^*\|^2 - d {\sigma_1^*}^2) (X - \mu_1^*)] / \sigma_1^* \right\| \|(\mu_1 - \mu_1^*) / \sigma_1^*\| \\
        &\le 2D_m \sup_{s \in \mathbb{S}^{d-1}} \E_{\D_j} [w_1^u(\|X - \mu_1^*\|^2 - d {\sigma_1^*}^2) \vdot{X - \mu_1^u}{s}]  / \sigma_1^* \\
        &\le 2D_m \sup_{s \in \mathbb{S}^{d-1}} \E_{\D_j} \left[ w_1^u \left( (\|v\|^2 - d{\sigma_j^*}^2) +  2\vdot{v}{\mu_j^* - \mu_1^*} + \|\mu_j^* - \mu_1^*\|^2 + d ({\sigma_j^*}^2 - {\sigma_1^*}^2) \right) \vdot{X - \mu_1^u}{s} \right]  / \sigma_1^*.
    \end{align*}
    For each item, 
    \begin{align*}
        \E_{\D_j} [w_1^u \vdot{X - \mu_1^u}{s} (\|v\|^2 - d{\sigma_j^*}^2)] &\le \sqrt{\E_{\D_j} [w_1^u \vdot{X - \mu_1^u}{s}^2]} \sqrt{\E_{\D_j} [ (\|v\|^2 - d{\sigma_j^*}^2)^2 ] } \\
        &\le O \left( \sqrt{(1 + \pi_1^*/\pi^*) \R{j1}^2 \exp(-\R{j1}^2/64\sigt{j}{1}^2)} {\sigma_j^*}^2 \sqrt{2d} \right) \\
        &\le O\left( \sqrt{d} {\sigma_j^*}^2 \R{j1} (1+\pi_1^*/\pi_j^*) \exp(-\R{j1}^2 / 128\sigt{j}{1}^2) \right).
    \end{align*}
    \begin{align*}
        2 \R{j1} \E_{\D_j} [w_1^u \vdot{X - \mu_1^u}{s} \vdot{v}{s'}] &\le 2\R{j1} \sqrt{\E_{\D_j} [w_1^u \vdot{X - \mu_1^u}{s}^2]} \sqrt{\E_{\D_j} [ w_1^u \vdot{v}{s'}^2 ] } \\
        &\le O \left( (1 + \pi_1^*/\pi_j^*) \R{j1}^2 \sigma_j^* \exp(-\R{j1}^2/64\sigt{j}{1}^2) \right),
    \end{align*}
    and the third item is
    \begin{align*}
        \R{j1}^2 \E_{\D_j} [w_1^u \vdot{X - \mu_1^u}{s}] &\le O \left( (1 + \pi_1^*/\pi^*) \R{j1}^3 \exp(-\R{j1}^2/64\sigt{j}{1}^2) \right).
    \end{align*}
    The final item, we again need Lemma \ref{lemma:weight_bound_for_variance} to get
    \begin{align*}
        d({\sigma_j^*}^2 - {\sigma_1^*}^2) \E_{\D_j} [w_1^u \vdot{X - \mu_1^u}{s}] &\le d({\sigma_j^*}^2 - {\sigma_1^*}^2) \sqrt{\E_{\D_j} [w_1^u]} \sqrt{\E_{\D_j} [w_1^u \vdot{X - \mu_1^u}{s}^2]} \\
        &\le O \left( (1 + \pi_1^*/\pi_j^*) \R{j1}^2 (\sigma_j^* \vee \sigma_1^*) \sqrt{d} \exp(-\R{j1}^2/64\sigt{j}{1}^2) \right).
    \end{align*}
    Hence, the first error term is bounded as
    \begin{align*}
        e_{j1} \le {\sigma_1^*}^2 O\left( (1+\pi_1^*/\pi_j^*) \sqrt{d} (\R{j1}/{\sigma_1^*})^3 \exp(-\R{j1}^2/128\sigt{j}{1}^2) \right). 
    \end{align*}
    
    The error bound for second error term starts with arranging equations as usual,
    \begin{align*}
        e_{j2} &= \sum_{l\neq1} |\E_{\D_j} [w_1^u w_l^u (X - \mu_l^u)^T (\mu_l - \mu_l^*) (\|X - \mu_1^*\|^2 - d {\sigma_1^*}^2) ] / {\sigma_l^u}^2 | \\
        &\le 2D_m \sum_{l\neq1} \sup_{s \in \mathbb{S}^{d-1}} \E_{\D_j} \left[ w_1^u w_l^u (\|X - \mu_1^*\|^2 - d {\sigma_1^*}^2) \vdot{X - \mu_l^u}{s} \right]  / \sigma_l^* \\
        &\le 2D_m \sqrt{\E_{\D_j} \left[ w_1^u \left( \|X - \mu_1^*\|^2 - d{\sigma_1^*}^2 \right)^2 \right]} \sum_{l\neq1}  \sqrt{\E_{\D_j} \left[ w_l^u \vdot{X - \mu_l^u}{s}^2 / {\sigma_l^*}^2 \right]} .
    \end{align*}
    For the first square root, useful inequality is $(a+b+c+d)^2 \le 4(a^2 + b^2 + c^2 + d^2)$ for any real $a,b,c,d$. Using this, 
    \begin{align}
        \label{ineq:bound_ej2_variance_jneq1}
        \E_{\D_j} \Bigg[ w_1^u &\left( \|X - \mu_1^*\|^2 - d{\sigma_1^*}^2 \right)^2 \Bigg] \nonumber \\
        &\le 4 \E_{\D_j} \left[ w_1^u \left( (\|v\|^2 - d{\sigma_j^*}^2)^2 + 4\vdot{v}{\mu_j^* - \mu_1^*}^2 + \|\mu_j^* - \mu_1^*\|^4 + d^2 ({\sigma_j^*}^2 - {\sigma_1^*}^2)^2 \right) \right] \nonumber \\
        &\le 4 \sqrt{\E_{\D_j}[w_1^u]}\sqrt{\E_{\D_j} [(\|v\|^2 - d{\sigma_j^*}^2)^4} \nonumber \\ 
        &\qquad \qquad \qquad + 16 \R{j1}^2 \E_{\D_j} [w_1^u \vdot{v}{s'}^2] + 4\R{j1}^4 \E_{\D_j} [w_1^u] + 4 d^2 ({\sigma_j^*}^2 - {\sigma_1^*}^2)^2 \E_{\D_j} [w_1^u] \nonumber \\
        &\le (1 + \pi_1^*/\pi_j^*) \left( O\left(d {\sigma_j^*}^4 \right) \exp(-\R{j1}^2/128\sigt{j}{1}^2) + O \left( d \R{j1}^4 \right) \exp(-\R{j1}^2/64\sigt{j}{1}^2) \right), 
    \end{align}
    where we used Lemma \ref{lemma:weight_bound_for_variance}. Meanwhile, the summation over $l \neq 1$ in right hand side is $O(1)$ by Lemma \ref{lemma:helper_second_term}. Therefore, 
    \begin{align*}
        e_{j2} \le {\sigma_1^*}^2 O\left( \sqrt{d} (1 + \pi_1^*/\pi_j^*) (\R{j1}/{\sigma_1^*})^2 \exp(- \R{j1}^2/128 \sigt{j}{1}^2) \right),
    \end{align*}
    
    As we can imagine, $e_{j3}$ and $e_{j4}$ can be shown to be bounded using the same procedure for $e_{j1}$ and $e_{j2}$. 
    
    Now we jump to $e_{j5}$ and $e_{j6}$. 
    \begin{align*}
        e_{j5} &= \left| \E_{\D_j} \left[ w_1^u (1 - w_1^u) (\|X - \mu_1^*\|^2 - d {\sigma_1^*}^2) (\|X - \mu_1^u\|^2 - d {\sigma_1^u}^2) / (2{\sigma_1^u}^2) \right] (\sigma_1^2 - {\sigma_1^*}^2) / {\sigma_1^u}^2 \right| \\
        &\le \left(2 D_m / {\sigma_1^*}^2 \sqrt{d} \right) \sqrt{\E_{\D_j} \left[ w_1^u (\|X - \mu_1^*\|^2 - d {\sigma_1^*}^2)^2 \right]} \sqrt{\E_{\D_j} \left[ w_1^u (\|X - \mu_1^u\|^2 - d {\sigma_1^u}^2)^2 \right]}.
    \end{align*}
    The rest of the procedure is repetition of \eqref{ineq:bound_ej2_variance_jneq1} (we get a same order-wise bound for both square roots). Therefore, we get
    \begin{align*}
        e_{j5} \le D_m {\sigma_1^*}^2 O\left( (1+\pi_1^*/\pi_j^*) \sqrt{d} (\R{j1}/{\sigma_1^*})^4 \exp(-\R{j1}^2 / 64 \sigt{j}{1}^2) \right).
    \end{align*}
    Finally, with the similar strategy, we have
    \begin{align*}
        e_{j6} &= \left| \sum_{l\neq 1} \E_{\D_j} \left[ w_1^u w_l^u (\|X - \mu_1^*\|^2 - d {\sigma_1^*}^2) (\|X - \mu_l^u\|^2 - d {\sigma_l^u}^2) / (2{\sigma_l^u}^2) \right] (\sigma_l^2 - {\sigma_l^*}^2) / {\sigma_l^u}^2 \right| \\
        &\le \left(2 D_m / \sqrt{d} \right) \sqrt{\E_{\D_j} \left[ w_1^u (\|X - \mu_1^*\|^2 - d {\sigma_1^*}^2)^2 \right]} \sum_{l\neq 1} (1/{\sigma_l^*}^2) \sqrt{\E_{\D_j} \left[ w_1^u (\|X - \mu_1^u\|^2 - d {\sigma_1^u}^2)^2 \right]},
    \end{align*}
     The summation over $l\neq 1$ is $O(1)$ as shown in Lemma \ref{lemma:helper_sixth_term}. Therefore,
    \begin{align*}
        e_{j6} \le D_m {\sigma_1^*}^2 O\left((1+\pi_1^*/\pi_j^*) \sqrt{d} (\R{j1}/{\sigma_1^*})^2 \exp(-\R{j1}^2/128\sigt{j}{1}^2) \right).
    \end{align*}
    
    Collecting all error terms, every term is less than 
    \begin{align*}
        {\sigma_1^*}^2 (1 + \pi_1^*/\pi_j^*) \sqrt{d} O\left( (\R{j1}/\sigma_1^*)^3 \exp(-\R{j1}^2/128\sigt{j}{1}^2) + (\R{j1}/\sigma_1^*)^4 \exp(-\R{j1}^2/64\sigt{j}{1}^2) \right).
    \end{align*}
    By Lemma \ref{lemma:sum_pnorm_error}, the entire summation is smaller than $c {\sigma_1^*}^2 \pi_1^* \sqrt{d}$ for some small constant $c$. Recall that $\sqrt{d}$ will be divided by $d$ in the end.
    
    \paragraph{Errors from own component $j=1$:} We will walk through the same procedure. The first term is:
    \begin{align*}
        e_{11} &= |\E_{\D_1} [w_1^u (1-w_1^u) (X - \mu_1^u)^T (\mu_1 - \mu_1^*) (\|v\|^2 - d {\sigma_1^*}^2) ] / {\sigma_1^u}^2 | \\
        &\le 2 D_m \sup_{s \in \mathbb{S}^{d-1}} \E_{\D_1} [(1 - w_1^u) (\|v\|^2 - d {\sigma_1^*}^2) \vdot{X - \mu_1^u}{s}] / {\sigma_1^*}  \\
        &\le 2D_m \sup_{s \in \mathbb{S}^{d-1}} \sum_{l\neq 1} \E_{\D_1} [w_l^u (\|v\|^2 - d{\sigma_1^*}^2) \vdot{X - \mu_1^u}{s}] / \sigma_1^* \\
        &\le 2D_m \sqrt{\E_{\D_1} [(\|v\|^2 - d{\sigma_1^*}^2)^2]} / {\sigma_1^*} \sup_{s \in \mathbb{S}^{d-1}} \sum_{l\neq1} \sqrt{\E_{\D_1}[w_l^u \vdot{X - \mu_1^u}{s}^2]} \\
        &\le D_m \sqrt{d} \sigma_1^* O\left( \sum_{l\neq 1} (1 + \pi_l^*/\pi_1^*) {\sigma_1^*} \exp(-\R{l1}^2/128\sigt{l}{1}^2) \right),
    \end{align*}
    which is smaller than $c D_m \sqrt{d} {\sigma_1^*}$ for some small $c$ by Lemma \ref{lemma:sum_pnorm_error}.
    
    Similarly, the second term can be bounded as
    \begin{align*}
        e_{12} &= \left| \E_{\D_1} [\sum_{l\neq1} w_1^u w_l^u (X - \mu_l^u)^T (\mu_l - \mu_l^*) (\|v\|^2 - d {\sigma_1^*}^2) ] / {\sigma_l^u}^2 \right| \\
        &\le 2 D_m \sum_{l\neq 1} \sup_{s \in \mathbb{S}^{d-1}} \E_{\D_1} [w_l^u (\|v\|^2 - d {\sigma_1^*}^2) \vdot{X - \mu_l^u}{s}] / {\sigma_l^*} \\
        &\le 2D_m \sup_{s \in \mathbb{S}^{d-1}} \sum_{l\neq 1} \E_{\D_1} [w_l^u (\|v\|^2 - d{\sigma_1^*}^2) \vdot{X - \mu_1^u}{s}] / \sigma_l^* \\
        &\le 2D_m \sqrt{\E_{\D_1} [(\|v\|^2 - d{\sigma_1^*}^2)^2} \sup_{s \in \mathbb{S}^{d-1}} \sum_{l\neq1} \sqrt{\E_{\D_1}[w_l^u \vdot{X - \mu_1^u}{s}^2] / {\sigma_l^*}^2} \\
        &\le c D_m \sqrt{d} {\sigma_1^*}^2,
    \end{align*}
    for small constant $c$, where in the last step we used Lemma \ref{lemma:helper_second_term}.
    
    $e_{13}$ and $e_{14}$ can be bounded similarly. Finally, $e_{15}$ can be bounded as
     \begin{align*}
        e_{15} &= \left| \E_{\D_1} \left[ w_1^u (1 - w_1^u) (\|v\|^2 - d {\sigma_1^*}^2) (\|X - \mu_1^u\|^2 - d {\sigma_1^u}^2) / (2{\sigma_1^u}^2) \right] (\sigma_1^2 - {\sigma_1^*}^2) / {\sigma_1^u}^2 \right| \\
        &\le \left(2 D_m / {\sigma_1^*}^2 \sqrt{d} \right) \sum_{l\neq1} \E_{\D_1} \left[ (\|v\|^2 - d {\sigma_1^*}^2) w_l^u (\|v + (\mu_1^* - \mu_1^u)\|^2 - d{\sigma_1^*}^2 + d ({\sigma_1^*}^2 - {\sigma_1^u}^2)^2 \right] \\
        &\le \left(2 D_m / {\sigma_1^*}^2 \sqrt{d} \right) \sum_{l\neq1} \E_{\D_1} \left[ w_l^u (\|v\|^2 - d {\sigma_1^*}^2) \left( (\|v\|^2 - d {\sigma_1^*}^2) + 2 D_m {\sigma_1^*} \vdot{v}{s} + D_m^2{\sigma_1^*}^2 + D_m {\sigma_1^*}^2 \sqrt{d} \right) \right] \\
        &\le \left(2 D_m / {\sigma_1^*}^2 \sqrt{d} \right) \sum_{l\neq1} \sqrt{\E_{\D_1} [ w_l^u ]} \sqrt{\E_{\D_1} [(\|v\|^2 - d {\sigma_1^*}^2)^4]} + 2 D_m\sigma_1^* \sqrt{\E_{\D_1} [w_l^u \vdot{v}{s}^2 ] } \sqrt{\E_{\D_1} [(\|v\|^2 - d{\sigma_1^*}^2)^2]} \\
        &\qquad \qquad \qquad \qquad + (D_m \sqrt{d} + D_m^2) {\sigma_1^*}^2 \sqrt{\E_{\D_1} [w_l^u] } \sqrt{\E_{\D_1} [(\|v\|^2 - d{\sigma_1^*}^2)^2]}
    \end{align*}
    Then, we are summing over $O\left((1+\pi_l^*/\pi_1^*) (R_{l1}/{\sigma_1^*}) \exp(-\R{l1}^2/128\sigt{l}{1}^2) \right)$ for $l \neq 1$, which is bounded again with Lemma \ref{lemma:sum_pnorm_error}. Dominating error bound is therefore $e_{15} \le c D_m \sqrt{d} {\sigma_1^*}^2$ again for small constant $c$. We follow the similar procedure for the final term $e_{16}$.
    \begin{align*}
        e_{16} &= \left| \E_{\D_1} \left[ \sum_{l\neq1} w_1^u w_l^u (\|X - \mu_1^*\|^2 - d {\sigma_1^*}^2) (\|X - \mu_l^u\|^2 - d {\sigma_l^u}^2) / (2{\sigma_l^u}^2) \right] (\sigma_l^2 - {\sigma_l^*}^2) / {\sigma_l^u}^2 \right| \\
        &\le \left( 2 D_m/ \sqrt{d} \right) \sum_{l\neq 1} (1/{\sigma_l^*}^2)  \E_{\D_1} \left[ w_1^u w_l^u (\|v\|^2 - d{\sigma_1^*}^2) ((\|v\|^2- d{\sigma_1^*}^2) + 2 \R{l1} \vdot{v}{s}+ 2\R{l1}^2 + d({\sigma_1^*}^2 - {\sigma_l^u}^2)) \right] \\
        &\le \left( 2 D_m/ \sqrt{d} \right) \sum_{l\neq 1} (1/{\sigma_l^*}^2) \Bigg( \E_{\D_1} \left[w_l^u (\|v\|^2 - d{\sigma_1^*}^2)^2 \right] + 2 \E_{\D_1} \left[ \R{l1} w_l^u (\vdot{v}{s} + \R{l1}) (\|v\|^2 - d{\sigma_1^*}^2) \right] \\
        &\qquad \qquad \qquad \qquad + \E_{\D_1} [w_l^u (\|v\|^2 - d{\sigma_1^*}^2) d ({\sigma_1^*}^2 - {\sigma_l^*}^2)] + \E_{\D_1} [w_l^u (\|v\|^2 - d{\sigma_1^*}^2) d ({\sigma_l^*}^2 - {\sigma_l^u}^2)] \Bigg).
    \end{align*}
    The remaining steps are bounding each four term in the summation.
    \begin{align*}
        \E_{\D_1} [w_l^u (\|v\|^2 - d{\sigma_1^*}^2)^2 ] &\le \sqrt{\E_{\D_1} [w_l^u]} \sqrt{\E_{\D_1} [ (\|v\|^2 - d{\sigma_1^*}^2)^4 ]} \\
        &\le O \left(d {\sigma_1^*}^4 (1 + \pi_l^* / \pi_1^*) \exp(-\R{l1}^2/128\sigt{l}{1}^2) \right),
    \end{align*}
    where we used \eqref{eq:chi_fourth_central_moment} for bounding fourth order central moment of degree-$d$ chi-square random variable. The second one is 
    \begin{align*}
        2\R{l1} \E_{\D_1} [w_l^u (\vdot{v}{s} + R_{l1}^*) (\|v\|^2 - d{\sigma_1^*}^2) ] &\le 2 \R{l1} \sqrt{2 \E_{\D_1} [w_l^u (\vdot{v}{s}^2 + \R{l1}^2) ]} \sqrt{\E_{\D_1} [ (\|v\|^2 - d{\sigma_1^*}^2)^2 ]} \\
        &\le O \left(\sqrt{d} \R{l1}^2 {\sigma_1^*}^2 (1 + \pi_l^* / \pi_1^*) \exp(-\R{l1}^2/128\sigt{l}{1}^2) \right).
    \end{align*}
    Third term we will use Lemma \ref{lemma:weight_bound_for_variance} to get,
    \begin{align*}
        d ({\sigma_1^*}^2 - {\sigma_l^*}^2) \E_{\D_1} [w_l^u (\|v\|^2 - d{\sigma_1^*}^2)] &\le d ({\sigma_1^*}^2 - {\sigma_l^*}^2) \sqrt{\E_{\D_1}[w_1^u]} \sqrt{\E_{\D_1}[(\|v\|^2 - d{\sigma_1^*}^2)]} \\
        &\le O\left( \sqrt{d} \R{l1}\sigt{l}{1} \exp(-\R{l1}^2/128\sigt{l}{1}^2) \sqrt{2d} {\sigma_1^*}^2 \right),
    \end{align*}
    and finally we have
    \begin{align*}
        d ({\sigma_l^*}^2 - {\sigma_l^u}^2) \E_{\D_1} [w_l^u (\|v\|^2 - d{\sigma_1^*}^2)] &\le \sqrt{d} D_m {\sigma_l^*}^2 \sqrt{\E_{\D_1}[w_1^u]} \sqrt{\E_{\D_1}[(\|v\|^2 - d{\sigma_1^*}^2)]} \\
        &\le O\left( \sqrt{d} {\sigma_l^*}^2 \exp(-\R{l1}^2/128\sigt{l}{1}^2) \sqrt{2d} {\sigma_1^*}^2 \right).
    \end{align*}
    Now, we can collect all terms and bound $e_{16}$ as
    \begin{align*}
        e_{16} &\le \sqrt{d} D_m {\sigma_1^*}^2 O\left( \sum_{l\neq 1} (\R{l1}/{\sigma_l^*})^2 (1 + \pi_l^*/\pi_1^*) \exp(-\R{l1}^2 / 128\sigt{l}{1}^2) \right) \le c \sqrt{d} D_m {\sigma_1^*}^2,
    \end{align*}
    with small constant $c$. 
    
    \paragraph{Errors from all components:} Now we have that errors across all components can be bounded as
    \begin{align*}
        \sum_{j=1}^k \pi_j^* (e_{j1} + ... + e_{j6}) &\le \pi_1^* c (D_m  \sqrt{d} {\sigma_1^*}^2) \\
        &\qquad + c (D_m \sqrt{d} {\sigma_1^*}^2) O\left( \sum_{j\neq1} (\pi_1^* + \pi_j^*) (R_{j1}/{\sigma_1^*})^4 \exp(-\R{j1}^2/128\sigt{j}{1}^2) \right) \\
        &\le c' D_m \sqrt{d} {\sigma_1^*}^2 \pi_1^*, 
    \end{align*}
    using Lemma \ref{lemma:sum_pnorm_error} with $q = 4$. Finally, recalling the errors for variances, 
    \begin{align*}
        |{\sigma_1^+}^2 - {\sigma_1^*}^2| \le c' D_m\sqrt{d} {\sigma_1^*}^2 {\pi_1^*} / d \pi_1^+ + \|\mu_1^+ - \mu_1^*\|^2 / d \le c_\sigma D_m {\sigma_1^*}^2 / \sqrt{d},
    \end{align*}
    where we considered that the mean estimator at the next iteration is also improved. The constant is $c_\sigma < 1$ given large enough separation between centers of Gaussians.
\end{proof}

In all cases, we can conclude that $D_m^+ = \max_j (\|\mu_j^+ - \mu_j^*\|/{\sigma_j^*}, |\pi_j^+ - \pi_j^*| / \pi_j^*, \sqrt{d} |{\sigma_j^+}^2 - {\sigma_j^*}^2|/{\sigma_j^*}^2) \le \gamma D_m$, for some constant $\gamma < 1$.

\subsection{Proof of Auxiliary Lemmas}
\label{appendix:auxiliary_lemma_for_E}
\subsubsection{Proof of Lemma \ref{lemma:sum_pnorm_error}}
\begin{proof}
    We will only show for $q = 4$ and the other cases will follow similarly. Again, let $x := \R{j1}^2 / \sigt{1}{j}^2$. Then, since $x \ge C^2 \ge 128^2$ by the separation condition, $x \ge 512 \log x$. Then the proof is again trivial:
    \begin{align*}
        \sum_{j\neq1} (\pi_1^* + \pi_j^*) \rho_\sigma^4 x^2 \exp(-x^2/128) &\le \sum_{j\neq1} (\pi_1^* + \pi_j^*) \rho_\sigma^4 \exp(-x^2/256) \\
        &\le \sum_{j\neq1} (\pi_1^* + \pi_j^*) \rho_\sigma^4 (\rho_\sigma/\pi_{min})^{-32} \ll c \pi_{min},
    \end{align*}
    for small constant $c$.
\end{proof}

\subsubsection{Proof of Lemma \ref{lemma:helper_second_term}}
\begin{proof}
    If $l = j$, upper bound is simply obtained by setting $w_j^u = 1$.
    \begin{align*}
        \E_{\D_j} [w_j^u \vdot{v + \mu_j^* - \mu_j^u}{s}^2] &\le \E_{\D_j} [ 2\vdot{v}{s}^2 + 2 {\sigma_j^*}^2 D_m^2] \le 4{\sigma_j^*}^2.
    \end{align*}
    
    If $l \neq j$, then
    \begin{align*}
        \E_{\D_j} [w_l^u \vdot{v + \mu_j^* - \mu_l^u}{s}^2] &\le \E_{\D_j} [ w_l^u (2\vdot{v}{s}^2 + 4 {R_{jl}^*}^2)] \le O \left({R_{jl}^*}^2 (1 + \pi_l^* / \pi_j^*) \exp\left(-{R_{jl}^*}^2 / 64(\sigma_j^* \vee \sigma_l^*)^2 \right) \right),
    \end{align*}
    where we used Lemma \ref{lemma:w1_vsl2_small_jneq1}. Then, summation over $l \neq j$ yields 
    \begin{align*}
        \sum_{l\neq j} &\sqrt{\E_{\D_j} [w_l^u \vdot{v + \mu_j^* - \mu_l^u}{s}^2] / {\sigma_l^*}^2 }  
        \le \sum_{l \neq j} O \left( ({R_{jl}^*}/\sigma_l^*) (1 + \pi_l^* / \pi_j^*) \exp\left(-{R_{jl}^*}^2 / 128(\sigma_j^* \vee \sigma_l^*)^2 \right) \right) \le c,
    \end{align*}
    for some small constant $c$ by Lemma \ref{lemma:sum_pnorm_error}.
\end{proof}

\subsubsection{Proof of Lemma \ref{lemma:helper_sixth_term}}
\begin{proof}
    If $l = j$, then the upper bound can be found by setting $w_l^u = 1$. That is,
    \begin{align*}
        \left( 1/{\sigma_j^*}^2 \right) &\Bigg( \sqrt{\E_{\D_j} [(\|v\|^2 - d {\sigma_j^*}^2)^2]} + 2 \sqrt{\E_{\D_j} [\vdot{v}{\mu_j^* - \mu_j^u}^2]} + d({\sigma_j^*}^2 - {\sigma_l^u}^2) \Bigg) \\
        &\le \left(\sqrt{2d} {\sigma_j^*}^2 + 4 {\sigma_j^*}^2 D_m + D_m {\sigma_j^*}^2 \sqrt{d} \right) / {\sigma_j^*}^2 = O(\sqrt{d}).
    \end{align*}
    
    Now consider when $l \neq j$. In addition to previous lemmas, we need the fact about chi-square distribution with degree $d$. Its fourth central moment is,
    \begin{align}
        \label{eq:chi_fourth_central_moment}
        \E_{v\sim \mathcal{N}(0,I_d)} [(\|v\|^2 - d)^4] = 12 d(d+4).
    \end{align} 
    With this fact,
    \begin{align*}
        \E_{\D_j} [w_l^u (\|v\|^2 - d {\sigma_j^*}^2)^2] &\le \sqrt{ \E_{\D_j} [w_l^u] } \sqrt{\E_{\D_j}[(\|v\|^2 - d {\sigma_j^*}^2)^4]} \le  8d{\sigma_j^*}^4 (1 + \pi_l^* / \pi_j^*) \exp(-{R_{jl}^*}^2 / 128 (\sigma_j^*\vee \sigma_l^*)^2).
    \end{align*}
    Other terms can also be bounded similarly as
    \begin{align*}
        \E_{\D_j} [w_l^u \vdot{v}{\mu_j^* - \mu_l^u}^2] &\le 2 {R_{jl}^*}^2 \E_{\D_j} [w_l^u \vdot{v}{s}^2] \le O \left((1 + \pi_l^* / \pi_j^*) {R_{jl}^*}^4 \right) \exp(-{R_{jl}^*}^2 / 64 (\sigma_j^* \vee \sigma_l^*)^2 ), \\
        \E_{\D_j} [w_l^u ] &\le O\left( (1 + \pi_l^* / \pi_j^*) \right) \exp(-{R_{jl}^*}^2 / 64 (\sigma_j^* \vee \sigma_l^*)^2 ).
    \end{align*}
    
    We use Lemma \ref{lemma:weight_bound_for_variance} to bound the value of $d({\sigma_j^*}^2 - {\sigma_l^u}^2) \E_{\D_j} [w_l^u]$. We first remove $u$ superscript as
    \begin{align*}
        d({\sigma_j^*}^2 - {\sigma_l^u}^2) \E_{\D_j} [w_l^u] &= (d({\sigma_j^*}^2 - {\sigma_l^*}^2) + d({\sigma_l^*}^2 - {\sigma_l^*}^2)) \E_{\D_j} [w_l^u] \\
        &\le d({\sigma_j^*}^2 - {\sigma_l^*}^2) \E_{\D_j} [w_l^u] + \sqrt{d} {\sigma_l^*}^2 \E_{\D_j} [w_l^u].
    \end{align*}
    From Lemma \ref{lemma:weight_bound_for_variance}, either we have $({\sigma_j^*}^2 - {\sigma_l^*}^2) \le 10 {R_{jl}^*} (\sigma_j^* \vee \sigma_l^*) / \sqrt{d}$ or
    \begin{align*}
        \E_{\D_j} [w_l^u] \le O\left( \exp(-d \min(1, t^2) / 256) \exp(- {R_{jl}^*}^2 / 64(\sigma_j^* \vee \sigma_l^*)^2) \right),
    \end{align*}
    where $t = ({\sigma_j^*}^2 - {\sigma_l^*}^2) / {\sigma_l^*}^2$. This gives us 
    \begin{align*}
        d({\sigma_j^*}^2 - {\sigma_l^*}^2) / {\sigma_l^*}^2 \E_{\D_j}[w_l^u] \le O((1 + \pi_l^* / \pi_j^*) \R{j1} \sigt{j}{l} \sqrt{d} \exp(-\R{jl}^2 /64\sigt{j}{l}^2).
    \end{align*} 
    Now summing up all terms over $l \neq j$, we get
    \begin{align*}
        \sum_{l\neq j} &\sqrt{d} \sqrt{1 + \pi_l^* / \pi_j^*} O\left( \R{jl}^2/ {\sigma_l^*}^2 \exp(-\R{jl}^2/128\sigt{j}{l}^2) + {\sigma_j^*}^2/{\sigma_l^*}^2 \exp(-\R{jl}^2/256\sigt{j}{l}^2) \right) \\
        &\le \sqrt{d}/\pi_j^* \sum_{l\neq j}(\pi_j^* + \pi_l^*) O\left( \R{jl}^2/{\sigma_l^*}^2 \exp(-\R{jl}^2/128\sigt{j}{l}^2) + {\sigma_j^*}^2/{\sigma_l^*}^2\exp(-\R{jl}^2/256\sigt{j}{l}^2) \right) \\
        &\le c\sqrt{d},
    \end{align*} 
    for some small constant $c$.
\end{proof}